\documentclass[10pt]{article}

\usepackage{arxiv}

\usepackage{natbib}
\bibpunct{(}{)}{;}{a}{,}{,}

\usepackage[utf8]{inputenc} 
\usepackage[T1]{fontenc}    
\usepackage{hyperref}       
\usepackage{url}            
\usepackage{booktabs}       
\usepackage{amsfonts}       
\usepackage{microtype}      
\usepackage{xcolor}         

\usepackage{caption}
\usepackage{subcaption}
\usepackage{wrapfig}

\usepackage{thm-restate}

\usepackage{macros}

\usepackage{tikz, pgfplots}
\usetikzlibrary{arrows.meta, calc, positioning}

\newcommand{\Lintervals}{\mathcal L}
\newcommand{\Rintervals}{\mathcal R}
\newcommand{\I}{\mathcal I}
\newcommand{\J}{\mathcal J}
\newcommand{\tL}{\widetilde L}
\newcommand{\tR}{\widetilde R}
\newcommand{\tI}{\widetilde I}
\newcommand{\tJ}{\widetilde J}
\newcommand{\tK}{\widetilde K}
\newcommand{\tO}{\widetilde O}
\newcommand{\Union}{\bigcup}

\newcommand{\union}{\cup}
\newcommand{\intersection}{\cap}
\newcommand{\reals}{\mathbb{R}}
\newcommand{\xcf}{x^\textnormal{cf}}   
\newcommand{\domain}{\mathcal{X}}
\newcommand{\funcs}{\mathcal{F}}

\usepackage{enumitem} 
\setlist{noitemsep, topsep=0cm}

\title{Attribution-based Explanations that Provide Recourse \linebreak Cannot be 
Robust}
\author{%
    \bf{Hidde Fokkema} \\
    Korteweg-de Vries Institute for Mathematics \\
    University of Amsterdam  \\
    Amsterdam, The Netherlands \\
    \texttt{h.j.fokkema@uva.nl}\\
    \And
    \bf{Rianne de Heide} \\
    Department of Mathematics \\
    Vrije Universiteit Amsterdam \\
    Amsterdam, The Netherlands \\
    \texttt{r.de.heide@vu.nl}\\
    \And
    \bf{Tim van Erven} \\
    Korteweg-de Vries Institute for Mathematics \\
    University of Amsterdam \\
    Amsterdam, The Netherlands\\
    \texttt{tim@timvanerven.nl}\\
}
\begin{document}
\maketitle

\begin{abstract}%
Different users of machine learning methods require different
explanations, depending on their goals. To make machine learning
accountable to society, one important goal is to get actionable options
for \emph{recourse}, which allow an affected user to change the decision
$f(x)$ of a machine learning system by making limited changes to its
input $x$. We formalize this by providing a general definition of
recourse sensitivity, which needs to be instantiated with a utility
function that describes which changes to the decisions are relevant to
the user. This definition applies to local attribution methods, which
attribute an importance weight to each input feature. It is often argued
that such local attributions should be robust, in the sense that a small
change in the input $x$ that is being explained, should not cause a
large change in the feature weights. However, we prove formally that it
is in general impossible for any single attribution method to be both
recourse sensitive and robust at the same time. It follows that there
must always exist counterexamples to at least one of these properties.
We provide such counterexamples for several popular attribution methods,
including LIME, SHAP, Integrated Gradients and SmoothGrad. Our results
also cover counterfactual explanations, which may be viewed as
attributions that describe a perturbation of $x$. We further discuss
possible ways to work around our impossibility result, for instance by
allowing the output to consist of sets with multiple attributions, and 
we provide sufficient conditions for specific classes of continuous functions 
to be recourse sensitive.
Finally, we strengthen our impossibility result for the restricted case
where users are only able to change a single attribute of $x$, by
providing an exact characterization of the functions $f$ to which
impossibility applies.%
\end{abstract}

\section{Introduction}

As machine learning (ML) is changing science and society in many ways,
its trustworthiness and interpretability are coming under increasing
scrutiny \citep{Varshney2022,molnar2022}. Since most ML systems are not
inherently interpretable, there have been many proposals to generate
explanations that communicate relevant aspects of their internal
workings \citep{linardatos2020explainable, samek2019explainable}. 
Which particular aspects are relevant, depends on the target
audience and its goals \citep{arrieta2020explainable,Varshney2022}. We consider
here the case that the target audience consists of users with
corresponding feature vectors $x \in \reals^d$, who are affected by the decisions
$f(x)$ of an ML system. It is assumed that each user has some limited
ability to change (a subset of) their features in $x$, and the goal of
the users is to use this ability to influence the resulting decision of
$f$ in a way that increases their utility by a sufficient amount. In
order to achieve this goal, an explanation for a given input $x$ should
provide actionable options for \emph{recourse}, i.e.\ changes to $x$
that both provide sufficient utility \emph{and} lie within the ability
of the user to realize.
\paragraph{Attribution Methods and Counterfactual Explanations}

We consider explanations that take the form of local attributions
$\phi_f(x) \in \reals^d$, which are vectors that assign a weight to each
feature in~$x$ that indicates its importance. Many explanation methods
produce such attributions. For instance, well-known methods like LIME
\citep{ribeiro2016should}, SHAP \citep{lundberg2017unified}, 
Integrated Gradients \citep{sundararajan2017axiomatic} and
SmoothGrad \citep{smilkov2017smoothgrad} are attribution methods. When applied to image
classification, where the features are pixels, the attribution vector
$\phi_f(x)$ is called a \emph{saliency map} and can be visualized as a
picture that highlights the most important pixels. An approach closely
related to attributions, which is often considered in the context of
recourse, is to provide counterfactual explanations
\citep{karimi2021survey, verma2020counterfactual, keane2021if}. Methods
of this type include those by
\citet{poyiadzi2020face, karimi2020model,wachter2017counterfactual, mothilal2020explaining, 
dandl2020multi, dhurandhar2018explanations}. 
These methods generate an
alternative (counterfactual) input $\xcf$ that is both similar to~$x$
and provides sufficient utility. For example, if $f$ is a classifier,
then $f(\xcf)$ might be a more desirable class for the user than $f(x)$.
The differences between $\xcf$ and $x$ are then interpreted as the
changes needed to flip the class, so that $\phi_f(x) = \xcf - x$ can again
be regarded as an attribution vector.

\paragraph{Robustness to Changes in the Inputs} 

Developing precise design criteria for explainability methods
has proven to be difficult
\citep{jacovi2020towards,zhou2021evaluating,hooker2019benchmark}. In the
absence of these, a way forward is to consider desirable criteria that
should be satisfied. One such criterion, which is commonly proposed in
the context of recourse, is for explanations to be robust to changes in
the inputs \citep{karimi2021survey, alvarez2018robustness}: if we reason
that similar users should get similar options for recourse, then small
changes in the input $x$ should not cause large jumps in the explanation
$\phi_f(x)$. This can either be formalized to mean that $\phi_f$ should
be continuous or, more restrictively, that it should be (locally)
Lipschitz continuous. We will adopt the weaker of these two, because
that strengthens our main theoretical results: a function that violates
continuity automatically also violates Lipschitz continuity. Thus, we
settle on the following definition of robustness:  
\begin{defn}\label{def:robustness}
    An attribution method $\varphi_f \colon \mathcal{X} \to \mathbb{R}^{d}$ is 
    called \emph{robust} if it is continuous. 
\end{defn}
Robustness can be seen as a measure of coherence
\citep{jacovi2022diagnosing}. It appears to be difficult to achieve, however,
because a sequence of empirical counterexamples have been found in which
several methods like LIME, SHAP and Integrated Gradients  are not robust
\citep{ghorbani2019interpretation,slack2020fooling,dombrowski2019explanations,
alvarez2018robustness}.\footnote{Integrated Gradients is continuous by
design, so in this case we can take the empirical results to mean that
it is not Lipschitz continuous. In practice, the two notions are
difficult to distinguish: if we only have sample access to a function,
it is difficult to discern whether a sudden jump between two samples is
caused by an actual discontinuity, or occurs because the function is
very steep.} On the other hand, it has been established that SmoothGrad
and C-LIME (a continuous variant of LIME) \emph{are} provably robust,
because they produce attributions $\phi_f$ that are always Lipschitz
continuous \citep{agarwal2021towards}. 

\paragraph{Main Contributions}

We take a more abstract look at why existing attribution methods may
fail at being robust to changes in the inputs. Our explanation is that
there is in fact a fundamental contradiction between robustness and
providing recourse. This is established by our main result,
Theorem~\ref{thm:impossibility_result} in
Section~\ref{sec:impossibility_results}, which shows that:

\begin{center}
  \begin{minipage}{0.9\textwidth}
  \centering
  \emph{For any way of measuring utility, there exists a (continuous)
  machine learning model $f$ for which no attribution method $\phi_f$
  can be both recourse sensitive and continuous.}
  \end{minipage}
\end{center}
Our result captures many possible variations of how recourse may be
defined via a permissive property we call \emph{recourse sensitivity}.
Recourse sensitivity is introduced in
Section~\ref{sec:recourse_sensitivity}. In particular, it allows
attributions $\phi_f(x)$ to be scaled arbitrarily (which makes it easier
for a suitable $\phi_f$ to exist), as long as the vector $\phi_f(x)$
points in any direction that would allow the user to obtain sufficient
utility. We also do not restrict to the case where users want to flip
the class of a classifier $f$, but allow for a general utility function
$u_f$ that measures the user's utility. Finally, the contradiction
between recourse sensitivity and continuity of $\phi_f$ does not require
$f$ to be some obscure function, but already occurs, for instance, for quadratic
functions~$f$ (see Section~\ref{sec:proof_idea}). This implies
that most model classed used in practice are expressive enough to
exhibit the problem.
In Section~\ref{sec:impossibility_results}, we further illustrate our main
impossibility result with experiments and analytical examples that show
cases in which the well-known attribution methods SmoothGrad
\citep{smilkov2017smoothgrad}, Integrated Gradients
\citep{sundararajan2017axiomatic}, LIME \citep{ribeiro2016should} and
SHAP \citep{lundberg2017unified} fail to be recourse sensitive. We also
provide an analytical example in which counterfactual explanations fail
to be continuous. We then reflect on our impossibility result in
Section~\ref{sec:discussion}, and discuss possible ways around it. While
the impossibility result implies that some functions $f$ are
problematic, it is still possible that joint recourse sensitivity and
continuous attributions are possible under (necessarily restrictive)
conditions on $f$, for instance if $f$ is linear. We study this in
Section~\ref{sec:suff_cond_main}, where we provide sufficient conditions
on $f$ that do generalize beyond the linear case. Finally, in
Section~\ref{sec:suff_and_nec_conditions_in_one_dimension}, we
strengthen the impossibility result from
Theorem~\ref{thm:impossibility_result} and the sufficient conditions
from Section~\ref{sec:suff_cond_main} by providing an exact
characterization of the set of functions $f$ to which impossibility
applies for two restricted special cases: first we characterize
impossibility for dimension $d=1$; then we extend this result to any $d
\geq 1$ under the assumption that the user is only able to change a
single feature. 

\subsection{Related Work}

\paragraph{The Taxonomy of Explanation Methods}
Interpretability of machine learning methods can be achieved by training
inherently interpretable models $f$ or by providing \emph{post-hoc}
explanations of a model $f$ that has already been trained. Explanations
can be \emph{global}, explaining aspects of the full function $f$, or
\emph{local}, explaining the behavior of $f$ around a given point $x$ 
\citep{zhou2021evaluating, molnar2022, Varshney2022, das2020opportunities}.
Recourse fits into this taxonomy as a post-hoc, local type of
explanations \citep{linardatos2020explainable, samek2019explainable}. 

\paragraph{Distance Measures}

The survey by \citet{karimi2021survey} provides a unified view on
existing algorithms that provide recourse via counterfactual
explanations. In the simplest case, such methods measure the distance
between $x$ and $\xcf$ by the Euclidean distance, but more refined
distance measures have also been proposed
\citep{wachter2017counterfactual, karimi2020model, 
poyiadzi2020face, joshi2019towards, arvanitidis2020geometrically}.  
For simplicity, we
restrict attention to the Euclidean distance in our results, but we
expect that they can be generalized to many other distance measures as
well.

\paragraph{Consequential Recommendations}

\citet{karimi2021survey} further describe a generalization of
counterfactual explanations, called \emph{consequential
recommendations}, in which users are not able to change individual
features directly, but can only influence features indirectly via more
abstract actions. The effect of actions on features is described by a
causal model, and instead of the change in features the users are
restricted by the cost of taking particular actions. As will be
discussed in Section~\ref{sec:recourse_sensitivity}, our definition of
recourse sensitivity is sufficiently general that it can also express
consequential recommendations. Since the conditions of our main theorem are
very mild, they will then also apply to many, but not all, causal
models.

\paragraph{Other Notions of Robustness}

As mentioned already, robustness can mean multiple things. So far, we
discussed (local) Lipschitz continuity and ordinary continuity. These
notions both consider robustness with respect to the input variable $x$.
One other notable interpretation of robustness is robustness of a
counterfactual with respect to changes to the model $f$. For instance, a
model may be periodically retrained \citep{ferrario2022robustness} or it
may be updated when someone wants to be removed from the training set
\citep{pawelczyk2022trade}. There have been multiple methods developed
to generate counterfactuals that are still valid under these model
shifts \citep{black2021consistent, upadhyay2021towards,
hamman2023robust}. These types of robustness are orthogonal to the type
of robustness we consider in this work.

\paragraph{Solidifying the Foundations of Explainability}

Explainability is an exciting new research area that is witnessing a
flurry of new methods and ideas. But it is clear that no single
explanation method can exist that is good for all purposes \citep{guidotti2018survey}. As
limitations of existing explanation methods are being discovered
\citep{adebayo2018sanity, kindermans2019reliability, rudin2019stop},
this has lead to a desire to solidify the foundations and
practice of explainability research. For instance, there is a lively
debate on how to measure desirable properties like faithfulness,
fidelity, plausibility, etc. 
\citep{jacovi2020towards, ge2021counterfactual, guidotti2018survey}. In this context, it is
important to know which desirable properties can coexist in principle,
and our work contributes to this understanding by pointing out that
robustness is incompatible with providing recourse.

A series of recent works find fault with post-hoc attribution methods,
either by showing empirical examples in which they exhibit undesirable
behavior or by establishing theoretical results that identify fundamental
limitations. On the empirical side,
\citet{rudin2019stop} gives multiple arguments why post-hoc methods
should not be used for high-stake decisions, because they are often not
faithful or provide too little detail; \citet{laugel2019dangers} show
that post-hoc counterfactuals have a high risk of being far away from
any ground truth data point; \citet{slack2020fooling} show that post-hoc
methods can be used to hide biases in a model; and in an adversarial
setting it has been shown that post-hoc explanations can easily be
manipulated \citep{dombrowski2019explanations, bordt2022post}. On the
theoretical side, subsequent to a pre-print of our work,
\citet{impossibility2022bilodeau} derived additional impossibility
results, which apply to hypothesis testing between pairs of model
behaviors from the output of attribution methods. They are able to
obtain much more general conclusions than we do, because they restrict
attention to the more narrow class of attribution methods that are
complete and linear, whereas our impossibility results apply to any
attribution method in general. A significant limitation of requiring
completeness is that it excludes all counterfactual methods that are
commonly used in algorithmic recourse.
Since many of the problems with attribution methods are related to challenges in
describing exactly when and for what purpose explanation methods can be
trusted, one way forward may lie in the calls for greater rigor by
\citet{lipton2018mythos} and \citet{doshi2017towards}.
\citet{leavitt2020towards} even go as far as claiming that
``interpretability research suffers from an over-reliance on
intuition-based approaches that risk --- and in some case have caused
--- illusory progress and misleading conclusions''. These works plead
for the development of theory that may lead to provably better
interpretability methods, and we view our work as a contribution in that
direction.

\section{Recourse Sensitivity}\label{sec:recourse_sensitivity}

In this section we formally introduce recourse sensitivity, and show
that, on its own, it can always be satisfied, for instance by
counterfactual explanations. 

\paragraph{Setting}

We assume that $x$ takes values in some domain $\domain \subseteq
\reals^d$, and that the machine learning model $f$ is an element of the
set $\funcs$ of all functions from $\domain$ to $\reals$. An attribution
method for a given function $f \in \funcs$ is a function $\phi_f \colon
\domain \to \reals^d$.

\paragraph{Utility Functions}

To define recourse sensitivity, we describe the user's preferences for a
given model $f \in \funcs$ by a
\emph{utility function} $u_f\colon \domain \times \domain \to \reals$ with
the interpretation that $u_f(x,y)$ is the utility experienced by the user
if they succeed in changing their original features $x$ to new features
$y$, which implies that the decision of machine learning model $f$
changes from $f(x)$ to $f(y)$. We assume that the user is satisfied if
they achieve utility $u_f(x,y) \geq \tau$, for some threshold $\tau  \in
\mathbb{R}$.\footnote{Mathematically, it is always possible to reduce to
the case that $\tau = 0$ by subtracting $\tau$ from the utility
function, but for simplicity we allow general $\tau$.} For instance, if
$f$ represents the score of one of the classes in a classification task,
with the sign of $f$ indicating whether the class is chosen by the
classifier, then the user may have a preferred class they would like to
be classified in. For example, the preferred class could be the class
for which the score is positive. This objective, which is closely
related to finding counterfactual explanations or adversarial examples
\citep{karimi2021survey}, can be described by
\begin{equation}\label{eqn:utility_class}
  u_f(x,y) := f(y)\ge 0.
\end{equation}
Alternatively, if $f$
predicts a credit risk score for the user, they might care about
increasing their score by some amount~$\tau$, which can be represented
by 
\begin{equation}\label{eqn:utility_diff}
  u_f(x,y) := f(y) - f(x) \geq \tau.
\end{equation}
And, as a third example, if $f$ outputs the probability of some event
and the user would like to increase (or decrease) that probability by a
certain percentage $p \times 100 \%$, their goal could be expressed as
\begin{equation}\label{eqn:utility_frac}
  u_f(x,y) := \frac{f(y)}{f(x)} \geq 1 + p
\end{equation}
(or $u_f(x,y) = \frac{f(x)}{f(y)} \geq 1/(1 - p)$).

In some of our formal results we will restrict attention to utility
functions that depend on $x$ and $y$ only via the decisions $f(x)$ and
$f(y)$ of $f$. This is a very natural restriction, which is satisfied by
all examples \eqref{eqn:utility_class}, \eqref{eqn:utility_diff} and
\eqref{eqn:utility_frac} given above.

\paragraph{Recourse Sensitivity}

Informally, we call an attribution method $\phi_f$ recourse sensitive if
the user can always achieve sufficient utility by moving in the
direction of the vector $\phi_f(x)$. We aim for a very permissive
definition, which covers all methods that can reasonably be said to
provide recourse, so if there are multiple such directions at an input
$x$, then we allow any direction; and if there is no such direction at
$x$, then $\phi_f(x)$ is allowed to be anything.

Formally, we assume the user is able to change their input $x$ to
an alternative input $y$ over at most some distance $\delta \geq 0$. 
This can be used to express that a feature like income cannot double in
a reasonable period of time. As discussed in the introduction, we
restrict attention to Euclidean distance for simplicity. We further
allow for additional constraints on the alternatives via a constraint
set $C(x)$. Thus, the set of attainable points for a user with original
input $x$ is
\[
  A(x) = \{y\in \mathcal{X}  \mid \|x - y\| \le \delta, y \in
  C(x) \}.
\]
The constraints $C(x)$ may express that the user is unable to change
some features in $x$ that they have no control over, like gender, age
group or location \citep{poyiadzi2020face,mothilal2020explaining}, and
we assume throughout that $x \in C(x)$, which means that the user always
has the option of not changing their features. It could also be the case
that the user can only change the features in a particular way (e.g.\
age can only increase) or that features can only be changed together,
for instance as described by an underlying causal model as in
consequential recommendations (see the introduction)
\citep{karimi2021survey}. In most of our results, the possible choices
for $C(x)$ that we will focus on are:
\begin{enumerate}[label=(\alph*)]
\item $C(x) = \domain$: the unrestricted case;
\label{en:constraint_full}
\item $C(x) = \{y \in \domain \mid \|x -
y\|_0 \le k\}$ with $\|z\|_0$ denoting the number of coordinates in $z$
that are non-zero: the sparse case in which
changing each feature requires effort and it is assume that the user can
change at most $k$ features; and \label{en:constraint_sparse}
\item $C(x) = \{y \in \domain \mid y = x +
\alpha z, \alpha \ge 0, z \in D\}$ for some set of directions $D
\subseteq \reals^d$: the case that the user is only allowed to move
in a restricted set of directions $D$. \label{en:constraint_directions} 
\end{enumerate}

The points around $x$ that are both attainable by the user and provide
sufficient utility to reach a given threshold~$\tau$ then correspond
to
\[
    T(x) = \{y \in A(x)  \mid u_f(x, y) \ge \tau\}.
\]
Our definition of recourse sensitivity now states that any attribution
at $x$ should point in the direction of some~$y$ that is in the set
$T(x)$:
\begin{defn}[Recourse Sensitivity]\label{def:recourse_sensitivity}
    Given a threshold $\tau \in \mathbb{R}$, 
    constraint function $C$, and model
    $f \in \funcs$, an attribution method $\phi_f \colon \domain \to
    \reals^{d}$ is called \emph{recourse sensitive} if
    \[
      \phi_f(x) = \alpha(y - x) \qquad \text{for some $\alpha > 0$ and
      $y \in T(x)$,}
    \]
    for all $x \in \domain$ for which $T(x)$ is non-empty.
\end{defn}
The case that $T(x)$ is empty corresponds to a user who has no options
for recourse, so no explanation could possible help them. In this case
we allow $\phi_f(x)$ to be arbitrary. 

\begin{remark}[Satisfying Recourse Sensitivity]\label{rem:satisfying_recourse_sensitivity}
It is clear that, in the absence of further requirements on $\phi_f$,
recourse sensitivity can trivially be satisfied by setting $\phi_f(x) =
y-x$ for some arbitrary $y \in T(x)$ whenever $T(x)$ is non-empty, and
setting $\phi_f(x) = 0$ otherwise. It is also satisfied by any
counterfactual explanation $\xcf$ that minimizes $\|\xcf - x\|$ subject
to the constraint that $u_f(x,\xcf) \geq \tau$ and $\xcf \in C(x)$. To
see this, note that, if $\|\xcf - x\| \leq \delta$, then $\xcf \in
T(x)$, so $\phi_f(x) = \xcf - x$ is a recourse sensitive choice. And if
$\|\xcf - x\| > \delta$, then $T(x)$ is empty and $\phi_f(x) = \xcf - x$
is also allowed.
\end{remark}

\subsection{Profile Picture Example}\label{sec:profpic_intro}

We will now illustrate the concept of recourse sensitivity in a more
concrete setting. Consider the following use case: a user has to upload
an official profile picture of themselves $x$ to a website to obtain a
personalized card, which grants them access to some
service.\footnote{This example was inspired by the third author's
frustrating attempts to obtain a new transport card for the Dutch
railways. In reality he could not figure out why his picture was
rejected.} The receiving party performs an automated check to verify if
there is enough contrast between the brightness of the person and the
background of the image, which is implemented by a function $f$ that
computes the squared difference between the average pixel values of the
background and the average pixel values of the person. The picture is
then accepted only if $f(x) \ge \lambda_{\text{thresh}}$ for some
threshold parameter $\lambda_{\text{thresh}}$. Users whose picture is
rejected want to submit a correct picture that is accepted,
which would correspond to the
utility function $u_f(x,y) = f(y) -\lambda_{\text{thresh}}\ge 0$.  The amount by which the
user is able to increase the contrast may be described by a suitable
choice of $\delta$, and optionally by describing constraints on how the
user can manipulate the image via $C(x)$. Two examples with
corresponding saliency maps are shown in
Figure~\ref{fig:example_profile_pics}.  Negative values indicate that
a part of the picture should be darker, while positive values
indicate that a part should be lighter. It can be seen that both
saliency maps indicate the parts of the profile picture that have to be
adjusted to increase the contrast, which makes them recourse sensitive:
For the rejected picture, increasing the contrast is a direction that
points towards sufficient utility, and is therefore recourse sensitive.
For the accepted picture, increasing the contrast would only strengthen the 
classification of the preferred class.
Further details about Figure~\ref{fig:example_profile_pics} are provided
in Appendix~\ref{sec:experiments}, and
this example is continued in Section~\ref{sec:profpic_cont}.

\begin{figure}[t]
    \centering
    \begin{tikzpicture}
        \node[draw=black] (accepted_example) at (0, 0)
            {
                \includegraphics[width=0.2\textwidth]{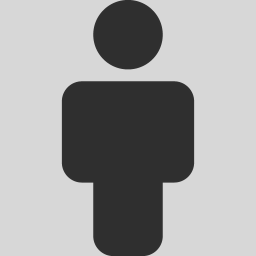}
                \includegraphics[width=0.2\textwidth]{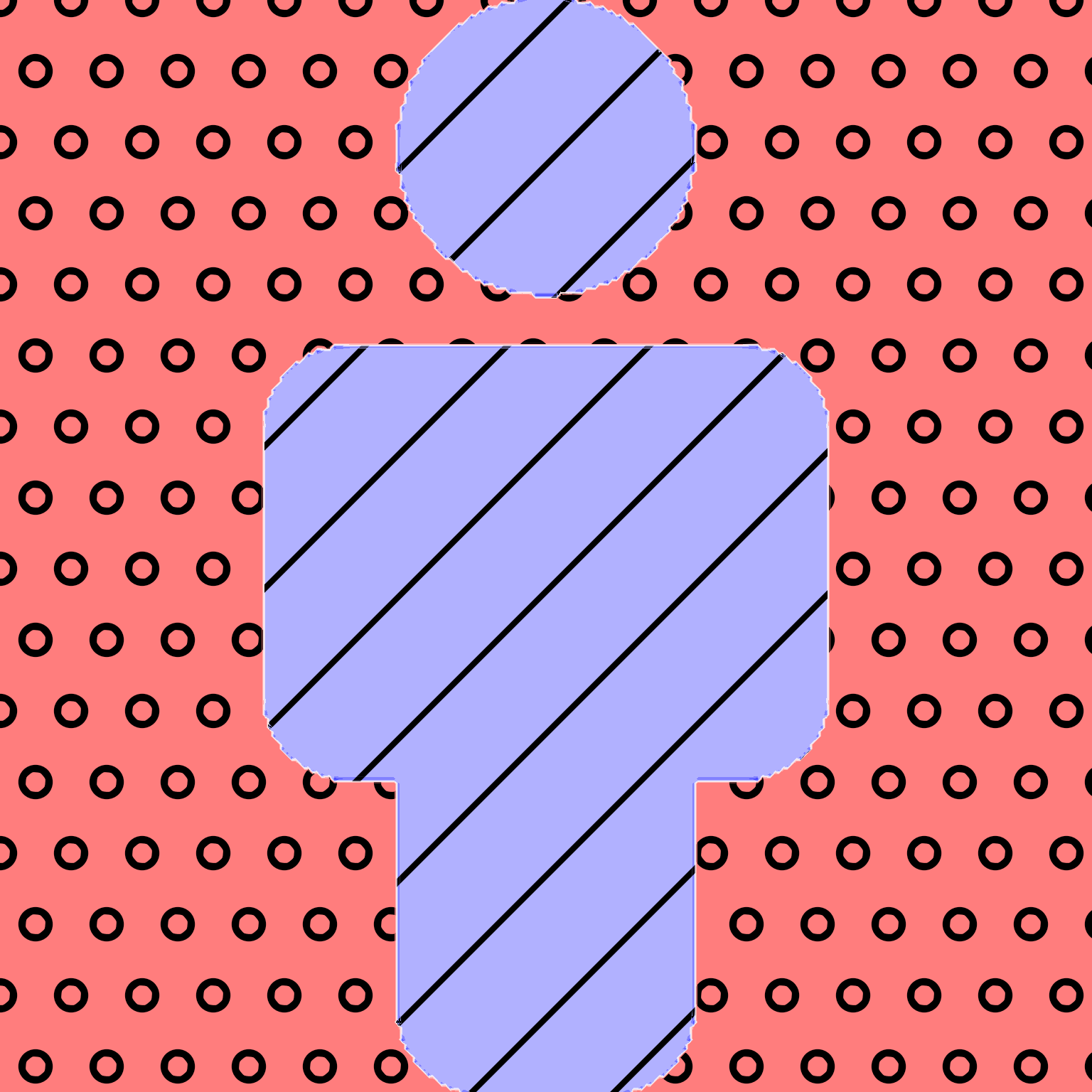}
            };
        \node[draw=black] (rejected_example) at (8, 0)
            {
                \includegraphics[width=0.2\textwidth]{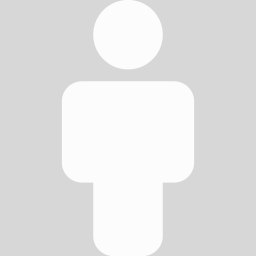}
                \includegraphics[width=0.2\textwidth]{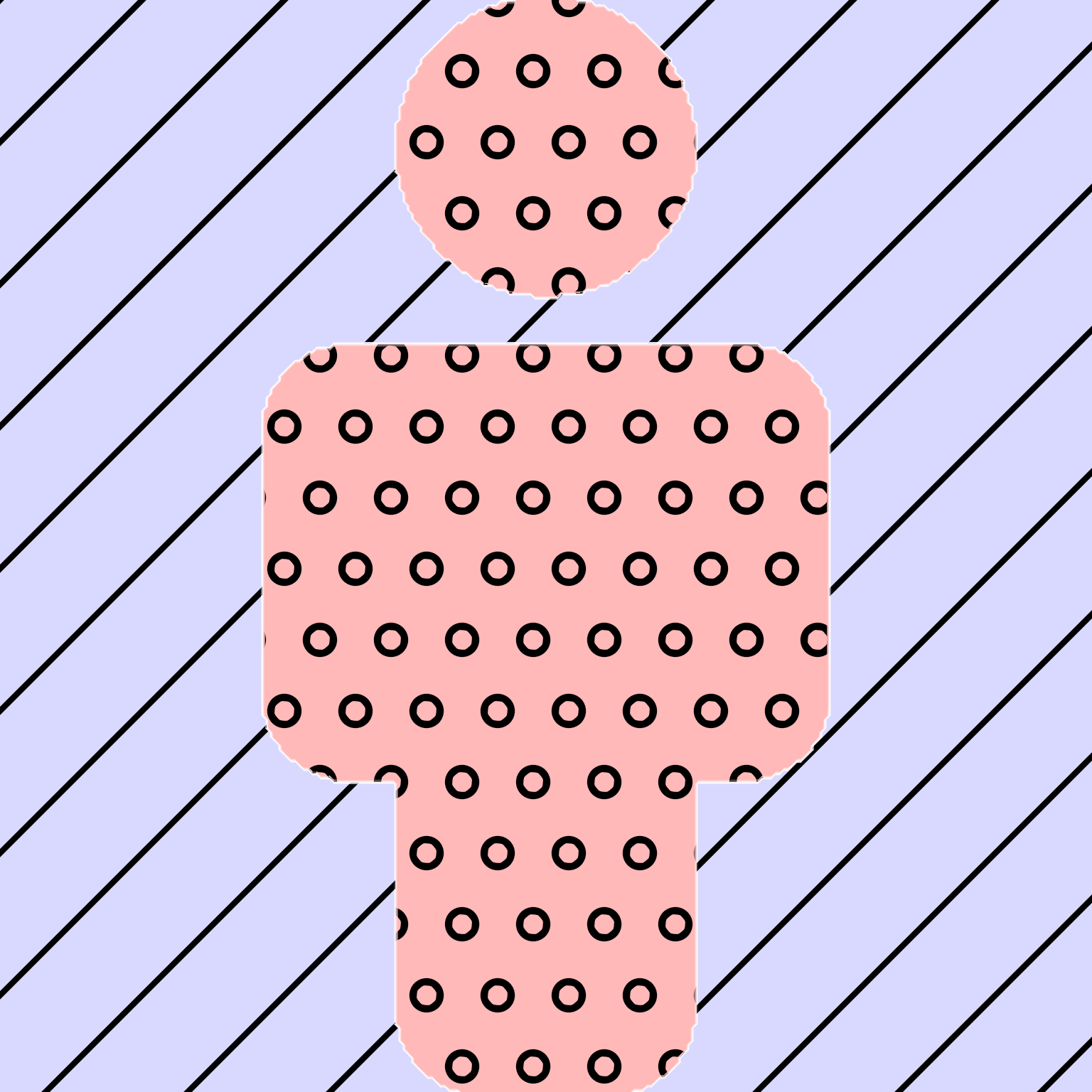}
            };

        \node[below=0.1 of accepted_example] (caption_1) {(a) Accepted
        profile picture};
        \node[below=0.1 of rejected_example] (caption_2) {(b) Rejected
        profile picture};
    \end{tikzpicture}
    \caption{Two examples of profile pictures (left in each subfigure) and their corresponding
    saliency maps (right in each subfigure). The red areas with circles
    indicate positive values, 
    and the blue areas with lines correspond to negative values. In both cases changing the picture in the direction of the
    saliency map will result in a larger contrast between the person and
    background, which indeed moves it further into/towards the accepted
    class.}
    \label{fig:example_profile_pics}
\end{figure}

\section{Impossibility of Recourse with Robustness in General}\label{sec:impossibility_results}

In this section we first present our main impossibility result, which
shows that no attribution method can be both recourse sensitive and
robust at the same time. Since no method can have both properties, it
follows that there must exist counterexamples for every existing
attribution method in which it violates at least one of the two. We
illustrate our main result by explicitly constructing such
counterexamples, both analytically for SmoothGrad, Integrated Gradients
and Counterfactual Explanations, and empirically for LIME, SHAP,
SmoothGrad and Integrated Gradients. At the end of the section we sketch
the idea behind the proof of our impossibility theorem. The full proof
is provided in Appendix~\ref{sec:proofs}.

\subsection{Main Impossibility Result}

We restrict our attention to utility functions that only depend on $x$ and
$y$ via the decisions of the machine learning model $f$, i.e.\ for which
there exists a function $\widetilde{u}$ such that $u_f(x, y) =
\widetilde{u}(f(x), f(y))$. This is a very natural restriction that
covers all examples from Section~\ref{sec:recourse_sensitivity}.

\begin{restatable}{theorem}{ImposResult}\label{thm:impossibility_result}
    Let $\delta > 0$ and $\tau \in \mathbb{R}$ be arbitrary, and let the constraint
    function $C(x)$ be any of the choices \ref{en:constraint_full},
    \ref{en:constraint_sparse} or \ref{en:constraint_directions} on
    p.\,\pageref{en:constraint_full}. Furthermore, assume the utility function
    $u_f$ is of the form $u_f(x, y) = \widetilde{u}(f(x), f(y))$, and
    that there exist $z_1, z_2 \in \reals$ for which
    $\widetilde{u}(z_1, z_2) \ge \tau$ and 
    $\widetilde{u}(z_1, z_1) < \tau$. Finally, assume that 
    $\domain\subseteq \reals^{d}$ contains a line segment $\ell$ of
    length strictly larger than $\delta$ and such that $\ell \subseteq
    C(x)$ for all $x \in \ell$. Then there exists a continuous function
    $f \colon \domain \to \reals$ for which no attribution method
    $\varphi_f$ can be both recourse sensitive and continuous.
\end{restatable}

The required existence of $z_1$ and $z_2$ rules out the trivial case
that the user can never achieve sufficient utility or will already
receive sufficient utility without changing their input $x$. The
existence of the line segment $\ell$ is used to ensure there is enough
room within the domain $\domain$ and the constraints $C(x)$ to construct
the counterexample. Although already very mild, the conditions of
Theorem~\ref{thm:impossibility_result} can potentially be generalized,
as discussed in Section~\ref{sec:conclusion}. As an example of a setting
that is covered by the theorem, we may for instance consider the
classification setting in which the user wants to be classified in some
preferred class, with a full domain and without any constraints. The
conditions of the theorem are then all satisfied, for any $z_1 < 0 <
z_2$, and the result simplifies to:
\begin{cor}[Unconstrained Classification]
  Suppose $\domain = \reals^d$, $C(x) = \reals^d$, $u_f(x,y) =
  f(y)$, $\tau \in \mathbb{R}$ and $\delta > 0$.
  Then there exists a continuous function $f \colon \domain \to \reals$ for
  which no attribution method $\varphi_f$ can be both recourse sensitive
  and continuous.
\end{cor}

\subsection{Analytical Examples}

We proceed to construct explicit analytical counterexamples $f$ and
utility functions $u_f$ for the SmoothGrad and Integrated Gradients
attribution methods. Since both attribution methods are always
continuous, it follows from Theorem~\ref{thm:impossibility_result} that
they cannot always be recourse sensitive, which is what the
counterexamples demonstrate. Both examples are in dimension $d=1$ and we
will assume that there are no constraints, i.e.\ $\domain = C(x) =
\reals$.

\begin{example}[SmoothGrad]\label{ex:x2smoothgrad}
\label{exmp:smoothgrad}
    Consider the function $f(x) = x^2$, the utility 
    $u_f(x, y) = f(y) - f(x)$ and $\tau \in \mathbb{R}$, which expresses that the user wants to
    increase $f$ by at least $\tau$. The attribution given 
    to each point by the SmoothGrad procedure will be
    \begin{align*}
        \varphi_{f}^{\text{SG}}(x)= \mathbb{E}_{a \sim N(0, \sigma^2)}
        \left[  f'(x + a)\right]  =  \mathbb{E}_{a \sim N(0, \sigma^2)}[2x + 2a] = 2 x
    .\end{align*}
    Here, $N(0, \sigma^2)$ denotes the normal distribution with mean $0$
    and some specified variance parameter $\sigma^2 > 0$. In almost all
    points $x$ this attribution indeed points in a direction that
    increases $f(x)$. However, in the point $x = 0$ the attribution is
    $0$, which does not provide meaningful recourse, because it does not
    tell the user that they can in fact increase $f(x)$ by changing $x$.
    Whenever $\delta \geq \sqrt{\tau}$ the user would be able to achieve
    a sufficient increase in utility by moving in any direction from
    $x=0$, and this violates recourse sensitivity.
\end{example}

\begin{example}[Integrated Gradients]
\label{exmp:integrated_gradients}
    We examine $f(x) = e^{- x^2}$ and $u_f(x, y) = \frac{f(x)}{f(y)} \ge \tau$ for some 
    threshold $\tau > 1$, which means the user wants to decrease $f(x)$
    by a factor of at least $\tau$. Also choose $\delta$ such that
    $\delta \ge \sqrt{\log \tau} $. If the user starts in $x=0$, then this 
    is possible by moving far away enough in both directions. 
    Indeed if you would move to some $y \in \mathbb{R}$ from $x=0$ with 
    $\delta \geq |y| \ge \sqrt{\log(\tau)}$, it is possible
    to decrease $f$ by the requested fraction $\tau$ within a
    $\delta$-distance of $x=0$, because then
    \begin{align}\label{eq:int_grad_x0}
        u_f(0, y) =  \frac{1}{e^{-y^2}}  = e^{y^2} \ge e^{\log(\tau)} = \tau
    .\end{align}

    We can explicitly calculate the attributions of the Integrated
    Gradients method, which depends on a baseline point $x^0$:
    \begin{align*}
        \varphi_{f}^{\text{IG}}(x) = (x - x^0) \int\limits_{0}^{1} 
            f'\left( x^0 +  \alpha (x - x^0)  \right) \, \mathrm{d}\alpha 
            = f(x) - f(x^0) = f(x) - 1,
    \end{align*}
    where we have chosen $x^0 = 0$ as the baseline. Note that
    $\varphi_f^{\text{IG}}(x) < 0$ for all $x \neq 0$ and
    $\varphi_f^{\text{IG}}(0)=0$. In this case, recourse is provided 
    for $x < 0$, as moving towards the negative side is the fastest way
    to decrease the
    output of $f$.
    However, recourse is not provided for $x=0$, for which $\varphi_f^{IG}(0) = 0$, but it
    should be non zero because of \eqref{eq:int_grad_x0}. 
    There are more points which do not get recourse in this example. The points with 
    $x>\delta$ can only decrease their output by moving to the right. To decrease the output
    by moving to the left would require moving a distance that is larger than $\delta$. 
    To see this, note that the function $f$ is symmetric and to decrease $f( x ) $ by moving
    to the left from some $x > \delta$ would at least require that $y < -x$, which is already
    a distance of $2\delta$ away from $x$. 
    As noted before, the attribution
    is always negative, so we find that $\varphi_f^{IG}$ is also not recourse sensitive 
    for the points $x > \delta$.
\end{example}

Let us look at one last example. We will show that in some cases for
binary classification, a counterfactual explanation cannot be
continuous. 

\begin{example}[Counterfactual Explanations]\label{exm:counterfactua}
 For this example
    set $\mathcal{X}=\mathbb{R}^{2},  u_f(x,y) = f(y), \tau = 0, \delta >1$ and 
    \begin{align*} 
        f(x) = \|x\| - 1
    .\end{align*}
    In this example, the points within the unit circle are classified in the negative 
    class and the points outside it as the positive class. The utility is such that if you are
    inside the circle, you want to move out of it. We can construct 
    a simple counterfactual method by setting
    \begin{align*}
        x^{\text{CF}}(x) = \argmin_{\|y\|\ge 1} \| x - y\| 
    .\end{align*}
    This optimization problem is well defined and has a unique solution for every point $x\neq 0$. 
    There, the problem is still well-defined, but no unique solution exists. For the points 
    with $\|x\| \ge 1$, the minimizer will be the point itself, and for $\|x\| < 1$, 
    it will be $x /\|x\|$. Using these counterfactuals we can build a recourse sensitive 
    attribution function by setting
    \begin{align*}
        \varphi_f(x) =
        \begin{cases}
            0 & \text{ if } \|x\| \ge 1, \\
            \frac{x}{\|x\|} - x & \text{ if $0 < \|x\| < 1$.}
        \end{cases}
    \end{align*}
    For the point $x=0$ we now have many options, as the attribution is
    allowed to point to any point on the unit circle. However, no choice
    we make can produce an attribution that is continuous at $x=0$,
    because the limit of $\varphi_f(x)$ is different when $x$ approaches
    $0$ along different lines.
\end{example}

\begin{figure}[t]
    \centering
    \begin{tikzpicture}
        \node[inner sep=0pt, draw=black] (orig_pic_1) at (0,0)
            {\includegraphics[width=0.15\textwidth]{figures/user-icon-edited-47.png}};
        \node[inner sep=0pt, draw=black] (attribution_1_1) at (3, 0)
            {\includegraphics[width=0.15\textwidth]{figures/vanilla_gradients_qd_user_icon_47.png}};
        \node[inner sep=0pt, draw=black] (attribution_1_2) at (5.75, 0)
            {\includegraphics[width=0.15\textwidth]{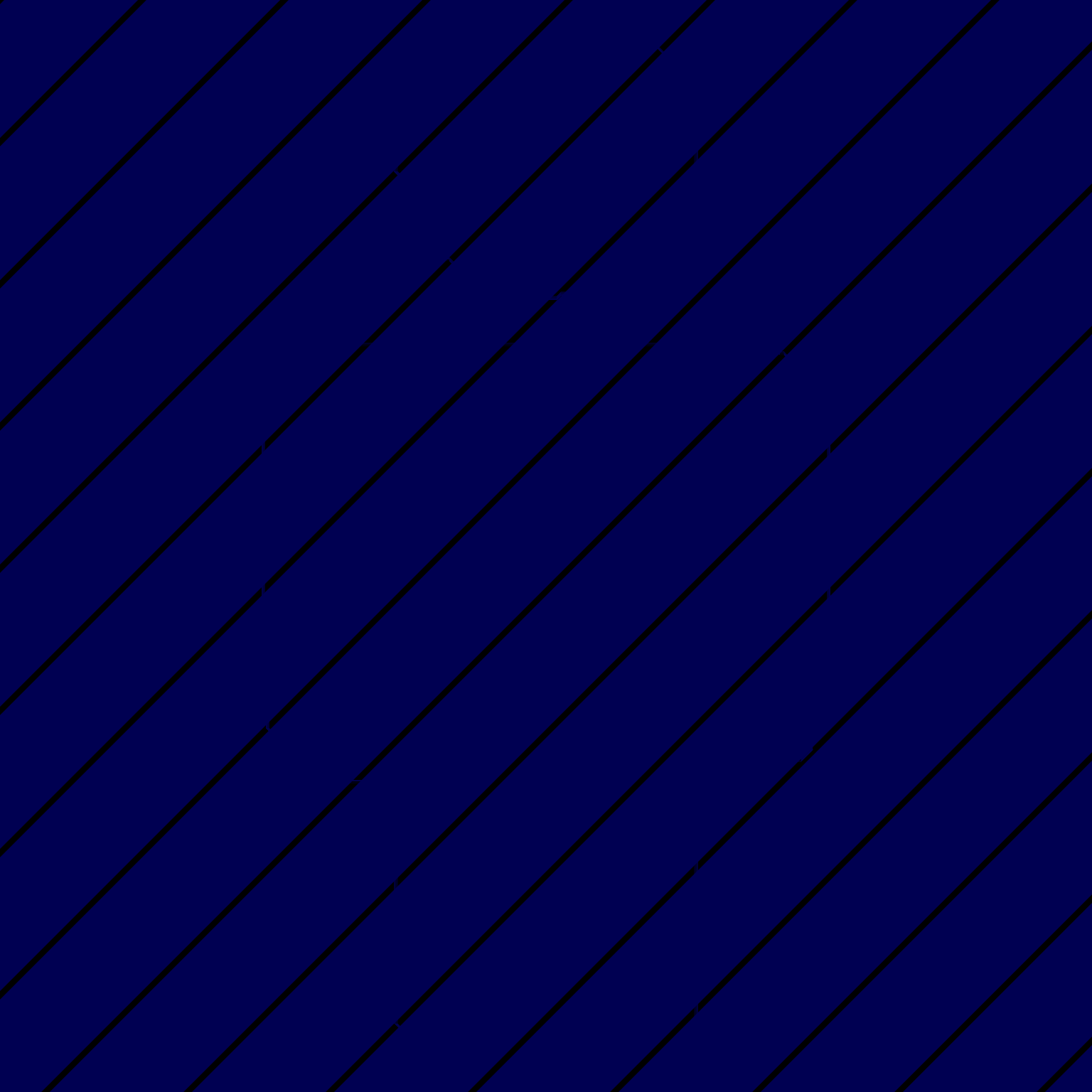}};
        \node[inner sep=0pt, draw=black] (attribution_1_3) at (8.5, 0)
            {\includegraphics[width=0.15\textwidth]{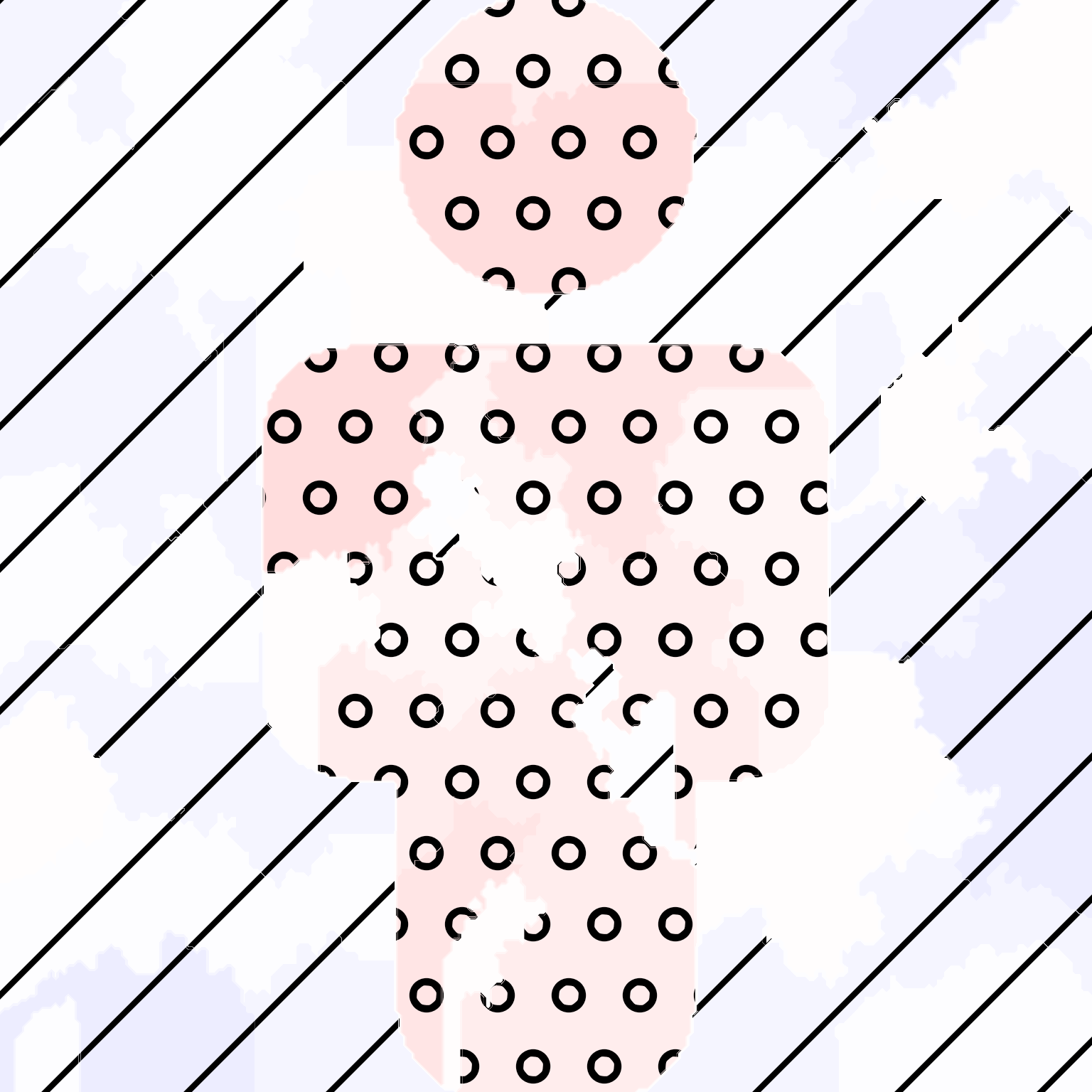}};
        \node[inner sep=0pt, draw=black] (attribution_1_4) at (11.25, 0)
            {\includegraphics[width=0.15\textwidth]{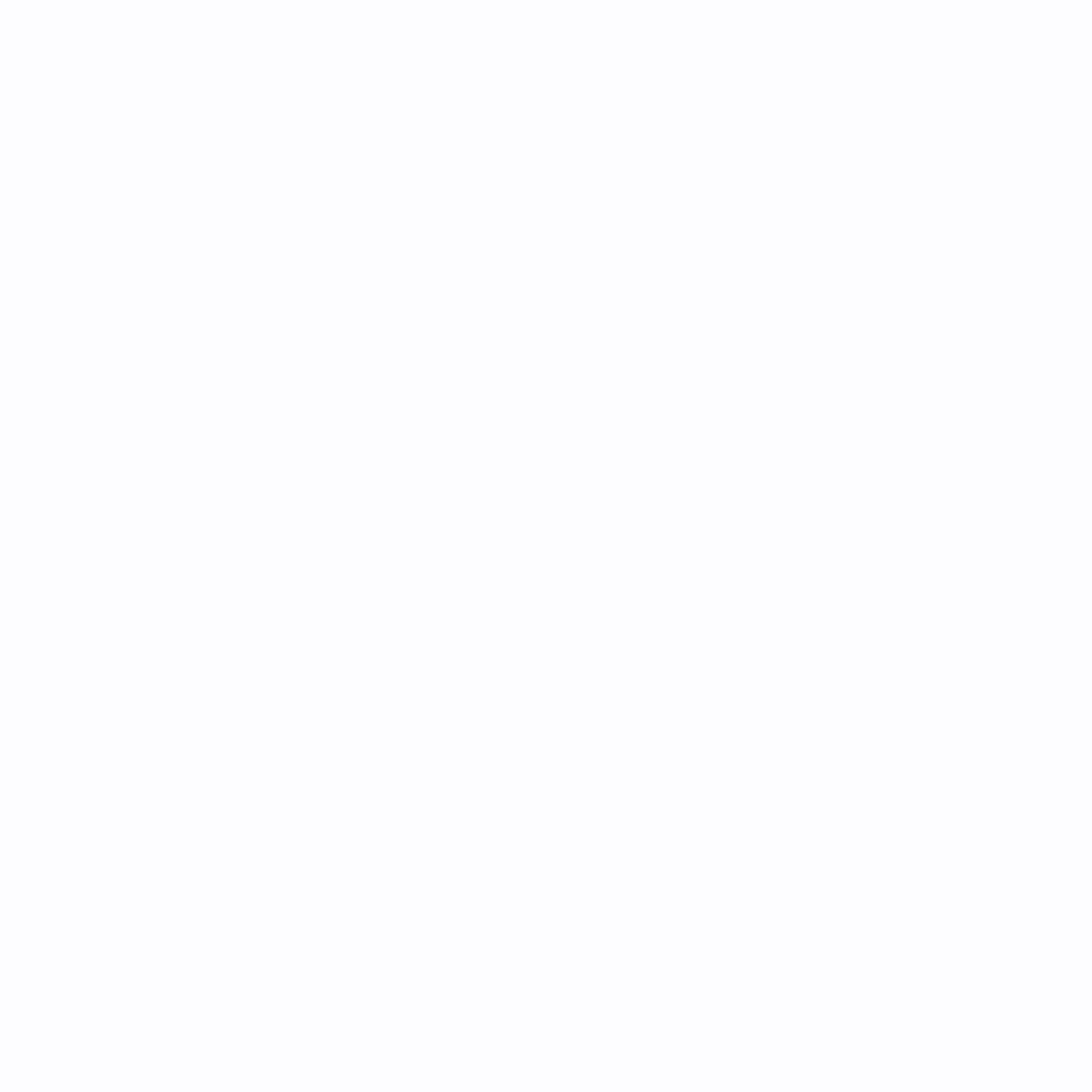}};

        \node[inner sep=0pt, draw=black] (orig_pic_2) at (0, -2.75)
            {\includegraphics[width=0.15\textwidth]{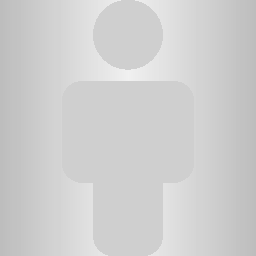}};
        \node[inner sep=0pt, draw=black] (attribution_2_1) at (3, -2.75)
            {\includegraphics[width=0.15\textwidth]{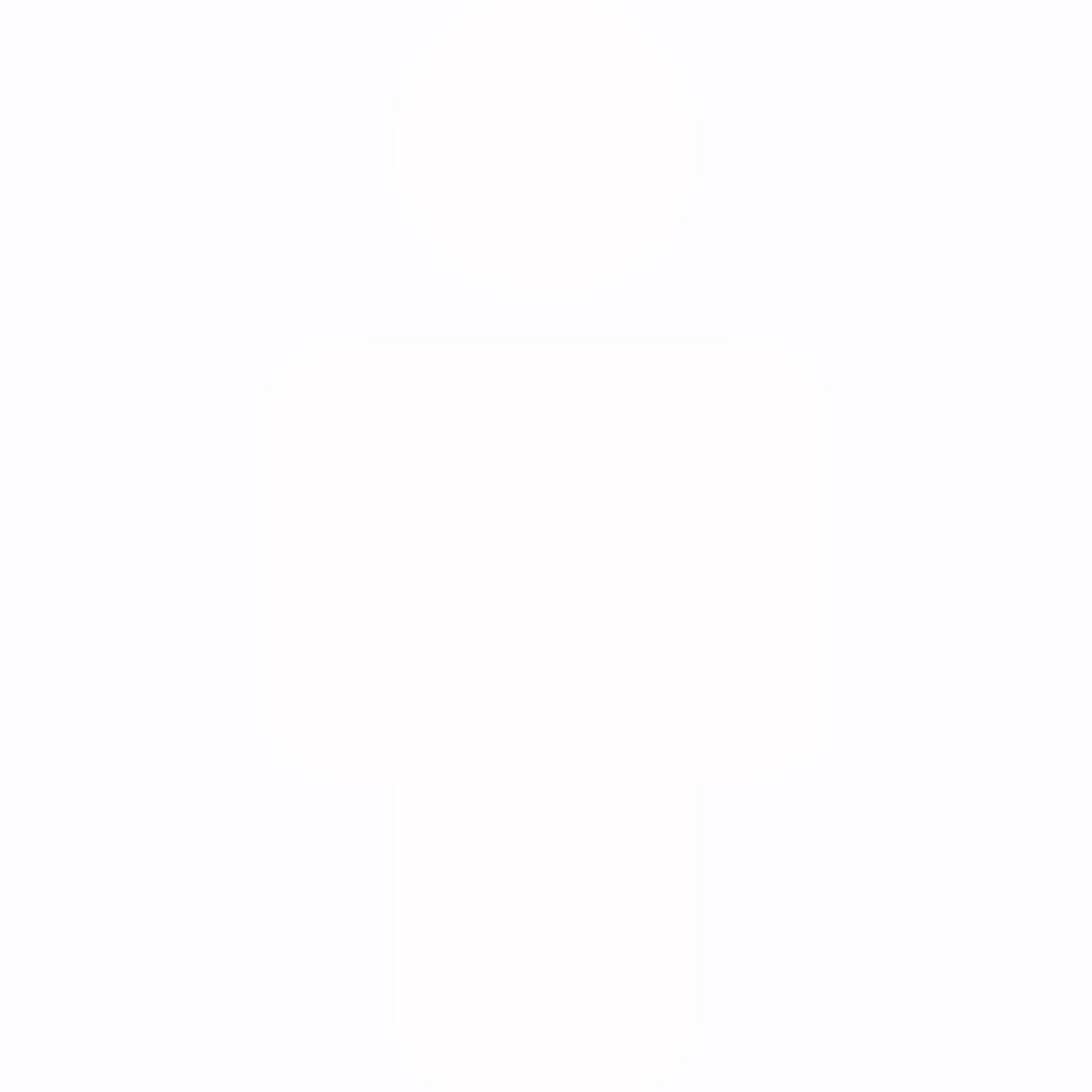}};
        \node[inner sep=0pt, draw=black] (attribution_2_2) at (5.75, -2.75)
            {\includegraphics[width=0.15\textwidth]{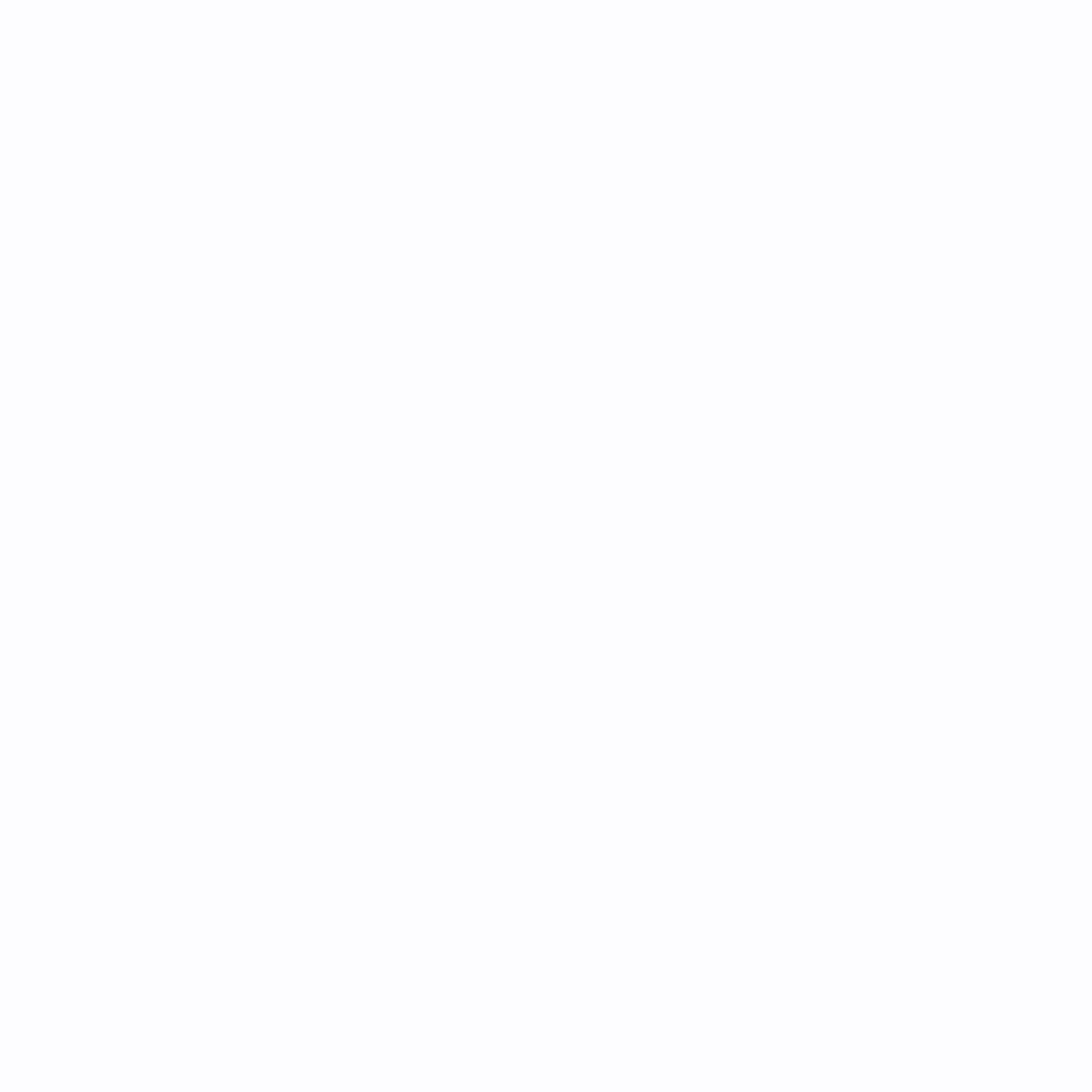}};
        \node[inner sep=0pt, draw=black] (attribution_2_3) at (8.5, -2.75)
            {\includegraphics[width=0.15\textwidth]{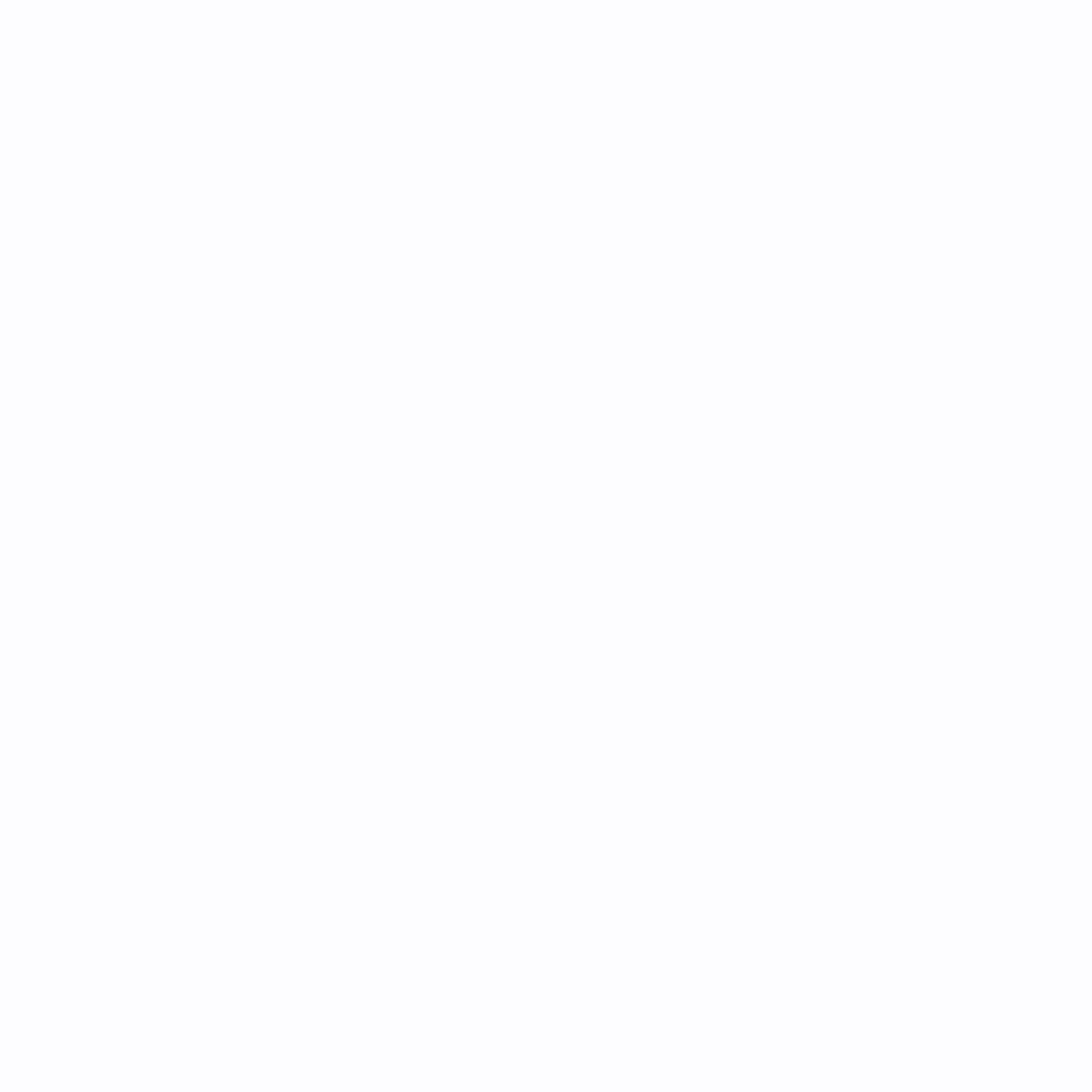}};
        \node[inner sep=0pt, draw=black] (attribution_2_4) at (11.25, -2.75)
            {\includegraphics[width=0.15\textwidth]{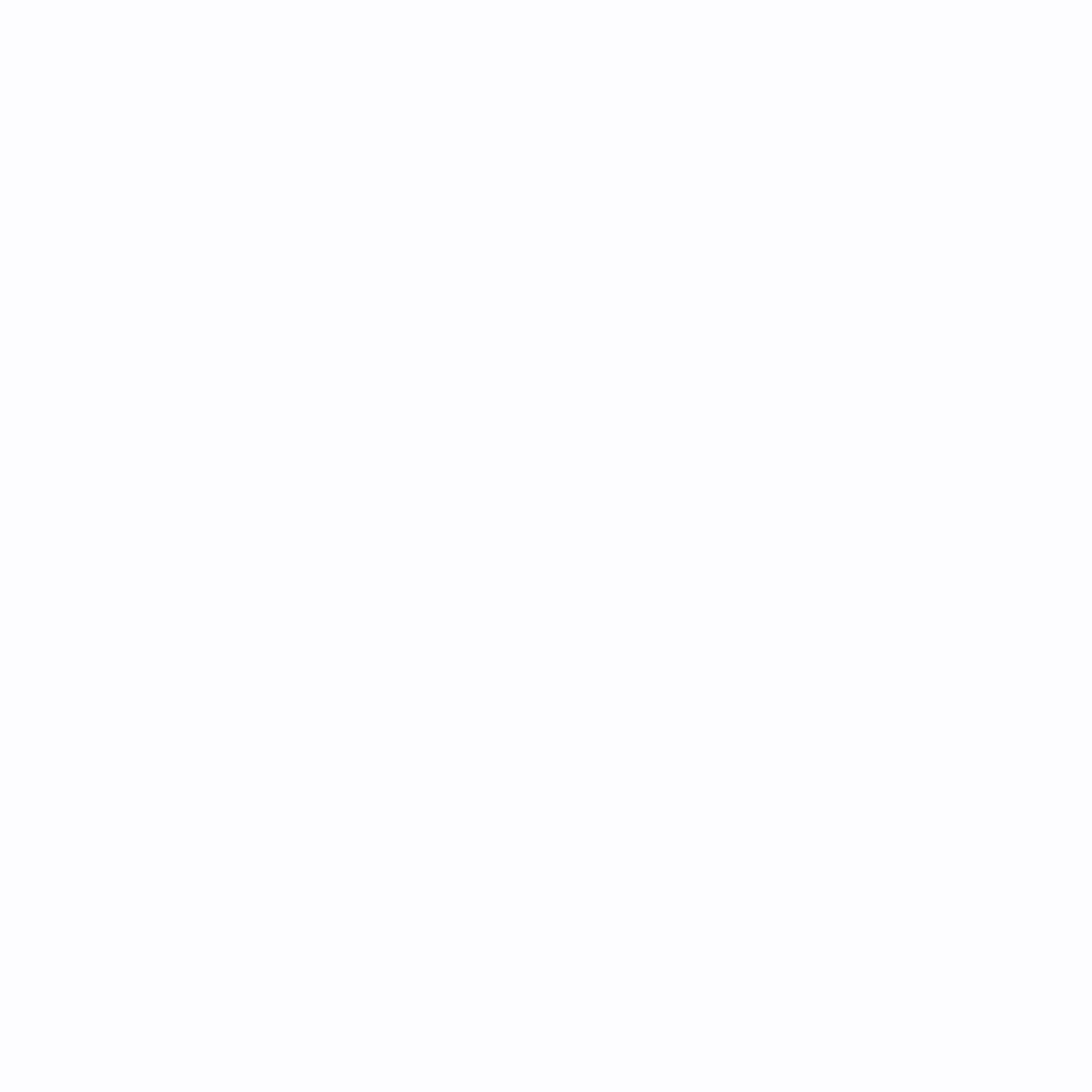}};

         \node[inner sep=0pt] (colorbar) at (13.4, -1.375) 
            {\includegraphics[width=0.06\textwidth]{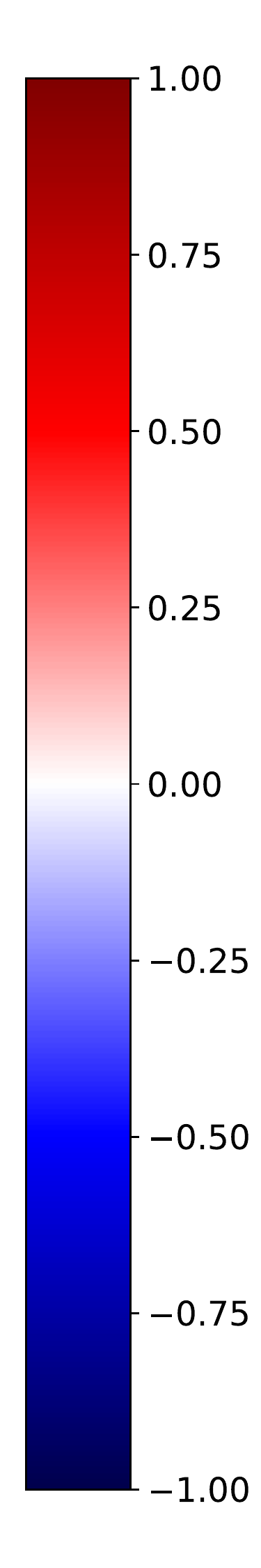}};

        \node (text_1) at (0, 1.5) {Profile Picture};
        \node (text_2) at (3, 1.5) {Gradient Methods};
        \node (text_3) at (5.75, 1.5) {LIME manual};
        \node (text_4) at (8.5, 1.5) {LIME auto};
        \node (text_5) at (11.25, 1.5) {SHAP};

        \draw (1.5, 1.15) -- (1.5, -3.9);
    \end{tikzpicture}

    \caption{Saliency maps by different methods for two profile
    pictures that are rejected by the model.}
    \label{fig:no_recourse_pictures}
\end{figure}

\subsection{Profile Picture Example (Continued)}\label{sec:profpic_cont}

In this section we complement our analytical examples from the previous
section with empirical examples in which well-known attribution methods
are recourse insensitive. We will demonstrate this for LIME, SHAP,
SmoothGrad and Integrated Gradients, which can all be used to generate
saliency maps. Our setup continues the profile picture example from
Section~\ref{sec:profpic_intro}. Since the gradient of $f$ is linear in
$x$ in this example, both SmoothGrad and Integrated Gradients simplify,
and both coincide with the Vanilla Gradients method $\phi_f(x) = \nabla
f(x)$ \citep{simonyan2014deep}. We will refer to all three together as
`Gradient Methods'. LIME for images depends on a partition of the image
into superpixels. We consider two variants. One variant is `LIME
manual', in which we provide the indices of the person as one
super-pixel, and the indices of the background for a second super-pixel.
The other variant is `LIME auto', which uses the default segmentation
algorithm of LIME to obtain the super-pixels. See
Appendix~\ref{sec:experiments} for further details about the
experimental set-up. 

Two example profile pictures with their corresponding saliency maps are
shown in Figure~\ref{fig:no_recourse_pictures}. Both examples are
rejected by the classifier. In both cases the user could change the
decision by making the background darker (provided that $\delta$ is
large enough). In the top row two of the methods, the Gradient Methods
and LIME auto, do provide recourse. For LIME auto it is difficult to
see, but the region of the person has a positive value and the
surrounding region a negative value. Changing the picture accordingly
would result in a darker background compared to the person, and
therefore satisfies the requirement of recourse sensitivity. In the
bottom example, the picture is also rejected. The attributions all show
a flat saliency map, for which moving in the direction $\phi_f(x)$ has
no effect on the value of $f$, so none of the attribution methods is
recourse sensitive. There is also significant disagreement between the
saliency maps: LIME manual assigns a large positive value to all pixels,
SHAP a very small almost unnoticeable negative value, and the remaining
methods attribute $0$ to all pixels.

\subsection{Proof Idea for Theorem~\ref{thm:impossibility_result}}\label{sec:proof_idea}

Having illustrated the implications of
Theorem~\ref{thm:impossibility_result} in the previous sections, we now
explain the idea behind its proof. We consider again the setting of
Example~\ref{ex:x2smoothgrad}, with $f(x) = x^2$, $u_f(x,y) = f(y) -
f(x)$ and $\delta \ge \sqrt{\tau} >0$. Let $L$ denote the interval of
points $x$ where recourse is possible by moving to the left, and let $R$
denote the interval of points where recourse is possible by moving to
the right. As illustrated in Figure~\ref{fig:x_squared_proof}, we then
find that neither interval fully contains the other, but they overlap on
$L \intersection R = \left[ \frac{\tau - \delta^2}{2\delta},
\frac{\delta ^2  - \tau}{2\delta} \right]$. Since the attribution has to
be negative to the left of this overlap and positive to the right of the
overlap, it has to go through zero somewhere inside the overlap (by
continuity and the intermediate value theorem), but this is not allowed
because the attribution is not allowed to be zero anywhere where
recourse is possible. Therefore, no attribution can be continuous and
recourse sensitive simultaneously.  In the actual proof of
Theorem~\ref{thm:impossibility_result}, we follow a similar argument,
but for a different function $f$, whose existence is guaranteed by the
assumptions of the Theorem. To generalize from one dimension to higher
dimensions, we then embed this function along the line segment~$\ell$. 

\begin{figure}[t]
    \centering
    \begin{tikzpicture}[decoration=brace, line width=0.3mm]
        \draw[->] (-2.5, 0) -- (2.5, 0); 
        \draw[->] (0, -0.2) -- (0, 2); 
        
        \draw[-{Bracket[width=4mm,line width=1pt,length=1.5mm]},
            color={rgb:red,34;green, 178;blue, 178}]
            (-2.5, -0.7) node[left] {$L$} -- (1, -0.7);
        \draw[{Bracket[width=4mm,line width=1pt,length=1.5mm]}-,
            color={rgb:red,178;green, 34;blue, 34}]
            (-1, -1.1) -- (2, -1.1) node[right] {$R$};
        
        \draw[dotted] (0, 0.4) -- (1, 0.4);
        \draw[dotted] (-1.5, 0.9) -- (0, 0.9);
        \draw[decorate, decoration={mirror}]  (0, 0.4) -- (0, 0.9) 
            node[pos=0.5, xshift=0.3cm] {$\tau$};
    
        \draw[decorate, decoration={aspect=0.6}] (-1.5, 2.1) -- (1, 2.1)
            node[pos=0.6, yshift=0.3cm] {$\delta$};
        
        \draw[dotted] (-1.5, 2.1) -- (-1.5, 0.9);
        \draw[dotted] (1, 2.1) -- (1, 0.4);

        \node at (-1.5, 0) {|};
        \node at (-1.5, -0.3) {$y$};
        
        \node at (1, 0) {|};
        \node at (1, -0.3) {$x$};

        \draw[domain=-2.5:2.5, smooth, variable=\x, black] 
            plot ({\x}, {0.4*\x*\x});

    \end{tikzpicture}
    \caption{$L$ and $R$ sets for $f(x)=x^2$.} 
    \label{fig:x_squared_proof}
\end{figure}
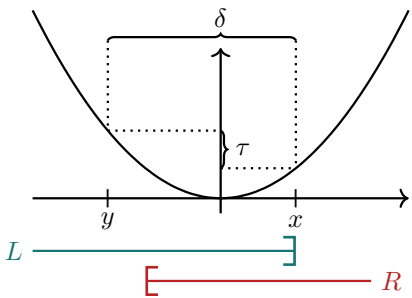

\section{Discussion}\label{sec:discussion}

In this section we pause to reflect on the implications of our
impossibility result, and suggest possible ways to circumvent it.

\paragraph{How Common is the Problem?}

As is clear from the proof sketch in Section~\ref{sec:proof_idea},
recourse sensitivity and robustness get into conflict already for simple
quadratic functions in 1D. And things are even worse: the same problem
arises for any other function with partially overlapping $L$ and $R$
sets, so that on one part of the line recourse is provided by pointing
to the left and on another part of the line recourse requires pointing
to the right. We should therefore be prepared to run into problematic
instances for most non-trivial learning models used in practice: e.g.\
linear models in which we add a quadratic feature, decision trees,
neural networks, etc. In Sections~\ref{sec:suff_cond_main} and
\ref{sec:suff_and_nec_conditions_in_one_dimension} we provide more
formal insights into the class of problematic functions $f$ by providing
conditions that are sufficient and/or necessary to combine recourse
sensitivity and robustness. 

\paragraph{Should We Prioritize Recourse Sensitivity or Robustness?}

In the context of algorithmic recourse, providing recourse is the
primary goal and robustness is a secondary consideration. Faced with a
choice between the two, it is therefore clear that we should prioritize
recourse. However, it may be possible to (partially) salvage robustness
in special cases, which seems important because it would be undesirable
to have explanations that jump around unnecessarily. We proceed to
explore this for counterfactual explanations. 

\subsection{Workarounds for Counterfactual Explanations}
\label{sec:workaroundsCounterfactuals}

As pointed out in Remark~\ref{rem:satisfying_recourse_sensitivity},
counterfactual explanations are always recourse sensitive. In
Section~\ref{sec:suff_cond_main} it will be shown that, under weak
conditions, their robustness fails only at points $x$ for which the
counterfactual projection $\xcf$ is not unique. This suggests two
natural, but ultimately unsatisfactory ways to work around our
impossibility result:

\paragraph{1. Robustness at Most Points}

If the points of non-uniqueness are sufficiently uncommon, then we may
simply ignore them and accept a lack of robustness in such exceptional
cases. It is indeed tempting to believe that non-uniqueness might be
rare enough. For instance, in the binary classification setting from
\eqref{eqn:utility_class}, the set of points with a non-unique
projection has measure 0 (see discussion below
Theorem~\ref{thm:classification_sufficient}), which means that it is in
a sense small compared to the ambient space $\domain$. But,
unfortunately, this is not sufficient to conclude that the set of
affected users would also be small: around any point of discontinuity of
$\phi_f$ there exists a whole neighborhood of users $x$ who can find a
nearby alternative $x'$ that would result in a very different
explanation. 
In extreme cases it is even possible that every user $x$ in
$\domain$ would be close to a point of discontinuity of $\phi_f$. We
would therefore need to have a stronger restriction on the set of
discontinuities before we can dismiss them as sufficiently uncommon.

\paragraph{2. Restrict to Very Simple Models}

A radical way to avoid discontinuities altogether is to restrict
attention to very simple models $f$ and constraint sets $C(x)$ for which
the projection $\xcf$ is always unique. For instance, if the user wants
to be classified to a preferred class and there are no constraints, then
this would be satisfied by functions $f$ that are linear in the original
features. Although effective, this approach is so restrictive that it
seems unworkable, because non-uniqueness will quickly reappear, for
instance if we add a quadratic feature as discussed at the start of the
section.

\subsection{Work-arounds by Changing the Explainability Task}

More appealing options to avoid impossibility become available if we
allow ourselves to change the explanation task. Some preliminary
thoughts in this direction are as follows:

\paragraph{1. Linearizing with Abstract Features}

It may sometimes be an option to provide explanations not in terms of
the original features $x$ but in terms of transformed (typically more
abstract) features $z = g(x)$ for some mapping $g$. If $f$ is linear in
$z$, then this would allow using a simple model after all.
Appendix~\ref{app:higher_order_attribution} provides a detailed example
to illustrate how this might work out.

\paragraph{2. Set-valued Explanations} 

As discussed in Section~\ref{sec:proof_idea}, attributions cannot be
continuous and recourse sensitive, because they may have to communicate
different options in adjacent regions and there is no continuous way to
transition between the options. A way around this may be to communicate
not one, but many or all possible directions that provide recourse as
a set $S_f(x)$. This is analogous to the definition of the
subdifferential of a convex function, which represents the set of all
possible tangents. Set-theoretic notions of continuity such as
hemi-continuity or continuity with respect to the Hausdorff metric
\citep{aubin2009set} could then be used to rephrase robustness in terms
of continuity of $S_f(x)$ instead of continuity of $\phi_f(x)$, making
robustness easier to satisfy. 

\section{Sufficient Conditions for Recourse with Robustness}\label{sec:suff_cond_main}

The impossibility result in Theorem~\ref{thm:impossibility_result}
implies that there exist continuous functions $f$ for which no
attribution method can both provide recourse and be robust. But this may
still be possible if we restrict attention to specific functions $f$
that are somehow nice enough. For instance, as mentioned in the
introduction, linear classifiers do allow robust and recourse sensitive
attribution functions when the goal is to move to a preferred class in
binary classification. In this section, we will first formalize a
generalization of this result to a slightly larger class of functions
for the binary classification setting. Then we will extend the result to
more general utilities. The proofs for this section can be found in
Appendix~\ref{sec:suff_conditions}. 

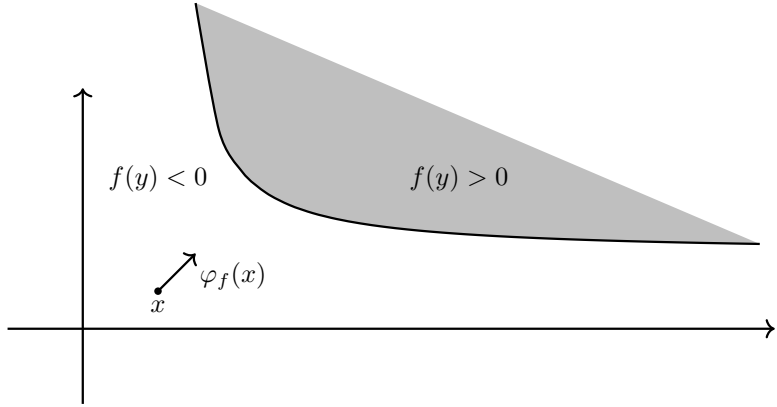
\begin{figure}[t]
    \centering
    \begin{tikzpicture}[line width=0.3mm]
        \draw[->] (-5, 0) -- (5.2, 0) node[right] {};
        \draw[->] (-4, -1) -- (-4, 3.2) node[above] {};
        
        \draw[domain=-2.5:5, smooth, variable=\x, black, fill=gray!50!white]
            plot ({\x}, {1 / (\x + 2.8) + 1});

        \node at (1, 2)  {$f(y) >0$};
        \node at (-3, 2) {$f(y) < 0$};

        \node at (-3, 0.5) [circle, fill, inner sep=1pt] {};
        \node at (-3, 0.3) {$x$};

        \draw[->] (-3, 0.5) -- (-2.5, 1) {};
        \node at (-2, 0.7) {$\varphi_f(x)$};

    \end{tikzpicture}
    \caption{Assumptions of Theorem~\ref{thm:classification_sufficient}
        illustrated.}
    \label{fig:thm_6_assumption}
\end{figure}

The following result shows that counterfactual explanations are both
recourse sensitive and robust for binary classification with a preferred
class if the counterfactual projections $\xcf$ are always uniquely
defined:
\begin{restatable}{theorem}{SuffClass}\label{thm:classification_sufficient}
    Consider the binary classification setting without constraints with
    $u_f(x,y) = f(y), \tau=0$ and $C(x) = \domain$, let $\delta> 0$ be
    arbitrary and take $f\colon \mathcal{X} \to \mathbb{R}$ to be any
    continuous function. If the set $U = \{y \in \mathcal{X}  \mid f(y) \ge 0\}$
    is convex, then the attribution method 
    \begin{align*}
        \varphi_f(x) \coloneqq \argmin_{y \in U}\|y - x\| - x = P_U(x) - x
    \end{align*}
    is well defined, and it is both recourse sensitive and continuous.
\end{restatable}
Theorem~\ref{thm:classification_sufficient} covers linear functions $f$,
but applies more broadly. For instance, any concave and continuous
function $f$ will satisfy its requirements. However, these requirements
are still very restrictive, as only very simple classifiers will be
concave.

To understand the conditions of
Theorem~\ref{thm:classification_sufficient}, we note that continuity of
$f$ is used to ensure that $U$ is closed, so that projections always
exist. In this case the convexity condition on $U$ is equivalent to the
assumption that projections onto $U$ are always unique, provided that
$\domain$ is convex. Thus the conditions of the theorem may also be
understood as requirements to ensure unique projections. In light of the
discussion in Section~\ref{sec:workaroundsCounterfactuals}, we further
remark that, without the convexity assumption on $U$, the set of points
$x$ with a non-unique projection onto $U$ still has measure $0$, as is
shown by \citet{Erdos1945} for general closed sets and $\domain =
\reals^d$. 

To extend Theorem~\ref{thm:classification_sufficient} to more general
utility functions and arbitrary constraints $C(x)$, we need to look at sets of the form 
\begin{align*}
    U(x) \coloneqq \{y  \in \mathcal{X}\mid  u_f(x, y) \ge \tau \}
    \cap C(x)
\end{align*}
and consider counterfactual explanations that project onto these sets.
Now the set that is projected on changes with $x$, so we will need more complicated assumptions to ensure that the projections still change
 continuously with respect to $x$. Intuitively, this will be the case if
 the sets $U(x)$ themselves vary continuously with $x$.
It turns out that the right type of set-valued continuity is
\emph{hemi-continuity}, which ensures that the sets $U(x)$ cannot
explode, implode or shift suddenly, when varying $x$. The definition and
relevant properties of hemi-continuity are reviewed in
Appendix~\ref{sec:suff_conditions}. Compared to
Theorem~\ref{thm:classification_sufficient} we are further able to
weaken the convexity 
assumption on $U(x)$ to the requirement that a unique projection 
exists for $x$ instead of for all points. This leads to the following
result, also proved in Appendix~\ref{sec:suff_conditions}:
\begin{restatable}{theorem}{SuffGen}\label{thm:general_sufficient}
    Let $\delta>0, \tau \in \mathbb{R}$, $f \colon \mathcal{X} \to \mathbb{R}$ be
    a function, $C(x)$ constraint sets and $u_f(x, y) $ a
    utility function with the following properties:
    \begin{enumerate}
        \item The set-valued map $U(x)$ is hemi-continuous; and 
        \item $U(x)$ is a closed set for every $x \in \domain$.
    \end{enumerate}
    Then there exists at least one attribution method
    \begin{align*}
        \varphi_f(x) \in \big(\argmin_{y \in U(x)}\|y - x\|\big) -x,
    \end{align*}
    and any such method is recourse sensitive. Moreover, $\varphi_f$ will
    be continuous on the restriction of $\domain$ to points $x$ for
    which $P_{U(x)}(x) := \argmin_{y \in U(x)} \|y - x\|$ is unique.
\end{restatable}
To see that Theorem~\ref{thm:classification_sufficient} follows from
Theorem~\ref{thm:general_sufficient}, we note that, if $U(x) = U$ is the
same for all $x$, then $U(x)$ is always hemi-continuous. In addition,
$U$ is the pre-image of a closed set under $f$, and will therefore be
closed if $f$ is continuous. Finally, uniqueness of all projections (and
therefore continuity) is a consequence of convexity of $U$.

Theorem~\ref{thm:general_sufficient} also covers other cases of
interest. For instance, it is sufficient to require the following:
\begin{enumerate}
    \item The function $u_f(x,y)$ is continuous;
    \item For each $x \in \mathcal{X}$, the function $y \mapsto u_f(x,y)$ is concave;
    \item The domain $\mathcal{X}$ is compact.
\end{enumerate}
These conditions imply that all $U(x)$ are convex and compact (i.e.\
closed and bounded), and consequently all projections are unique. It can
also be checked that they ensure that $U$ is hemi-continuous. All
requirements of Theorem~\ref{thm:general_sufficient} are therefore
satisfied and a continuous and recourse sensitive attribution function
exists. 

Another example, which is covered by
Theorem~\ref{thm:general_sufficient}, but not by the previous sufficient
conditions, is the following: suppose the user wants to at least double
the outcome of the model, which can be expressed by taking $u_f(x,y) =
\frac{f(y)}{f(x)}$ and $\tau = 2$. Then the following model $f$
satisfies the conditions:
\begin{cor}\label{cor:suff_example}
    Let $\delta > 0$, $\tau=2$, $\mathcal{X}= \mathbb{R}^{d} \setminus
    \{0\}$, $C(x) = \domain$, $f(x) = e^{b\|x\|}$ for some $b >0$  and
    $u_f(x,y) = \frac{f(y)}{f(x)}$. Then $\varphi_f$ as defined in
    Theorem~\ref{thm:general_sufficient} is uniquely defined, recourse
    sensitive and continuous.
\end{cor}
To prove Corollary~\ref{cor:suff_example}, it can be shown that $U(x) =
\{y \in \mathcal{X} \mid \|y\|\ge \|x\| + \frac{\ln(2)}{b}\} $ is
hemi-continuous. The sets $U(x)$ are also closed, and the projections
onto $U(x)$ are unique for every $x$, because we excluded $0$ from our
domain. Thus, all the requirements of
Theorem~\ref{thm:general_sufficient} hold and a continuous and recourse
sensitive attribution function exists. 

\section{Single Feature Recourse Sensitivity: Exact Characterization}
\label{sec:suff_and_nec_conditions_in_one_dimension}

In this section, we provide an in-depth study of the case that users are
only able to change a single feature, which corresponds to the sparse
constraint \ref{en:constraint_sparse} from
Section~\ref{sec:recourse_sensitivity} with $k=1$. In this case, we are
able to provide an exact characterization of the functions $f$ for which
a continuous, recourse sensitive attribution function can exist. We
first restrict ourselves to the one-dimensional case $d=1$, and then
generalize the result to the multi-dimensional case $d \geq 1$. All
proofs of the results in this section can be found in
Appendix~\ref{sec:proofs}. To formulate our results,
we will need the concept of separated sets, which corresponds to the
intuition that the sets are disjoint and have at least one point in
between them everywhere:
\begin{defn}
    \label{def:separated_sets}    
    Two sets $A, B \subseteq \mathcal{X}$ are called \emph{separated} if 
    $\cl(A) \cap B = \varnothing$ and $A \cap \cl(B) = \varnothing$. 
\end{defn}
An equivalent definition is that there exist open sets $U, V \subseteq
\mathcal{X}$ such that $A \subseteq U, B \subseteq V$ and $U \cap V =
\varnothing$. 

\subsection{One Dimension}
In this subsection we assume that $\mathcal{X}\subseteq \mathbb{R}$ is
one-dimensional. Define the following three sets
\begin{align*}
    L &= \{x \in \mathcal{X} \mid \text{ there exists some }y \in 
        [x-\delta, x) \cap C(x) \text{ with }u_f( x, y) \ge \tau\},   \\
    R &= \{x \in \mathcal{X} \mid \text{ there exists some }y \in 
        (x, x + \delta] \cap C(x)\text{ with }u_f( x, y) \ge \tau\}, \\
    O &= \{x \in \mathcal{X} \mid x \in C(x) \text{ and } u_f(x,x) \ge \tau\} 
.\end{align*}
Similarly to Section~\ref{sec:proof_idea}, these sets are the feasible points for
which recourse is obtainable by moving to the left, to the right or by
doing nothing, respectively. The following result now tells us that a
continuous recourse sensitive attribution function can exist if and only
if we can decompose $L, R$ and $O$ in a particular way: 
\begin{restatable}{theorem}{oneDimRes}\label{thm:2a}
  Let $\delta > 0, \tau \in \mathbb{R}, f \colon \mathcal{X} \to \mathbb{R}$  and $C(x)$ 
  be arbitrary, 
  then there exists a continuous
  recourse sensitive attribution function $\varphi_f$ for $f$ if and only if
  there exist $\widetilde{L} \subseteq L, \widetilde{R} \subseteq
  R$  and $\widetilde{O} \subseteq O$ such that 
  \begin{enumerate}
      \item $\widetilde{L} \cup \widetilde{R} \cup \widetilde{O} = L \cup R \cup O$;
      \item $\widetilde{L}$ and $\widetilde{R}$ are separated;
      \item $\cl( \widetilde{O} ) \cap \tL = \varnothing$ and 
          $\cl(\widetilde{O}  ) \cap \tR = \varnothing $.
        \label{cond:2aO}
  \end{enumerate}
\end{restatable}

\paragraph{Sufficient $\tL$ and $\tR$ sets}

Theorem~\ref{thm:2a} refers to the existence of any $\tL, \tR$ and $\tO$
that satisfy its conditions, but we will proceed to show that it
sufficient to check separatedness only for a restricted number of
choices for $\tL, \tR$ and $\tO$, which may even reduce to a single
case. For simplicity, we will assume that $O = \varnothing$, so
Condition~\ref{cond:2aO} is automatically satisfied, but the result can
be generalized to general $O$ as well. To describe the choices for $\tL$
and $\tR$, it will be helpful to decompose $L$ and $R$ into the sets
$\Lintervals = \{L_i \mid i \in \I\}$ and $\Rintervals = \{R_j \mid j
\in \J\}$ of (maximal) intervals they contain, where these intervals may
be open or closed on either side. Thus $L = \Union_{i \in \I} L_i$ and
$R = \Union_{j \in \J} R_j$, and every two distinct intervals $A,B \in
\Lintervals$ (or $A,B \in \Rintervals$) are separated; otherwise they
could be joined into a single larger interval $A \union B$. We further
note that in general the number of intervals in $L$ or $R$ may be
uncountable, so the index sets $\I$ and $\J$ may be uncountable as well.
Since splitting an interval in two will never result in two separated
sets, we only need to make decisions about whole intervals. Moreover,
most of these choices are forced, and we can define the remaining
possibilities in terms of the following index sets:
\begin{alignat*}{4}
  \tI &= \{i \in \I &&\mid \neg &&\exists j \in \J &&\text{ such that } L_i
  \subseteq R_j\},\\
  \tJ &= \{j \in \J &&\mid \neg &&\exists i \in \I &&\text{ such that } R_j
  \subseteq L_i\},\\
  \tK &= \{i \in \I &&\mid &&\exists j \in \J &&\text{ such that } L_i = R_j
  \}.
\end{alignat*}

\begin{restatable}{theorem}{SuffSets}\label{thm:2b}
  Let $\delta > 0, \tau  \in \mathbb{R}, f \colon \mathcal{X} \to \mathbb{R}$ and $C(x)$ 
  be arbitrary,
  and let $u_f$ be any utility function with  
  $u_f(x,x) < \tau$ for all  $x \in \mathcal{X}$. Then there exists a continuous
  recourse sensitive attribution function $\varphi_f$ for $f$ if and only if
  there exists a partition $\{\tK_1,\tK_2\}$ of $\tK$ such that the sets
  \begin{equation}\label{eqn:2b}
    \tL = \Big(\Union_{i \in \tI} L_i\Big) \union \Big(\Union_{i \in
    \tK_1} L_i\Big)
    \quad \text{and} \quad
    \tR = \Big(\Union_{j \in \tJ} R_j\Big) \union \Big(\Union_{i \in
    \tK_2} L_i\Big)
  \end{equation}
  are separated.
\end{restatable}
Thus, the only choice in selecting $\tL$ and $\tR$ is how to divide the
intervals indexed by $\tK$, which appear both in $L$ and $R$. In
particular, if $\tK$ is empty, then so are $\tK_1$ and $\tK_2$, and we
only need to check separatedness for a single choice of $\tL$ and $\tR$.

\subsection{Higher Dimensions}
It is possible to extend Theorem~\ref{thm:2a} to higher dimensions, i.e.\
$\mathcal{X} \subseteq \mathbb{R}^{d}$, whenever
the user has control over only one feature at the same time. The constraint set
in this case becomes ${C(x) = \{y \in \mathcal{X} \mid \|x - y\|_0\le 1\}}$.  
Analogously with the one dimensional case, we first define sets on which the attribution 
is allowed to be positive or negative in the $i$'th feature. These sets are
\begin{align*}
    L^{i} &= \{x \in \mathcal{X} \mid \text{ there exists some } 
        \mathcal{X} \ni y = x - \alpha e_i,
        \alpha \in (0, \delta], \text{ such that } u_f( x, y) \ge \tau\},  \\
    R^{i} &= \{x \in \mathcal{X} \mid \text{ there exists some } 
        \mathcal{X}\ni y = x + \alpha e_i, 
        \alpha \in (0, \delta], \text{ such that } u_f( x, y) \ge \tau\}, \\
    O &= \{x \in \mathcal{X}  \mid  u_f(x,x) \ge \tau\} 
.\end{align*}

Where $e_i$ denotes the $i$'th basis vector of the standard basis. If the sets
$L^{i}$, $R^{i}$ and $O$ can be decomposed in a way similar to the decomposition
in Theorem~\ref{thm:2a}, then a continuous recourse sensitive
attribution function 
exists in the higher dimensional case. 

\begin{restatable}{theorem}{HigherDim}\label{thm:higher_dim_2}
  Let $\delta > 0, \tau \in \mathbb{R}$, $C(x) = \{y \in \mathcal{X}  \mid \|x - y\|_0 \le 1\} $ 
  and $f \colon \mathcal{X} \to \mathbb{R}$ be arbitrary.
  Then there exists a continuous
  recourse sensitive attribution function $\varphi_f$ for $f$ if and only if
  there exist $\widetilde{L}^{i} \subseteq L^{i}$ and $\widetilde{R}^{i} \subseteq
  R^{i}$ for all $i=1, \ldots, d$ and $\widetilde{O} \subseteq O$
  such that 
  \begin{enumerate}
      \item $\widetilde{O} \cup \bigcup_{i=1}^{d} (\tL^{i} \cup \tR^{i})
          = O \cup \bigcup_{i=1}^{d} (L^{i} \cup R^{i})$;
      \item All sets in $\{\widetilde{L}^{1}, \widetilde{R}^{1}, \ldots, \tL^{d}, \tR^{d}\}$
        are pairwise separated;
      \item $\cl( \widetilde{O} ) \cap \tL^{i} = \varnothing$ and 
          $\cl(\widetilde{O}  ) \cap \tR^{i} = \varnothing $ for all $i=1,\ldots, d$.
  \end{enumerate}
\end{restatable}

\section{Conclusion}\label{sec:conclusion}

We showed that there are machine learning models for which it is
impossible for any attribution method to be both recourse sensitive and
continuous (i.e.\ robust). This was illustrated by examples exhibiting
the problem for specific attribution methods, and we gave an exact
characterization of the set of problematic models for the case where the
user is only able to make sparse changes that affect a single feature.
It was further shown how, by making restrictive assumptions on $f$ that
satisfy certain sufficient conditions, it is possible to circumvent our
impossibility result.

We view our work as a contribution to establishing solid foundational
definitions for explainable machine learning. To obtain these in the
context of providing recourse, it would be of particular interest to
follow up on possible solutions to work around our impossibility result,
for instance along the lines discussed in Section~\ref{sec:discussion}.
Another direction for future work would be to extend the
characterizations from
Section~\ref{sec:suff_and_nec_conditions_in_one_dimension} to the case
where the user can change multiple features. This would pose significant
new technical challenges, because, in contrast to the single-feature case, there are then an
infinite number of directions that an attribution can point to. In
addition, the very general definition of the utility function results in
very unstructured spaces of possible directions. It may therefore be
needed to specialize to particular utility functions to make progress.
Finally, we remark that we defined recourse sensitivity using the
Euclidean distance, but our proofs hardly use any of its special properties.
It should therefore be possible to extend our result to other distances,
such as (weighted) combinations of $\ell_p$ norms or weighted Manhattan
distance \citep{karimi2021survey}, or to a setting where the norm is
replaced by a (possibly asymmetrical) cost mechanism or causal
mechanism. 

\newpage

\acks{%
  The authors would like to thank Joris Bierkens for suggesting to add
  Theorem~\ref{thm:classification_sufficient} and Royi Jacobovic for
  pointing out Berge's Maximum Theorem, which is key to the proof of
  Theorem~\ref{thm:general_sufficient}. De Heide and Van Erven were
  supported by the Netherlands Organization for Scientific Research
  (NWO) under grant numbers 019.202EN.004 and VI.Vidi.192.095,
  respectively.
}

\bibliography{eaibib}


\newpage

\appendix

\section{Details of Experimental Set-up}\label{sec:experiments}
All the code to reproduce the experiments and figures in this paper can be found 
in a GitHub
repository\footnote{\href{https://github.com/HiddeFok/recourse-robust-explanations-impossible}{github.com/HiddeFok/recourse-robust-explanations-impossible}}.
All experiments were run locally on an Apple MacBook Pro M1 13", 2020
with 8GB of RAM. 

\subsection{Profile Picture Toy Dataset}
A total of $53$ gray scale figures were created from the \textbf{User Icon}
picture, found on \url{www.iconarchive.com}.
\footnote{The icon is provided for free for non-commercial 
use.} Each figure consists of two components, the person and 
a background. The figures have varying contrasts between these two components. 
We labeled each figure by hand according to this contrast. A figure with high 
enough contrast is labeled as ``Accepted'', while low contrast results in 
a ``Rejected'' label. The labeling was done in such a way that a perfect classifier 
exists, which is based on the quadratic difference between the mean pixel value 
of the person and the background. In the following expressions, $x \in \mathbb{R}^{N}$ 
denotes the vectorized version of a picture of size $N = wh$, where $w$ and $h$ 
are the width and height of the picture. The classification function is given by
\begin{align*}
    f(x) = \left( \frac{1}{|I_{\text{per}}|}\sum_{i \in I_{\text{per}}} x_i -
    \frac{1}{|J_{\text{back}}|}\sum_{j \in J_{\text{back}}} x_j \right)^2 
.\end{align*}
Where, $I_{\text{per}}$ denotes the indices of the pixels belonging to the person, 
and $J_{\text{back}}$ contains the indices of the  background. A figure is accepted 
if $f(x) \ge \lambda_{\text{thresh}}$ for some threshold parameter 
$\lambda_{\text{thresh}}$. By increasing the threshold from the minimum value of all quadratic 
differences to the maximum value, the parameter with the highest accuracy was chosen. 
This lead to the choice
$\lambda_{\text{thres}}=5961.34$, which achieved perfect accuracy across both classes.  

Several attribution methods were applied to each figure with this $f$. The methods 
used were Vanilla Gradients \citep{simonyan2014deep}, SmoothGrad 
\citep{smilkov2017smoothgrad}, Integrated Gradients \citep{sundararajan2017axiomatic}, 
LIME \citep{ribeiro2016should} and SHAP \citep{lundberg2017unified}. 
In Figure~\ref{fig:experiment_results} six example pictures with their attributions 
are displayed. 
The attribution methods based on gradients were calculated
analytically.  
The attributions for the Vanilla Gradients, SmoothGrad and Integrated Gradients
are given by
\begin{align*}
    \varphi_{f}^{\text{VG}}(x)_k &= 
        \begin{cases}
            \frac{2}{|I_{\text{per}}|}\left( \frac{1}{|I_{\text{per}}|}
                \sum_{i \in I_{\text{per}}} x_i - \frac{1}{|J_{\text{back}}|}
            \sum_{j \in J_{\text{back}}} x_j \right) & \text{ if } k \in I_{\text{per}},  \\
            \frac{-2}{|J_{\text{back}}|}\left( \frac{1}{|I_{\text{per}}|}
                \sum_{i \in I_{\text{per}}} x_i - \frac{1}{|J_{\text{back}}|}
            \sum_{j \in J_{\text{back}}} x_j \right) & \text{ if } k \in J_{\text{back}}, 
        \end{cases} 
    \\
    \varphi_{f}^{\text{SG}}(x)_k &= \mathbb{E}_{a \sim N(x, \sigma^2 I_d)}
        \left[ \nabla f(a)_k\right]  = \varphi_{f}^{\text{VG}}(x)_k,
    \\
    \varphi_{f}^{\text{IG}}(x)_k &= (x_k - x^{0}_k) \int_{0}^{1} \nabla 
        f\left( x^{0} + t (x - x^{0}) \right)_k \, \text{d}t  
        = \varphi_{f}^{\text{VG}}(\frac{1}{2}(x + x^{0}))_k
,\end{align*}
where $x^{0}$ is some baseline picture. We choose $x^{0}$ to be the picture with all 
pixel values set to~0.

For LIME \citep{ribeiro2016should} and SHAP \citep{lundberg2017unified},
the libraries provided by their respective authors were
used\footnote{Library for LIME: \url{https://github.com/marcotcr/lime},
library for SHAP: \url{https://github.com/slundberg/shap}}.
The LIME package is provided under the BSD 2-Clause License and SHAP is provided 
under the MIT License. For LIME we used version 0.2.0.1 and for SHAP version 0.40.0.
The default parameters were used, unless specified otherwise. For the LIME method, 
superpixels are needed to create the attribution. We used two different methods to
find these superpixels. In the `LIME manual' method we manually supplied
two superpixels: the first superpixel consists of the person and the
second superpixel is the background. In the `LIME auto' method we used
the default segmentation algorithm. Finally, for some of the picture
manipulation we used the scikit-image \citep{scikit-learn} package,
version 1.0, under the BSD 3-Clause License\footnote{Scikit-image
package: \url{https://github.com/scikit-image/scikit-image}}.

\begin{figure}[ht]
    \centering
    \begin{tikzpicture}
        \node[inner sep=0pt, draw=black] (orig_pic_1) at (0,0)
            {\includegraphics[width=0.15\textwidth]{figures/user-icon-edited-5.png}};
        \node[inner sep=0pt, draw=black] (attribution_1_1) at (3, 0)
            {\includegraphics[width=0.15\textwidth]{figures/vanilla_gradients_qd_user_icon_5.png}};
        \node[inner sep=0pt, draw=black] (attribution_1_2) at (5.75, 0)
            {\includegraphics[width=0.15\textwidth]{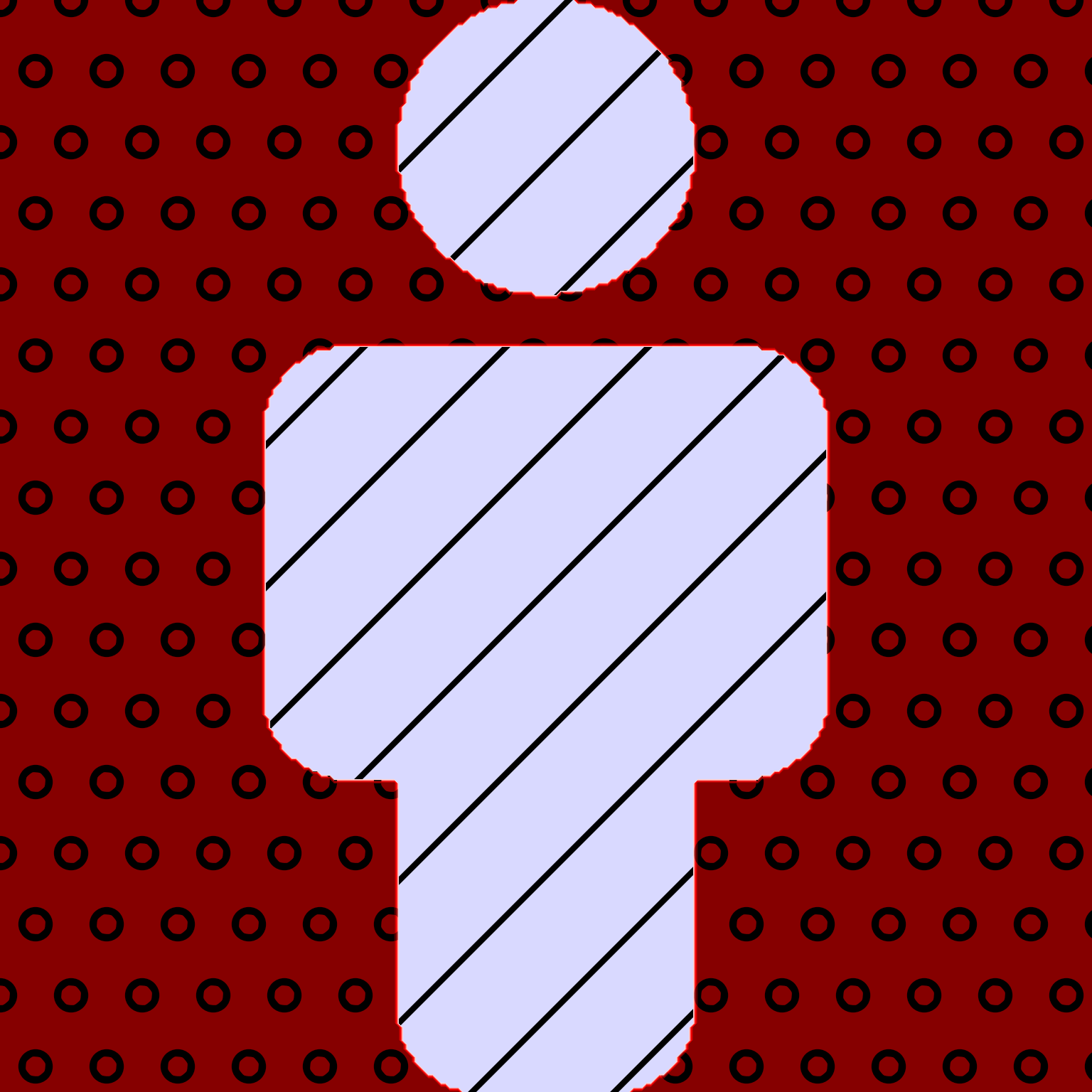}};
        \node[inner sep=0pt, draw=black] (attribution_1_3) at (8.5, 0)
            {\includegraphics[width=0.15\textwidth]{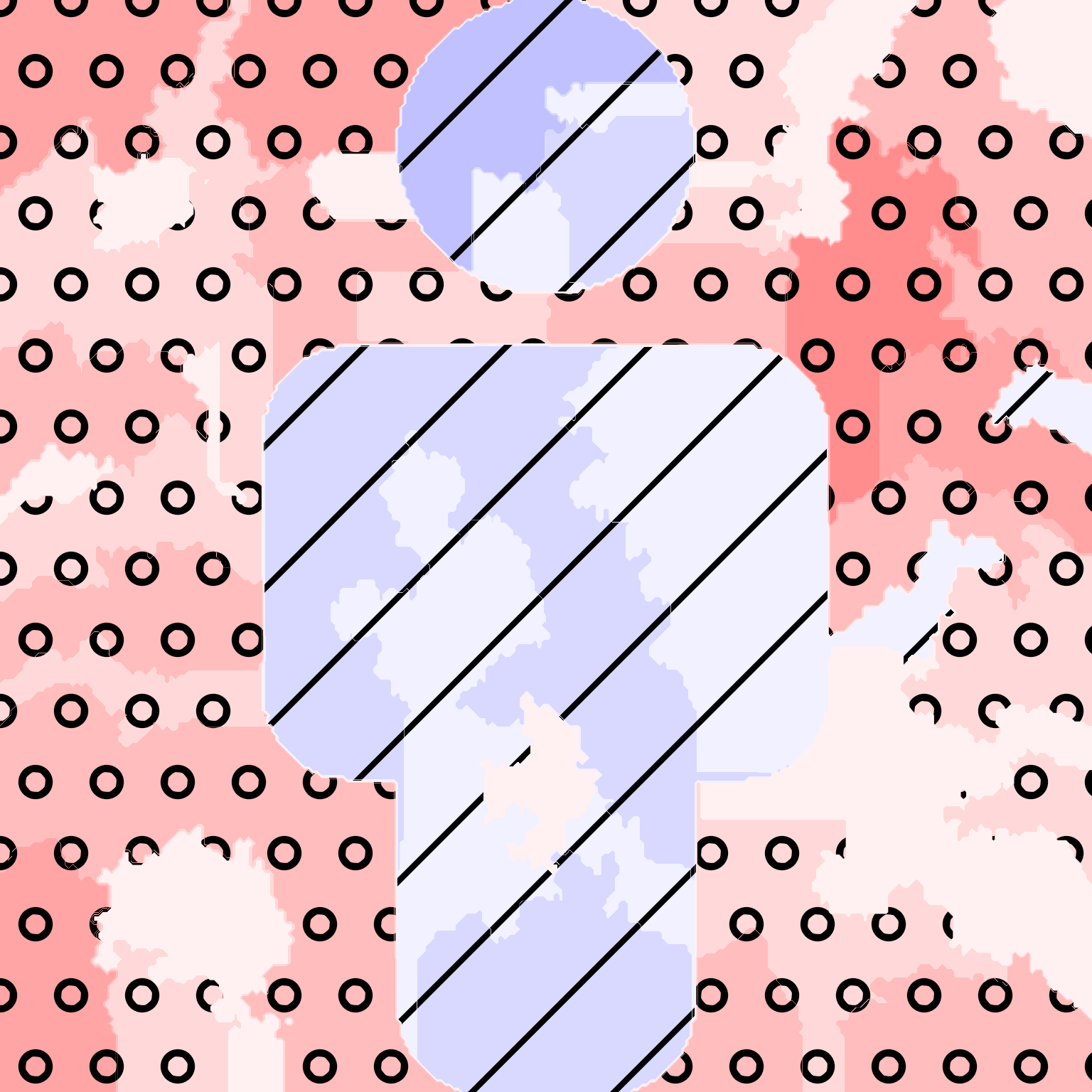}};
        \node[inner sep=0pt, draw=black] (attribution_1_4) at (11.25, 0)
            {\includegraphics[width=0.15\textwidth]{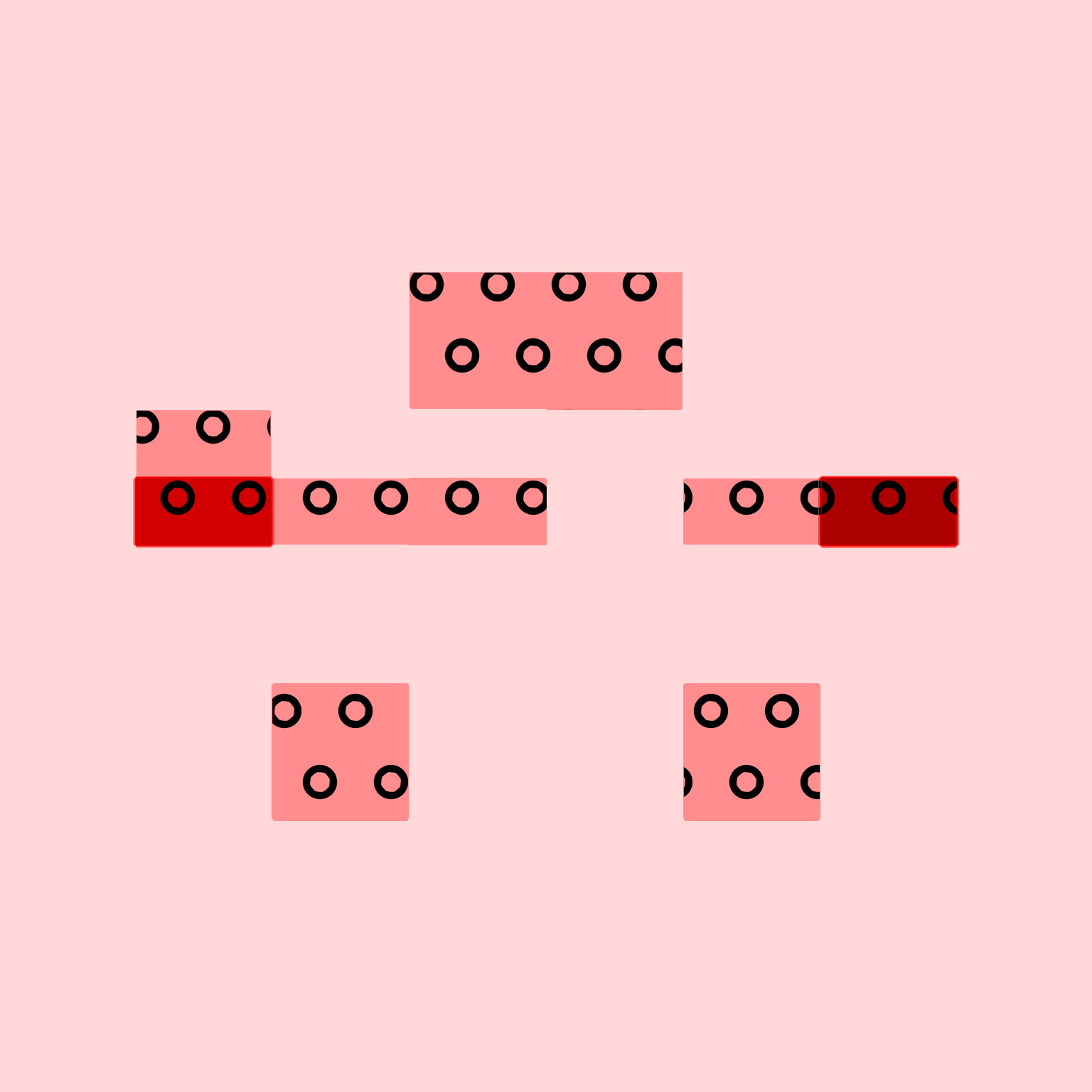}};

        \node[inner sep=0pt, draw=black] (orig_pic_2) at (0, -2.75)
            {\includegraphics[width=0.15\textwidth]{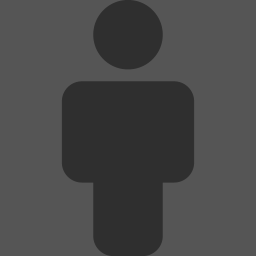}};
        \node[inner sep=0pt, draw=black] (attribution_2_1) at (3, -2.75)
            {\includegraphics[width=0.15\textwidth]{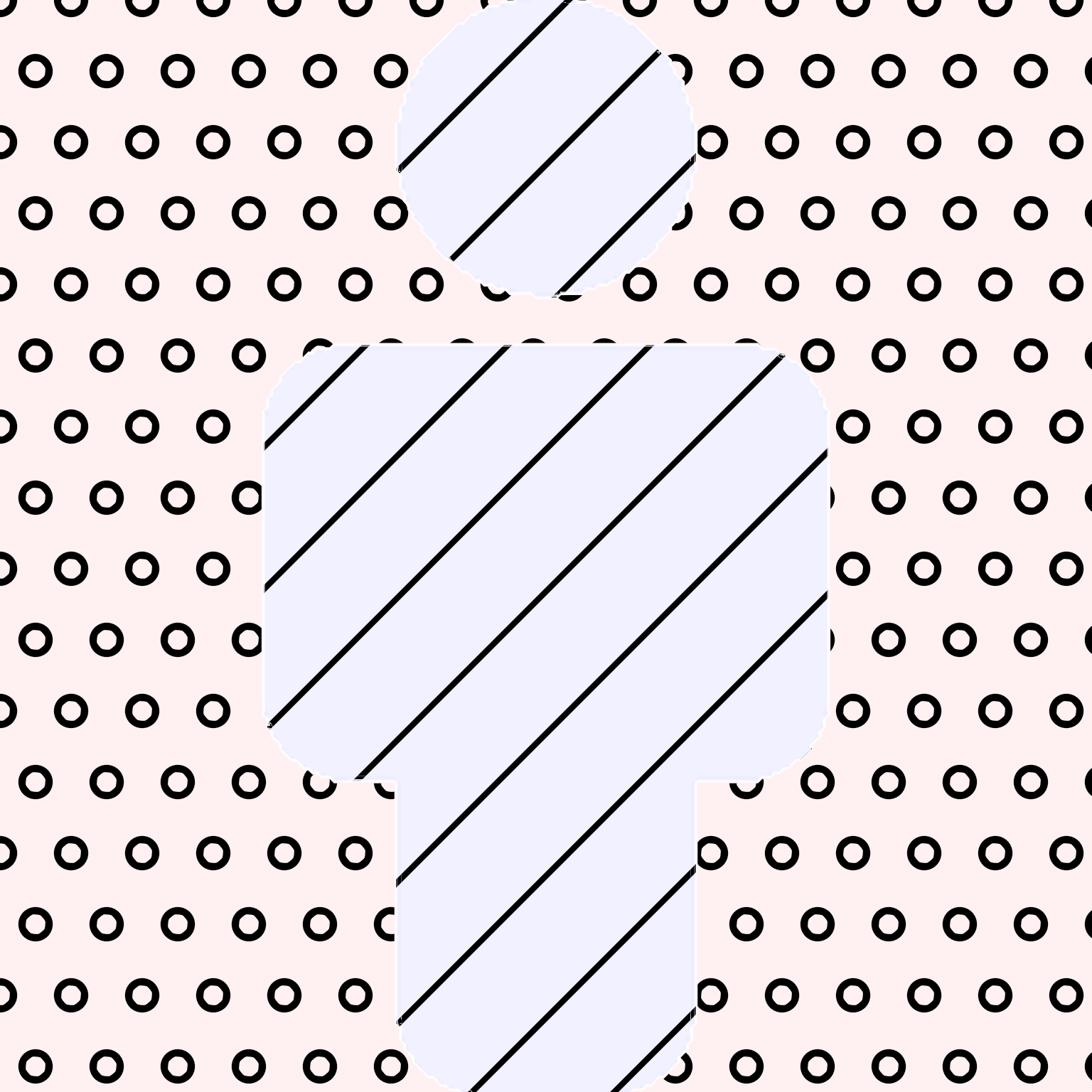}};
        \node[inner sep=0pt, draw=black] (attribution_2_2) at (5.75, -2.75)
            {\includegraphics[width=0.15\textwidth]{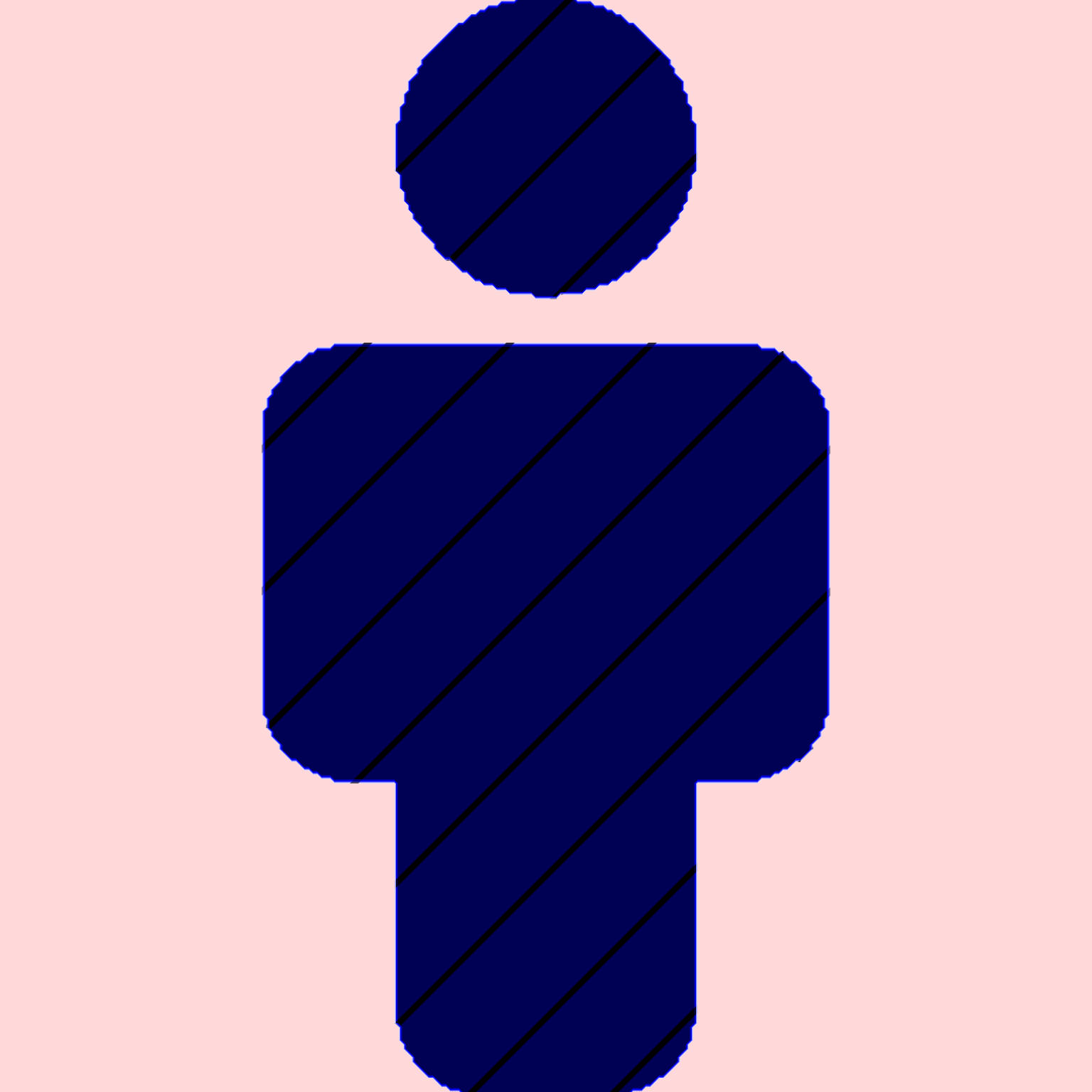}};
        \node[inner sep=0pt, draw=black] (attribution_2_3) at (8.5, -2.75)
            {\includegraphics[width=0.15\textwidth]{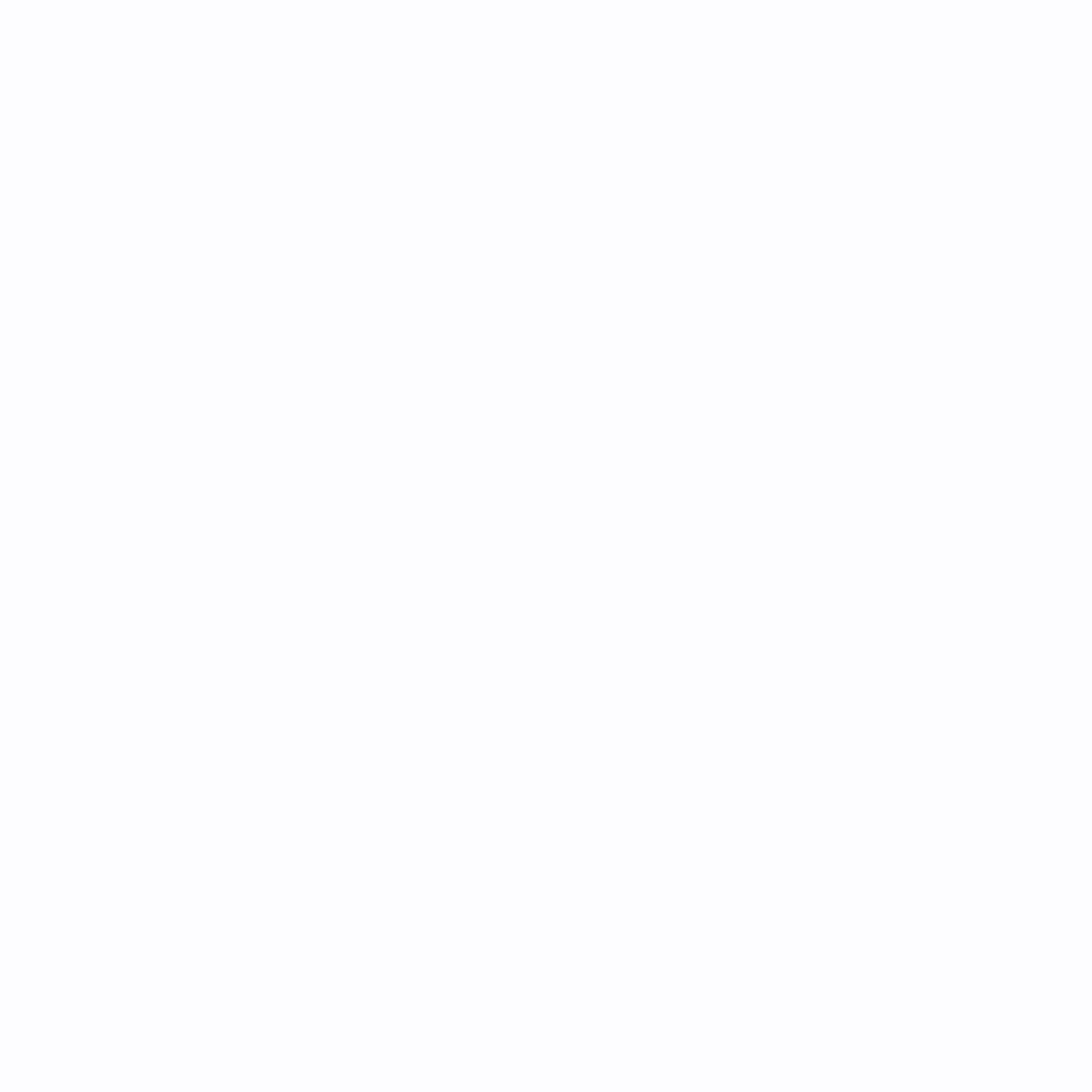}};
        \node[inner sep=0pt, draw=black] (attribution_2_4) at (11.25, -2.75)
            {\includegraphics[width=0.15\textwidth]{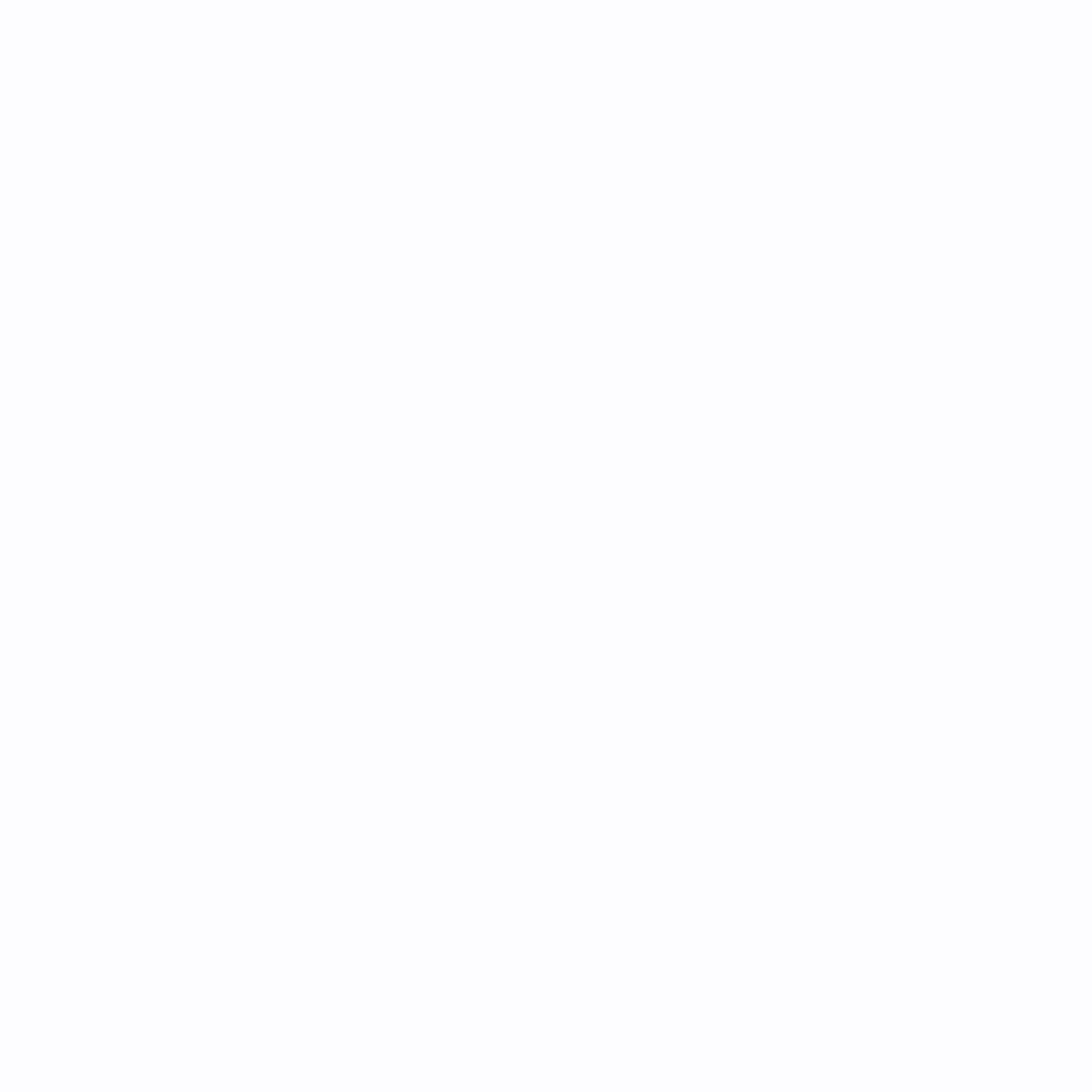}};

        \node[inner sep=0pt, draw=black] (orig_pic_2) at (0, -5.5)
            {\includegraphics[width=0.15\textwidth]{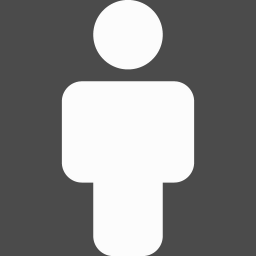}};
        \node[inner sep=0pt, draw=black] (attribution_2_1) at (3, -5.5)
            {\includegraphics[width=0.15\textwidth]{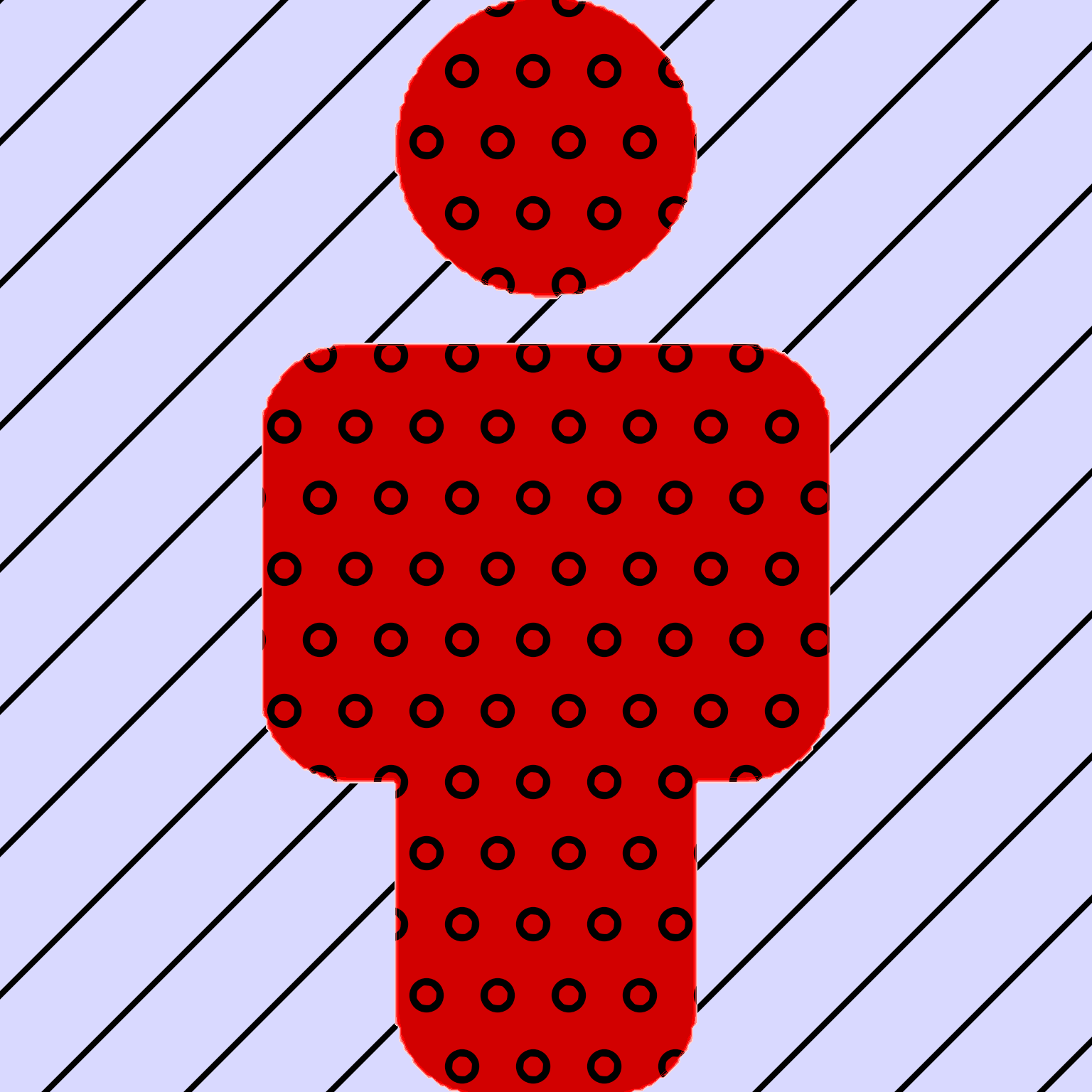}};
        \node[inner sep=0pt, draw=black] (attribution_2_2) at (5.75, -5.5)
            {\includegraphics[width=0.15\textwidth]{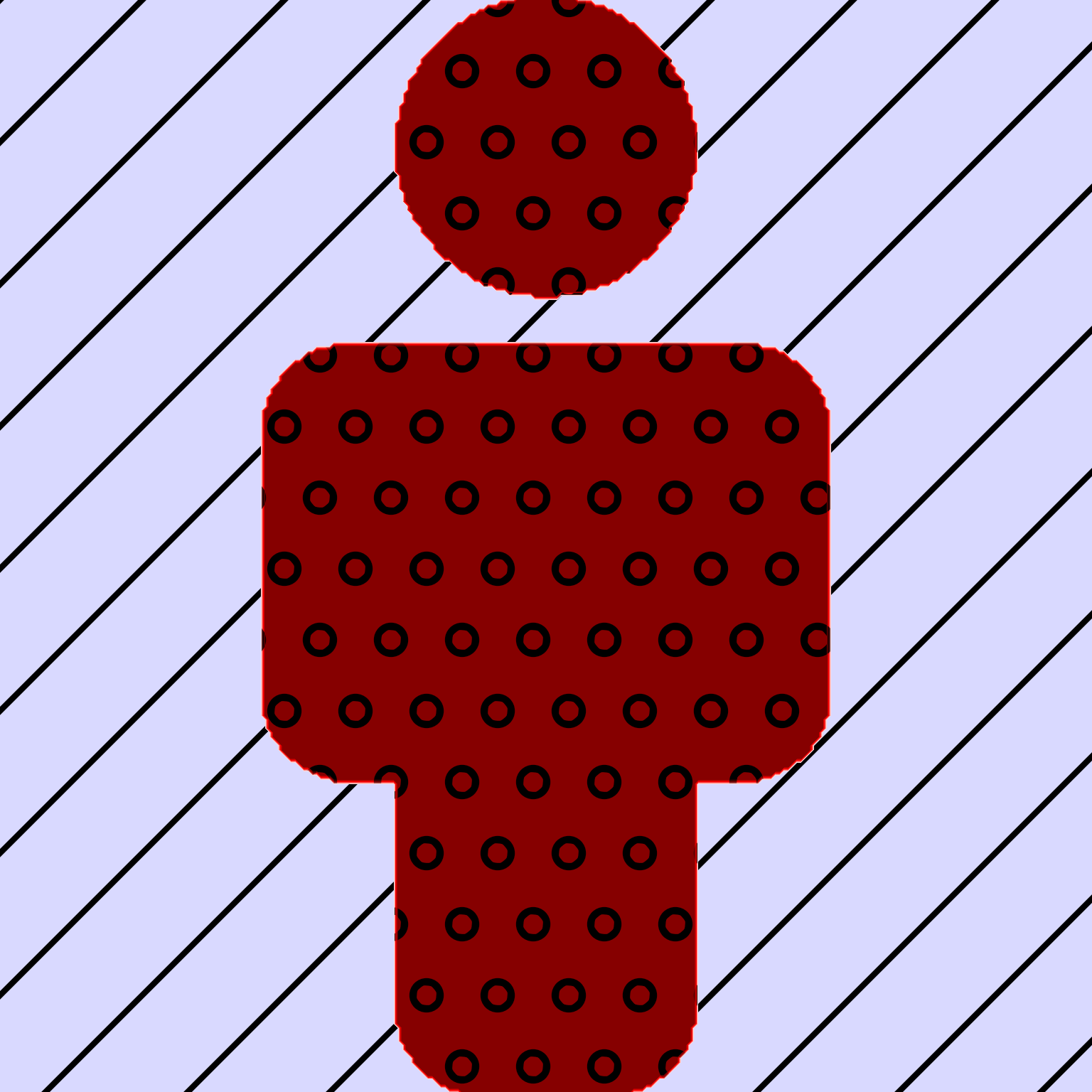}};
        \node[inner sep=0pt, draw=black] (attribution_2_3) at (8.5, -5.5)
            {\includegraphics[width=0.15\textwidth]{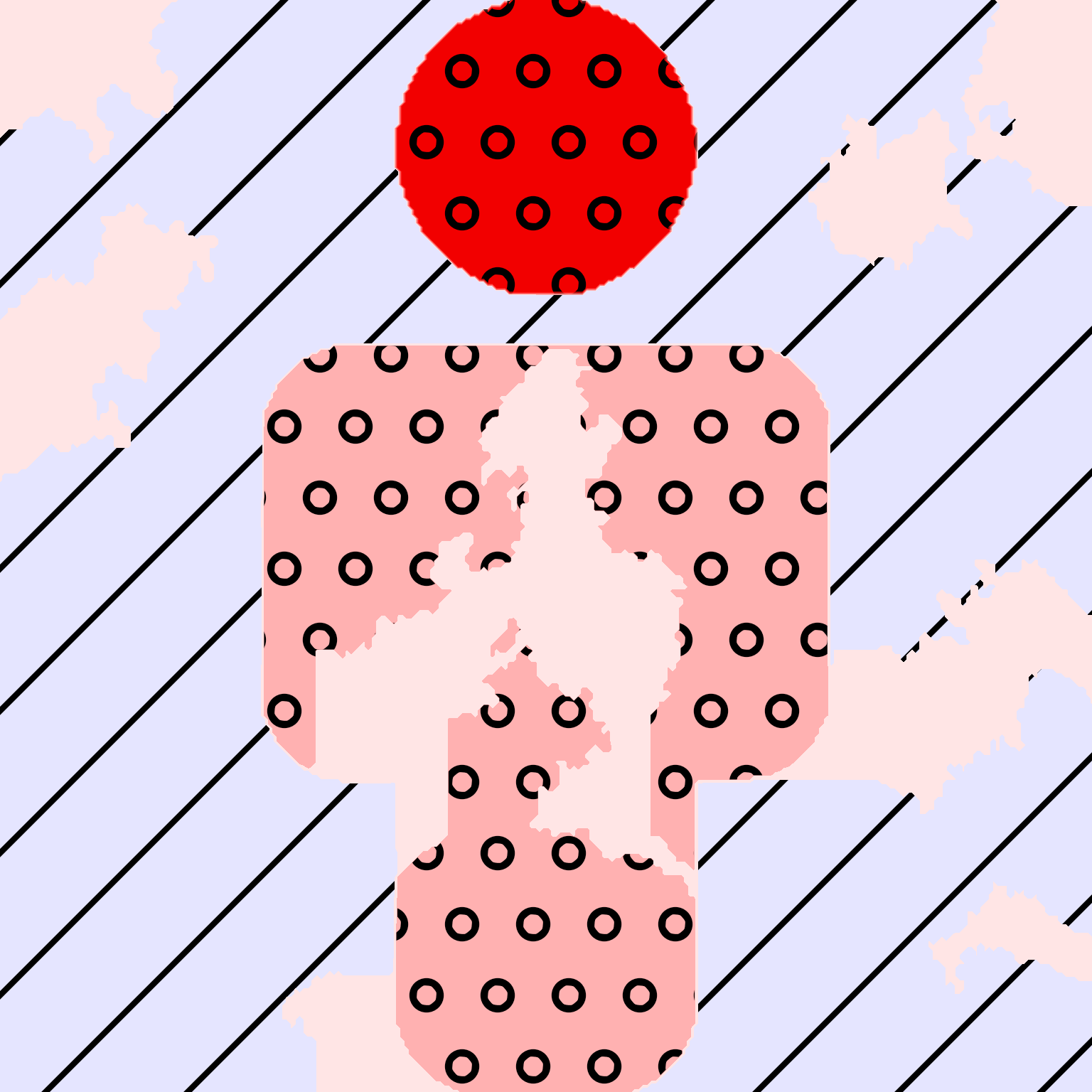}};
        \node[inner sep=0pt, draw=black] (attribution_2_4) at (11.25, -5.5)
            {\includegraphics[width=0.15\textwidth]{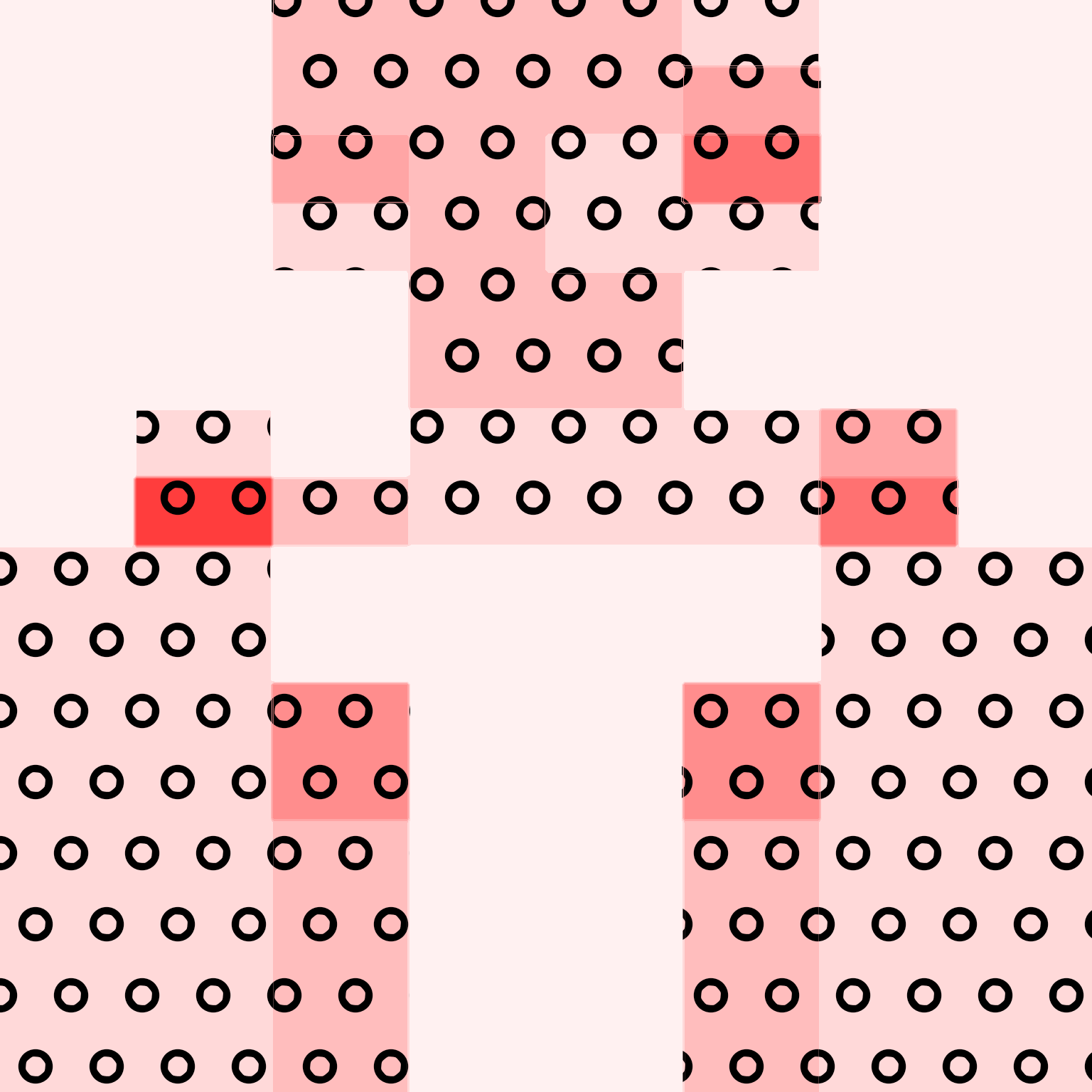}};

        \node[inner sep=0pt, draw=black] (orig_pic_2) at (0, -8.25)
            {\includegraphics[width=0.15\textwidth]{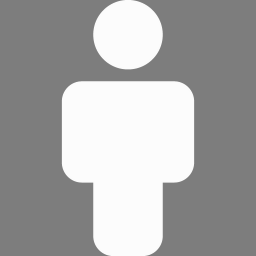}};
        \node[inner sep=0pt, draw=black] (attribution_2_1) at (3, -8.25)
            {\includegraphics[width=0.15\textwidth]{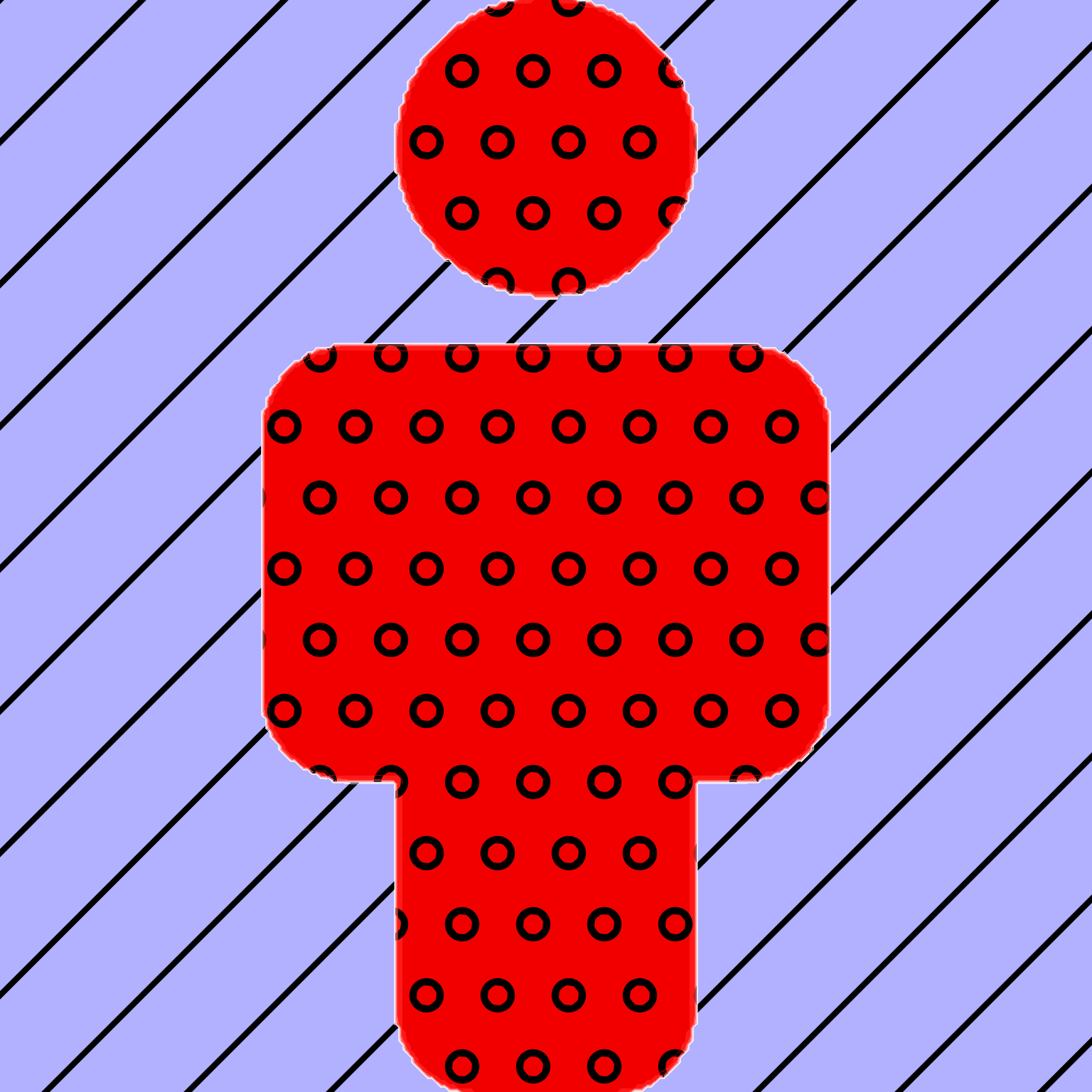}};
        \node[inner sep=0pt, draw=black] (attribution_2_2) at (5.75, -8.25)
            {\includegraphics[width=0.15\textwidth]{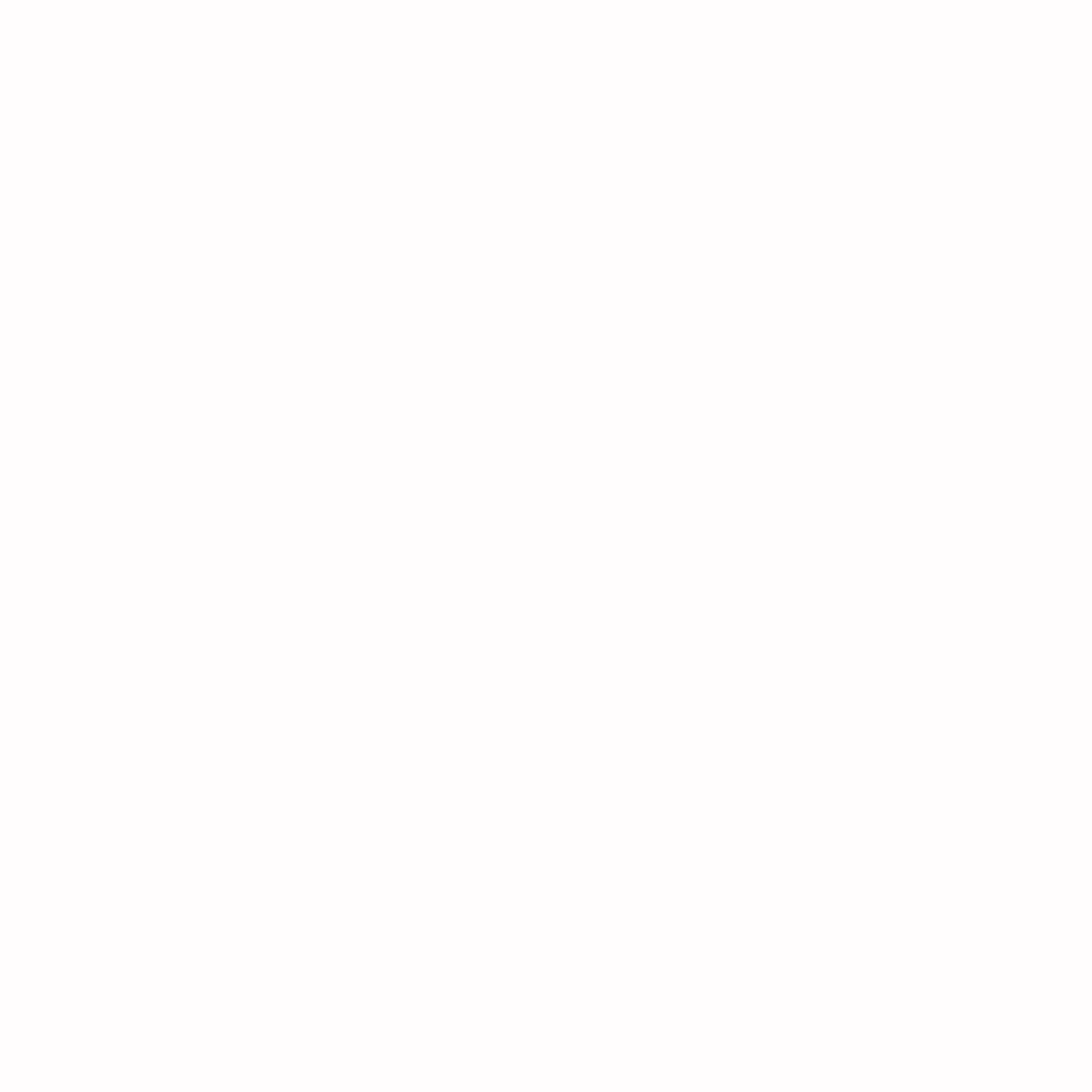}};
        \node[inner sep=0pt, draw=black] (attribution_2_3) at (8.5, -8.25)
            {\includegraphics[width=0.15\textwidth]{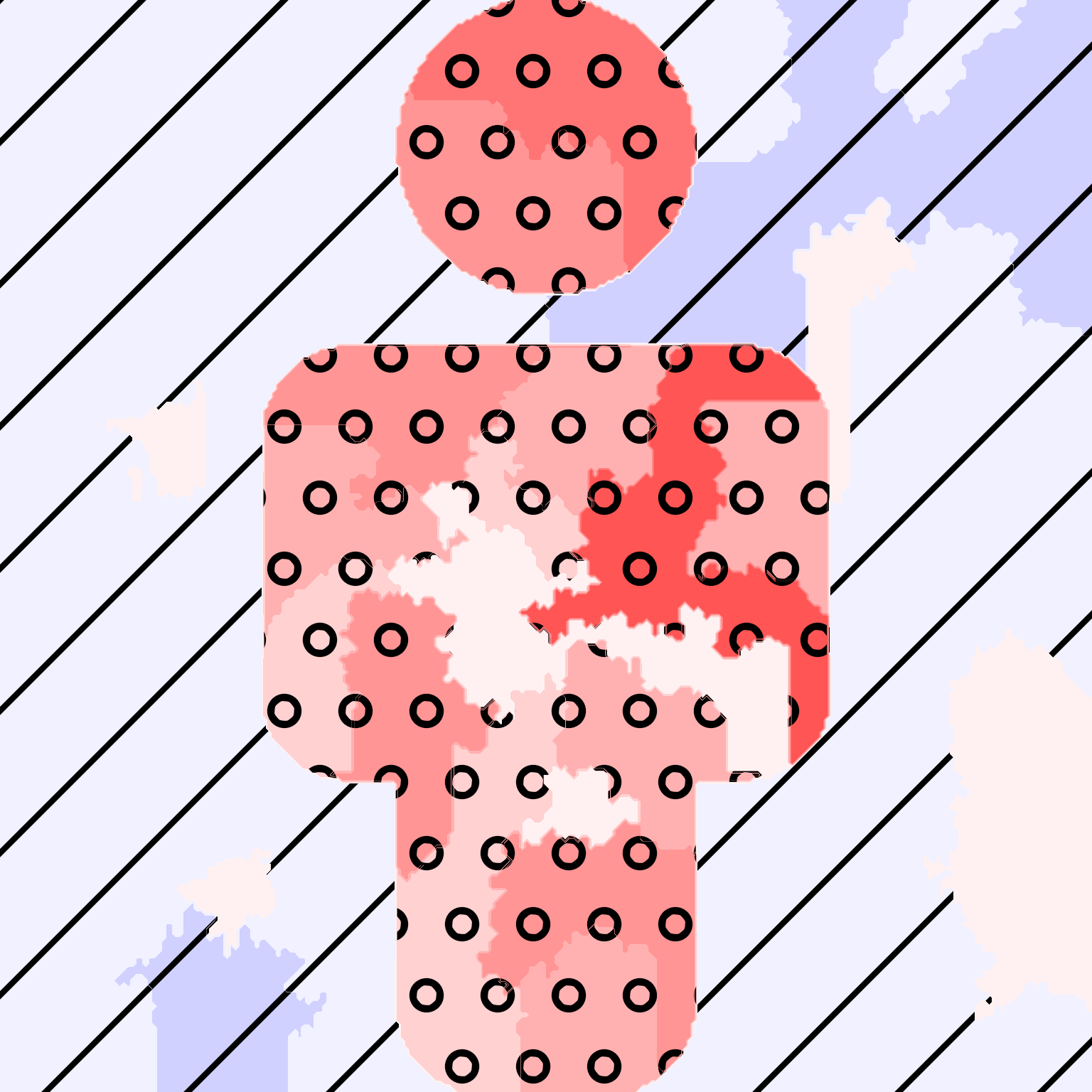}};
        \node[inner sep=0pt, draw=black] (attribution_2_4) at (11.25, -8.25)
            {\includegraphics[width=0.15\textwidth]{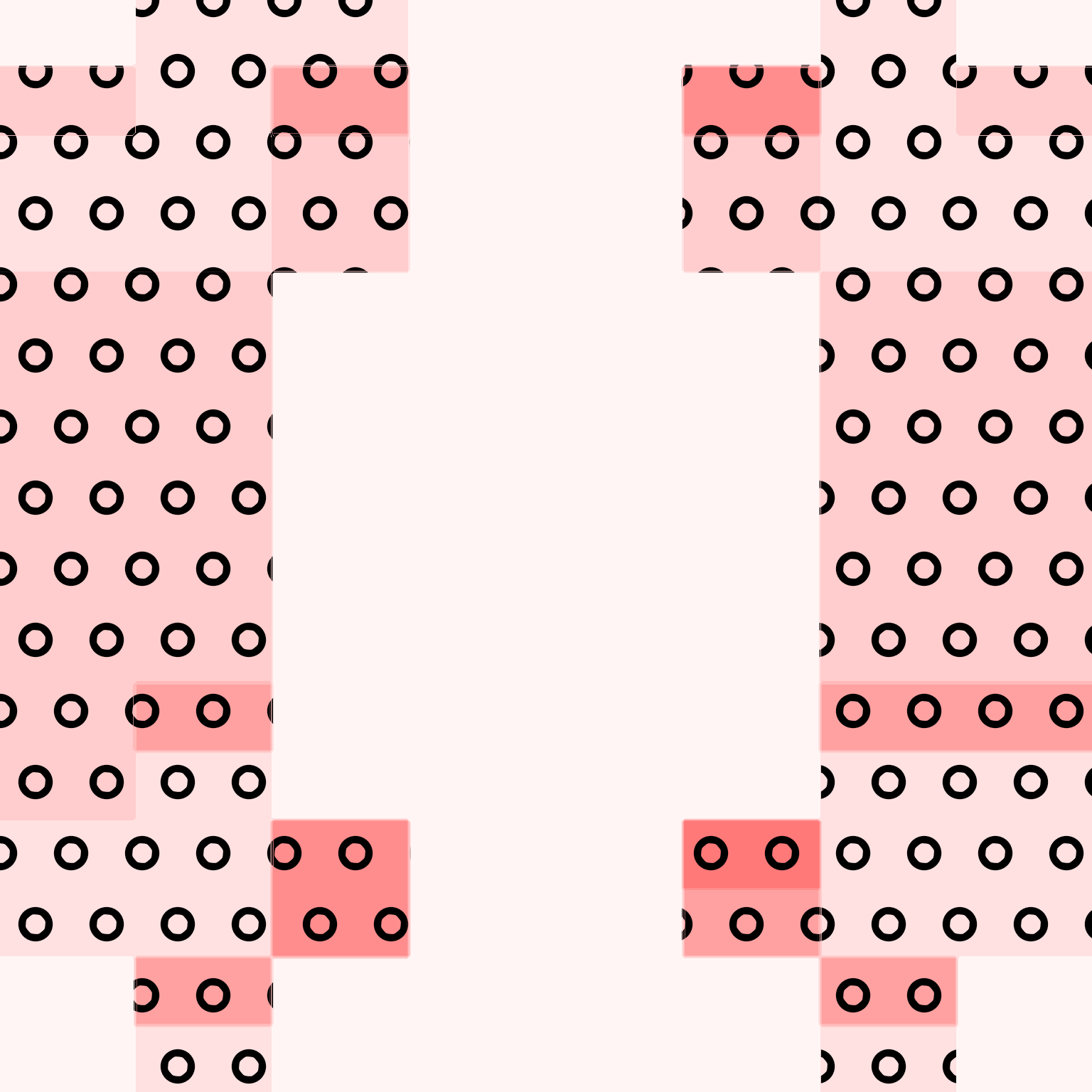}};
        
        \node[inner sep=0pt, draw=black] (orig_pic_2) at (0, -11)
            {\includegraphics[width=0.15\textwidth]{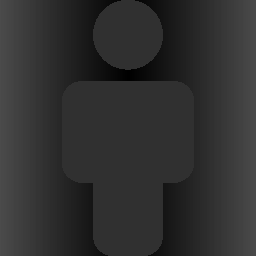}};
        \node[inner sep=0pt, draw=black] (attribution_2_1) at (3, -11)
            {\includegraphics[width=0.15\textwidth]{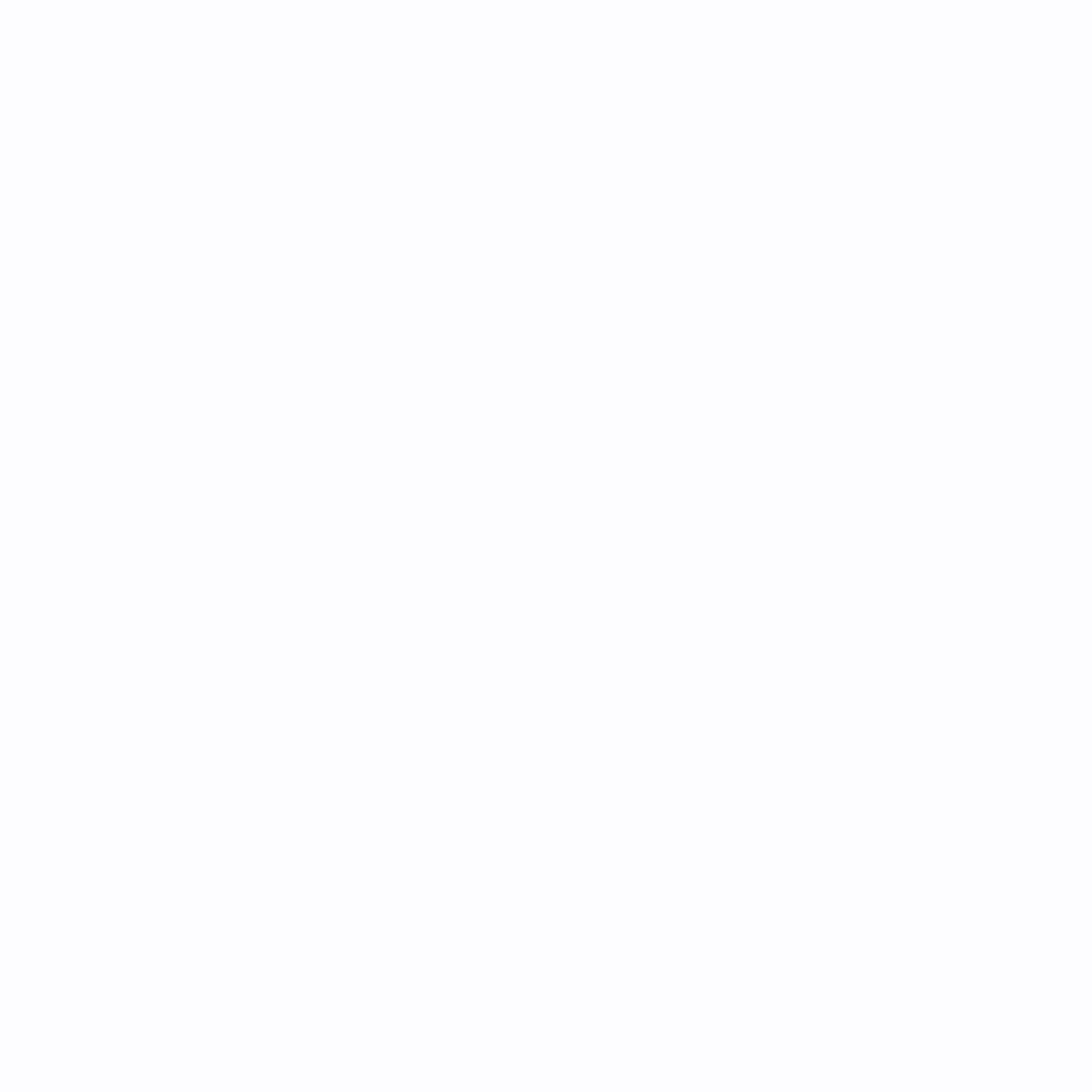}};
        \node[inner sep=0pt, draw=black] (attribution_2_2) at (5.75, -11)
            {\includegraphics[width=0.15\textwidth]{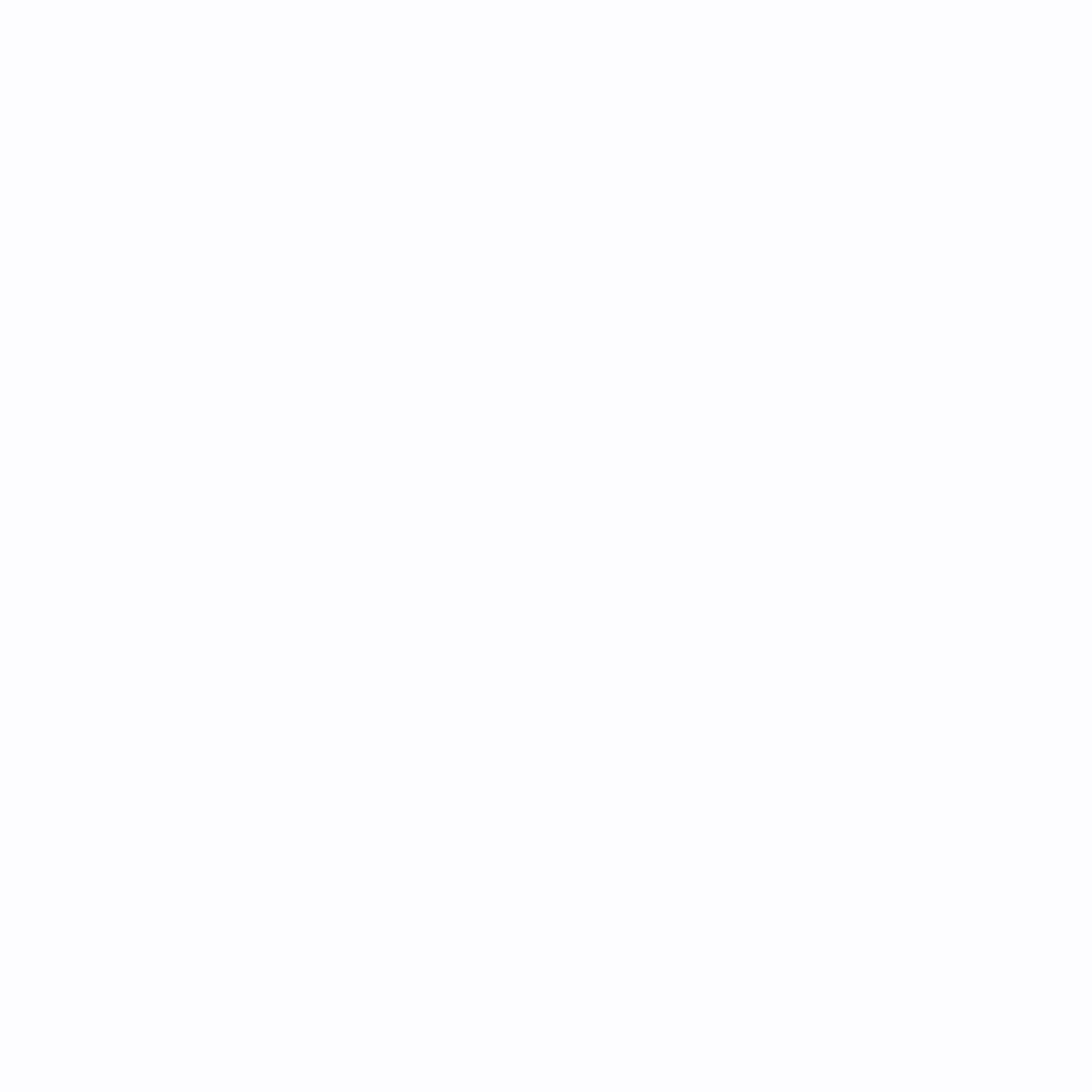}};
        \node[inner sep=0pt, draw=black] (attribution_2_3) at (8.5, -11)
            {\includegraphics[width=0.15\textwidth]{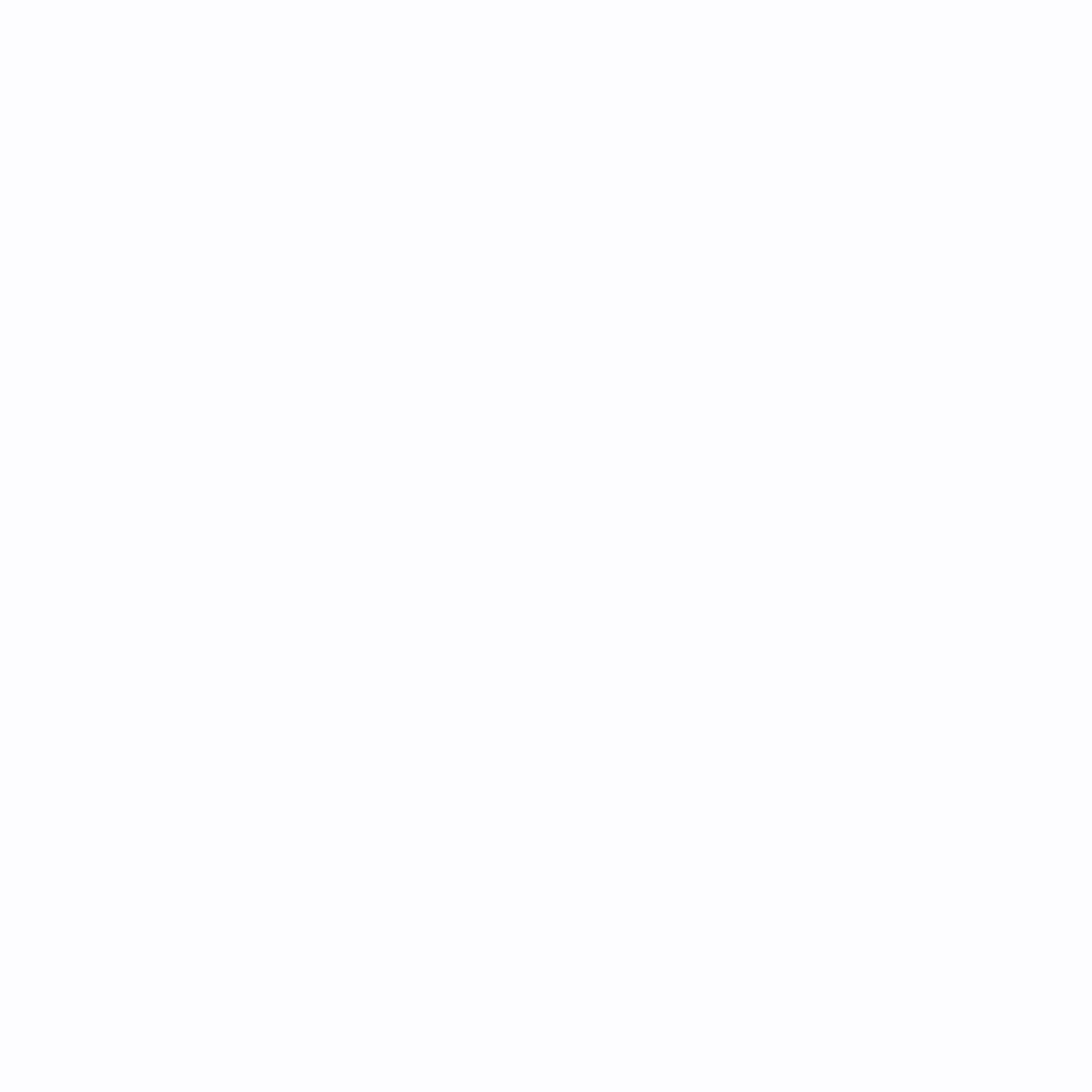}};
        \node[inner sep=0pt, draw=black] (attribution_2_4) at (11.25, -11)
            {\includegraphics[width=0.15\textwidth]{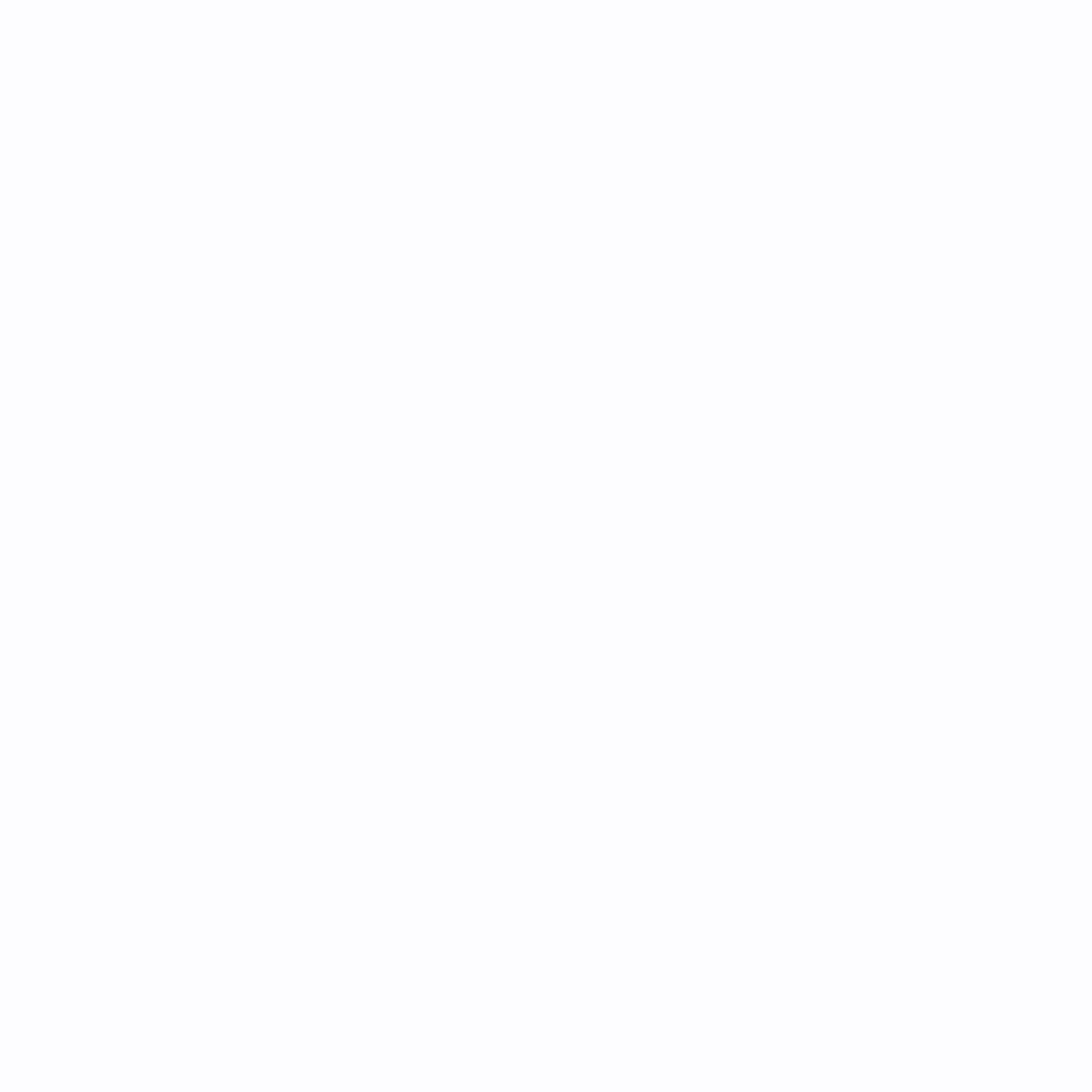}};

        \node[inner sep=0pt, draw=black] (orig_pic_2) at (0, -13.75)
            {\includegraphics[width=0.15\textwidth]{figures/user-icon-zero-attr-1.png}};
        \node[inner sep=0pt, draw=black] (attribution_2_1) at (3, -13.75)
            {\includegraphics[width=0.15\textwidth]{figures/vanilla_gradients_qd_user_icon_53.png}};
        \node[inner sep=0pt, draw=black] (attribution_2_2) at (5.75, -13.75)
            {\includegraphics[width=0.15\textwidth]{figures/self_lime_qd_user_icon_53.png}};
        \node[inner sep=0pt, draw=black] (attribution_2_3) at (8.5, -13.75)
            {\includegraphics[width=0.15\textwidth]{figures/auto_lime_qd_user_icon_53.png}};
        \node[inner sep=0pt, draw=black] (attribution_2_4) at (11.25, -13.75)
            {\includegraphics[width=0.15\textwidth]{figures/shap_qd_user_icon_accepted_53.png}};

        \node (text_1) at (0, 1.5) {Original};
        \node (text_2) at (3, 1.5) {Gradient Methods};
        \node (text_3) at (5.75, 1.5) {LIME manual};
        \node (text_4) at (8.5, 1.5) {LIME auto};
        \node (text_5) at (11.25, 1.5) {SHAP};

        \draw (1.5, 1.1) -- (1.5, -14.9);
    \end{tikzpicture}

    \caption{Additional examples of pictures and their attributions. From top to bottom the labels 
    were: Accepted, Rejected, Accepted, Accepted, Rejected, Rejected.}
    \label{fig:experiment_results}
\end{figure}

\newpage
All proofs of the results in the main text can be found in the
following sections. For clarity we will repeat the statements.

\section{Proof of Theorem~\ref{thm:impossibility_result}}\label{sec:proofs}

\ImposResult*

\begin{proof}
    We will split this proof into three parts. First, we consider the
    one-dimensional case, $\mathcal{X} = \mathbb{R}$. Then, we will show how to
    deal with $\mathcal{X}= \mathbb{R}^{d}$.  Finally, we will discuss the
    result in its most general form, meaning $\mathcal{X} \subseteq
    \mathbb{R}^{d}$.
	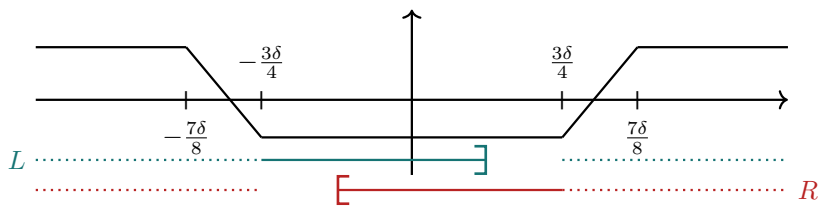
\begin{figure}[b]
		\centering
		\begin{tikzpicture}[line width=0.3mm]
            \draw[->] (-5, 0) -- (5, 0) node[right] {};
            \draw[->] (0, -1) -- (0, 1.2) node[above] {};
            
            \draw[-{Bracket[width=4mm,line width=1pt,length=1.5mm]},
                color={rgb:red,34;green, 178;blue, 178}] 
                (-2, -0.8) node[left] {} -- (1, -0.8);
            \draw[dotted,
                color={rgb:red,34;green, 178;blue, 178}] 
                (-5, -0.8) node[left] {$L$} -- (-2, -0.8);
            \draw[dotted,
                color={rgb:red,34;green, 178;blue, 178}] 
                (2, -0.8) node[left] {} -- (5, -0.8);

            \draw[{Bracket[width=4mm,line width=1pt,length=1.5mm]}-,
                color={rgb:red,178;green, 34;blue, 34}] 
                (-1, -1.2) -- (2, -1.2) node[right] {};
            \draw[dotted,
                color={rgb:red,178;green, 34;blue, 34}] 
                (2, -1.2) -- (5, -1.2) node[right] {$R$};
            \draw[dotted,
                color={rgb:red,178;green, 34;blue, 34}] 
                (-5, -1.2) -- (-2, -1.2) node[right] {};
            
            \node at (-3, 0) {|};
            \node at (-3, -0.5) {$-\frac{7 \delta}{8}$};
            
            \node at (-2, 0) {|};
            \node at (-2, 0.5) {$-\frac{3 \delta}{4}$};
            
            \node at (2, 0) {|};
            \node at (2, 0.5) {$\frac{3 \delta}{4}$};
            
            \node at (3, 0) {|};
            \node at (3, -0.5) {$\frac{7 \delta}{8}$};
            
            \draw[domain=-5:-3, smooth, variable=\x, black] plot ({\x}, 0.7);
            \draw[domain=-3:-2, smooth, variable=\x, black] plot ({\x}, {-1.2 * \x - 2.9});
            \draw[domain=-2:2, smooth, variable=\x, black] plot ({\x}, -0.5);
            \draw[domain=2:3, smooth, variable=\x, black] plot ({\x}, {1.2 * \x -2.9});
            \draw[domain=3:5, smooth, variable=\x, black] plot ({\x}, 0.7);
		\end{tikzpicture}
		\caption{The constructed function $f(x)$. The dotted lines indicate that those values could
        be part of $L$ or $R$, but do not have to be part of either set necessarily. }
		\label{fig:constructed_f_plot}
	\end{figure}
    Consider the case when $\mathcal{X} = \mathbb{R}$ and $C(x) = \mathcal{X}$.
    By assumption, there are two points $z_1, z_2 \in\mathbb{R} $ such that
    $\widetilde{u}(z_1, z_2)\ge \tau$.  We can construct a continuous function
    $f$ explicitly (see Figure~\ref{fig:constructed_f_plot}):
	\begin{align*}
        f(x) =
        \begin{cases}
            z_1 & |x| < \frac{3}{4}\delta, \\
            \frac{8(z_2 - z_1)}{\delta}|x| + (7z_1 - 6z_2)
                &  \frac{3\delta}{4} \le |x| \le \frac{7\delta}{8}, \\
            z_2 & |x| > \frac{7\delta}{8}.
        \end{cases}
	\end{align*}
    We will apply Theorem~\ref{thm:2a} to show that no attribution
    method $\phi_f$ can be both recourse sensitive and continuous on
    this function $f$. To this end, we first note that, for $x \in
    [-\frac{3\delta}{4}, \frac{\delta}{8}]$ and $y = x- \delta$ we
    have that $u_f(x, y) =\widetilde{u}(f(x), f(y)) =
    \widetilde{u}(z_1, z_2) \ge \tau$. Which means that the
    attribution $\varphi$ is allowed to point to the left on
    $[-\frac{3\delta}{4}, \frac{\delta}{8}]$. By a similar argument,
    we find that $\varphi_f$ is allowed to point to the right on
    $[-\frac{\delta}{8}, \frac{3\delta}{4}]$.
	Furthermore, we also see that $\varphi_f$ is not allowed to point towards
	the left on $[\frac{\delta}{4}, \frac{3\delta}{4}]$, because $f$ does not change on 
	$[x- \delta, x]$  if $x \in [\frac{\delta}{4}, \frac{3\delta}{4}]$. 
	Analogously, it can be shown that $\varphi_f$ cannot point towards the right 
	on $[-\frac{3\delta}{4}, -\frac{\delta}{4}]$. With regard to 
	Theorem~\ref{thm:2a}, this means that $[\frac{\delta}{4}, \frac{3\delta}{4}]
	\subseteq \widetilde{R}$ and $[-\frac{3\delta}{4}, -\frac{\delta}{4}]
	\subseteq \widetilde{L}$ for any decomposition $\widetilde{L}, \widetilde{R}$ with 
	$\widetilde{L} \cup \widetilde{R} = L \cup R$. Now, it is not possible to separate
	$[-\frac{3\delta}{4}, -\frac{\delta}{4}]$ from the interval 
	$[-\frac{3\delta}{4}, \frac{\delta}{8}]$, which means that we would need 
	$[-\frac{3\delta}{4}, \frac{\delta}{8}] \subseteq \widetilde{L}$. Similarly, we would need
	$[-\frac{\delta}{8}, \frac{3\delta}{4}] \subseteq \widetilde{R}$. It follows that 
	$\widetilde{L}$ and $\widetilde{R}$ are not disjoint and in particular can never 
	be separated. We conclude that for this continuous $f$ no continuous recourse 
	sensitive $\varphi_f$ can exist. Note that this argument implies that 
    no  $\varphi_f$ could exist on the interval $[-\delta, \delta]$. So, failing 
    to provide recourse or be robust is a local issue. 
	
	We will now generalize the above argument to the setting where $\mathcal{X}=\mathbb{R}^{d}$ and 
    the constraint is given by 
    $C(x)=\{y \in \mathbb{R}^{d} \mid \| x- y\|_{0} \le k  \} $ or 
    $C(x) = \{y \in \mathbb{R}^{d}  \mid y = x + \alpha z, \alpha \ge 0, z \in D\}$.
    Actually, these two versions of constraints can be dealt with simultaneously, by rotating 
    the input space in the latter case in such a way that one of the vectors $z \in D$ lies
    alongside one axis. 
    The argument for the one dimensional result can now be 
    embedded in these cases. Namely, for $x=0$, find the component
	that is allowed to change by the constraints. Call this the $i$'th component.
	We can define a similar function as the function above, 
	\begin{align*}
	f(x) =
	\begin{cases}
        z_1 & |x_i| < \frac{3}{4}\delta,  \\
        \frac{8(z_2 - z_1)}{\delta}|x_i| + (7z_1 - 6z_2)
            &  \frac{3\delta}{4} \le |x_i| \le \frac{7\delta}{8},\\
        z_2 & |x_i| > \frac{7\delta}{8}.
	\end{cases}
	\end{align*}
	This function is again continuous and only changes in the $i$'th coordinate. Repeating the
	argument of the one-dimensional case, we see that an attribution is allowed to be negative
	in the $i$ 'th component on 
	$L^{i} = \mathbb{R}^{i-1} \times 
	[-\frac{3\delta}{4}, \frac{\delta}{8}] \times \mathbb{R}^{d - i-1}$
	and negative on 
	$R^{i} = \mathbb{R}^{i-1} \times
	[-\frac{\delta}{8}, \frac{3\delta}{4}] \times \mathbb{R}^{d - i-1}$. Just as before, we
	also see that $\varphi^i_f(x)$ is necessarily negative on the set
	$\mathbb{R}^{i-1} \times 
	[-\frac{3\delta}{4}, -\frac{\delta}{4}] \times \mathbb{R}^{d - i-1}$, but this set
	cannot be separated from $L^{i}$, which ensures that $\varphi_f(x)_i$ has to be negative 
	on the whole of $L^{i}$. Alternatively, $\varphi_f(x)_i$ has to be positive on 
	$\mathbb{R}^{i-1} \times 
	[\frac{\delta}{4}, \frac{3\delta}{4}] \times \mathbb{R}^{d - i-1}$. Hence, also on the
	whole of $R^{i}$ by the inability of separating $R^{i}$ from this set. However, as $L^{i}$ and
    $R^{i}$ are not disjoint, this is a contradiction. Thus, no
    continuous attribution function
    can exist for $f$ in higher dimensions.

    Finally, we need to handle the multidimensional case that
    $\mathcal{X} \subseteq \mathbb{R}^{d}$. By assumption we have a line
    segment $\ell  \subseteq \mathcal{X}$ with the property that $\ell
    \subseteq C(x)$ for all $x \in \ell$. We can apply a suitable
    transformation to the input space, so that we can fall back on our
    previous argument. This transformation is to first translate the
    line segment such that its middle point becomes the origin. Next, we
    apply a rotation such that the line segment lies along side the
    $i$'th axis. Call this translation and rotation $M$ and $\rho$,
    respectively. The desired function now becomes
    \begin{align*}
         g(x) = f\circ \rho \circ M(x)
    ,\end{align*}
    where $f$ is the function of the precious case. The function $g$ does not allow any
    continuous recourse sensitive attribution function, as $f$ did not allow this on the 
    line segment $[-\delta, \delta]$ in the $i$ 't component and 
    $g(\ell) = f\circ \rho \circ M(\ell) \supseteq f([-\delta,\delta])$. Now, if a continuous and
    recourse sensitive attribution function $\varphi_g$ would exist for $g$, we could construct 
    one for $f$ as well. This is done by setting 
    \begin{align*}
        \varphi_f(x) = \varphi_g \circ M^{-1} \circ \rho^{-1}(x)
    .\end{align*}
    The inverses exist and as translations and rotations do not change distances, this will 
    be a continuous recourse sensitive attribution function for $f$, which was not possible. So, 
    no continuous recourse sensitive attribution method can exist for $g$ on 
    $\mathcal{X}\subseteq \mathbb{R}^{d}$.
\end{proof}

\section{Proofs of Section~\ref{sec:suff_cond_main}}\label{sec:suff_conditions}
\SuffClass*

\begin{proof}
    The function $P_U$ is well defined by the fact that $U$ is closed and convex.
    The set $U$ is closed as it is the pre-image of a closed set
    under a continuous function and this ensures that the projection
    exists. Convexity of $U$ guarantees uniqueness of the projection. 
    Additionally, it is known that the projection is a continuous function 
    if the projection exists and is unique.
    It follows that $\phi_f$ is also continuous.
    This leaves us to check that the map is recourse sensitive. If $x$ is
    such that $f(x) \ge  0$, then $\varphi_f(x) =0$, which is a valid
    attribution, since $u_f(x,x) = f(x) \ge 0$. So, assume $f(x)< 0$ and
    take $\alpha=1$ in the definition of recourse sensitivity. Then
    either $\|P_U(x) - x\| > \delta$, in which case $T(x) = \emptyset$
    and recourse sensitivity holds trivially, or $\|P_U(x) - x\| \leq
    \delta$ so that $P_U(x) \in T(x)$ because $P_U(x) \in U$ by
    definition, so $\varphi_f$ is again recourse sensitive. We conclude
    that $\varphi_f$ is both continuous and recourse sensitive. 
\end{proof}

As stated in the main text, we will need some additional tools from the
field of multi-valued analysis to prove the general result. First, we
will need a definition of continuity for set-valued expressions. 

\begin{defn}[Hemi-continuity]\label{def:hemi}
    For topological spaces $\mathcal{X}$ and $\mathcal{Y}$, a set-valued function 
    $U: \mathcal{X} \to 2^{\mathcal{Y}}$ is called \emph{upper hemi-%
    continuous (UHC)} at $x_0 \in \mathcal{X}$ if, for any open $B \subseteq \mathcal{Y}$ 
    with $U(x_0) \subseteq B$, there exists an open neighbourhood $A$ of $x_0$ such that
    for all $x \in A$, $U(x)$ is a subset of $B$. 

    A set-valued function $U: \mathcal{X} \to 2^{\mathcal{Y}}$ is called \emph{lower hemi-%
    continuous (LHC)} at $x_0 \in \mathcal{X}$, if for any 
    open set $B \subseteq \mathcal{Y}$ intersecting $U(x_0)$ there exists 
    an open neighbourhood $A$ of $x_0$ such that
    $U(x)$ intersects $B$ for all $x \in A$

    If $U$ is UHC and LHC at $x_0$, then $U$ is called \emph{hemi-continuous at $x_0$}. If $U$ is
    hemi-continuous at every $x_0 \in \mathcal{X}$, then $U$ is called \emph{hemi-continuous}.
\end{defn}

We will also need the following two Lemmas. The first relates UHC and LHC to normal 
continuity, when $U$ is single-valued. The second tells us when the graph
of $U$  is a closed set.

\begin{lemma}\label{lem:single_set}
    If $U$ is UHC or LHC at $x_0 \in \mathcal{X}$ and single-valued in some neighbourhood
    $\mathcal{N}$ around $x_0$ , then the function $f: \mathcal{N} \to \mathcal{Y}$ such
    that $U(x) = \{f(x)\}$ is a continuous function at $x_0$.
\end{lemma}
\begin{proof}
    Take some sequence $\{x_n\}_{n=1}^{\infty}$ that converges to $x_0$. Recall that convergence
    is equivalent with the following. For any open neighbourhood $B$ of $x_0$, 
    there exists an $N \in \mathbb{N}$ such  that $n \ge N$ implies that $x_n \in B$. 

    Start by assuming that $U$ is UHC and single-valued. Take any open neighbourhood $B$ of 
    $f(x_0)$. By $U$ being UHC, we can find an open neighbourhood $A$ of $x_0$ such that 
    $U(y)\subseteq B$ for all $y \in A$. Using the above characterisation of convergence, we can 
    find an $N \in \mathbb{N}$ such that $x_n \in A$, whenever $n\ge N$. This also means that
    B$\{f(x_n)\} =U(x_n) \subseteq B$. In particular, $f(x_n) \in B$. As $B$ was arbitrary, 
    it follows that $\displaystyle\lim_{n\to \infty}f(x_n) = f(x_0)$ and that $f$ is continuous at $x_0$.

    Next, we assume that $U$ is LHC and single-valued. Again, take any any open set $B$ such that 
    $U(x_0) \cap B \neq \varnothing$. By the fact that $U(x_0)$ is single-valued, this actually
    means that $\{f(x_0)\} =U(x_0) \subseteq B$. By an analogous argument 
    we again find that $f$ is continuous at $x_0$. 
\end{proof}

\begin{lemma}\label{lem:closed_graph_property}
    If $U$ is UHC and $U(x)$ is a closed set for all $x \in \mathcal{X}$, then the set
    \begin{align*}
        \mathrm{Gr}(U) = \{(x, y) \in \mathcal{X} \times \mathcal{Y}  \mid  y \in U(x)\} 
    \end{align*}
    is closed.
\end{lemma}
\begin{proof}
    See Proposition~$1.4.8$ in \citet{aubin2009set}.
\end{proof}

\begin{theorem}[Berge's Maximum Theorem]\label{thm:Berge}
    Let $\mathcal{X}, \mathcal{Y} \subseteq \mathbb{R}^{d}$, assume that:
    \begin{enumerate}
        \item The function $v: \mathcal{X}\times \mathcal{Y} \to \mathbb{R}$ is a continuous function;
        \item The set-valued function $U: \mathcal{X}\to 2^{\mathcal{Y}}$ is hemi-continuous, 
            never empty, and assumes compact sets.
    \end{enumerate}
    Then, the parametrized optimization problem $v^{*}(x) \coloneqq \inf_{y \in U(x)} v(x,y)$ 
    is continuous  and the set-valued solution function 
    $U^{*}(x) = \{y \in U(x)  \mid v^{*}(x) = v(x,y)\}$ is UHC and compact-valued. 
\end{theorem}
\begin{proof}
    See Chapter $6$ in \citet{berge1997topological}.
\end{proof}

Now, Theorem~\ref{thm:general_sufficient} follows almost immediately from Theorem~\ref{thm:Berge}.
The only issue is that the set $U(x)$ needs to be compact to apply Theorem~\ref{thm:Berge}. 
However, we do not want to impose this. Luckily, there exists a relaxation of Berge's
Maximum Theorem, where we do not need compact-valued sets. This will 
require an additional property of the optimization problem, but this will be 
satisfied by the Euclidean norm.  
\begin{theorem}[Berge's Maximum Theorem for Non-Compact Image Sets]\label{thm:Berge_non_compact}
    Let $\mathcal{X}, \mathcal{Y} \subseteq \mathbb{R}^{d}$, assume that:
    \begin{enumerate}
        \item The function $v: \mathcal{X}\times \mathcal{Y}\to \mathbb{R}$ 
            is continuous and that for every compact 
            $K \subseteq \mathcal{X}$ the set 
            \begin{align*}
                D_v(\lambda; K) = \{(x, y) \in K \times \mathcal{Y}  
                \mid  y \in U(x), v(x,y) \le \lambda\} 
            \end{align*}
            is compact for all $\lambda \in \mathbb{R}$;
        \item The set-valued $U: \mathcal{X}\to 2^{\mathcal{Y}}$ is LHC and never empty.
    \end{enumerate}
    Then, the parametrized optimization problem 
    $v^{*}(x) = \inf_{y \in U(x)} v(x,y)$ is continuous and the solution set-valued 
    function $U^{*}(x) = \{y \in U(x)  \mid v^{*}(x) = v(x, y)\}$
    is UHC and compact-valued.
\end{theorem}

At this point, we have all the tools required to prove Theorem~\ref{thm:general_sufficient}. 
Let us repeat the statement. 

\SuffGen*

\begin{proof}
    We will split the proof into two parts. First, we will consider the case where
    the projection onto the sets $U(x)$ is actually unique for all $x \in \mathcal{X}$.
    Afterwards, we will 
    discuss the case where the projection is not unique for every point. 

    We want to apply Theorem~\ref{thm:Berge_non_compact}, where $v(x,y)$ is given 
    by $v(x,y) = \|y - x\|$ and $\mathcal{X}=\mathcal{Y}\subseteq \mathbb{R}^{d}$.
    The set-valued $U$ is given by all feasible points that achieve sufficient 
    utility,
    \begin{align*}
        U(x) = \{y \in \mathcal{X} \mid u_f(x,y) \ge \tau\} \cap C(x) 
    .\end{align*}
    The parametrized optimization problem will be given by 
    $v^{*}(x) = \inf_{y \in U(x)}\|y - x\|$. By assumption, this infimum is attained, because
    $U(x)$ is closed, and it is unique. 
    It rests to check that the sets 
    $D_v(\lambda;K)$ are compact for all compact $K$ and $\lambda \in \mathbb{R}$.

    Let us decompose $D_v(\lambda;K)$ by setting
    \begin{align*}
        D_v(\lambda;K) 
        &= \{(x, y) \in K \times \mathcal{Y}  \mid 
            y \in U(x), \|x -y \| \le \lambda
        \} \\
        &= \{(x, y) \in K \times \mathcal{Y}  \mid 
            y \in U(x)
        \} \cap \{
            (x, y) \in K \times \mathcal{X} \mid  \|x -y \| \le \lambda
        \}
    .\end{align*}
    The first of these sets can be further decomposed by intersecting
    $K \times \mathcal{Y}$ and $\text{Gr}(U)$, for $\text{Gr}(U)$ as
    defined in Lemma~\ref{lem:closed_graph_property}. The set $K\times \mathcal{Y}$ is closed, because it is the product
    of two closed sets, and the set $\text{Gr}(U)$ is closed by 
    Lemma~\ref{lem:closed_graph_property}, 
    so their intersection must be closed as well.
    Similarly, it can be seen that the set $\{(x, y) \in K \times
    \mathcal{Y}  \mid \|x -y\|\le \lambda\} $ is closed, by writing it as the 
    intersection between $K \times \mathcal{Y}$ and $\{(x,y) \in \mathcal{X}
    \times \mathcal{Y} \mid \|x-y\|\le \lambda\}$. The latter set is seen to be closed
    as it is the inverse image of closed set under a continuous function.
    Furthermore, the set $\{(x, y) \in K \times
    \mathcal{Y}  \mid \|x -y\|\le \lambda\} $ 
    is a bounded set as it can be seen as the set $K$
    with a strip around it of size $\lambda$. It follows that
    $D_v(\lambda;K)$ closed and bounded, hence compact. As $\lambda$ and
    $K$ were arbitrary we see that $v(x,y) = \|x -y\|$ has the desired
    property. 
    
    As noted in the Theorem statement, the attribution function will be given by
    \begin{align*}
        \varphi_f(x) = \argmin_{y \in U(x)}\|x - y\| - x = P_{U(x)}(x) - x
    .\end{align*}
    We can now apply Theorem~\ref{thm:Berge_non_compact} and see that the solution
    $P_{U(x)}(x)$ is UHC and compact-valued. Furthermore, the projection does exist and is 
    unique.
    Invoking 
    Lemma~\ref{lem:single_set} then tells us that $P_{U(x)}(x)$ is a continuous function. 
    We see that $\varphi_f(x)$ is continuous and recourse sensitive by design. 

    Now, we will drop the assumption that every point $x \in \mathcal{X}$ has a 
    unique projection onto $U(x)$. Consider the subset $X \subseteq \mathcal{X}$ 
    for which each point does have projection onto $U(x)$. Now, we
    can repeat the proof shown above using Berge's Maximum theorem with the sets
    $\mathcal{X}= X$ and $\mathcal{Y}=\mathcal{X}$, as the sets $U(x)$ will not be 
    subsets of $X$ in general. This will result in a continuous projection from $X$
    onto the sets $U(x)$ for all points in $X$ and the function
    \begin{align*}
        \varphi_f(x)\coloneqq\argmin_{y \in U(x)} \|x -y \|= P_{U(x)} -x
    \end{align*}
    will be well defined, continuous and a recourse sensitive attribution function 
    on the restricted set $X$.
\end{proof}

\section{Proofs of Section~\ref{sec:suff_and_nec_conditions_in_one_dimension}}
\oneDimRes*

\begin{proof}
    \emph{If}: Assume that $\widetilde{L} \subseteq L, \widetilde{R} \subseteq R$ and 
    $\widetilde{O}\subseteq O$ 
	exist with properties $(1) - (3)$. We will construct $\varphi_f$ 
	explicitly. To this end, we define the distance to a set $A$ as 
	\begin{align}\label{eq:set_distance}
        d(x, A) = \inf_{y \in A} |x - y|
	.\end{align}
	It is known that  $d(x, A) $ is continuous for any set, see for example Chapter~$2.5$ 
    in \citet{mendelson1990introduction}. 
	By separatedness of $\widetilde{L}$ and $\widetilde{R}$ we can find open neighborhoods
	$U_1, V_1 \subseteq \mathbb{R}$ of $\widetilde{L} $ and $\widetilde{R}$ respectively, 
	such that $U_1 \cap V_1 = \varnothing$. Furthermore, by property $(3)$ we can also find other
    open neighbourhoods of $U_2, V_2 \subseteq \mathbb{R}$ of $\tL$ and  $\tR$
    such that $\widetilde{O}\cap U_2 =\varnothing$ and 
    $\widetilde{O} \cap V_2 = \varnothing$. The sets $U= U_1 \cap U_2$ and  $V = V_1 \cap V_2$
    are still open neighbourhoods of $\tL$ and  $\tR$ and they 
    are disjoint from each other and  $\widetilde{O}$. Now define
	\begin{align*}
        \varphi_f^{-}(x) &= \frac{d(x, \mathbb{R}  \setminus U)}
            {1 + d(x, \mathbb{R} \setminus U)}, \\
        \varphi_f^{+}(x) &= \frac{d(x, \mathbb{R}  \setminus V)}
            {1 + d(x, \mathbb{R} \setminus V)}
	.\end{align*}
	Using these functions, we can construct $\varphi$ by setting
	\begin{align*}
        \varphi_f(x) = \varphi_f^{+}(x) - \varphi_f^{-}(x)
	.\end{align*}
	We have to check that $\varphi_f$ is negative on $L \setminus (R \cup O)$, positive on 
	$R \setminus (L \cup O)$, $0$ on 
    $O  \setminus (R \cup L)$, and non-zero on $(L \cup R) \setminus O$, 
    and . First, we will show that $\varphi_f$ 
	is negative on $\widetilde{L}$, positive on $\widetilde{R}$, and $0$ on $\tO$ actually. 
	Let $x \in \widetilde{L}$, then it 
	is not part of $\mathbb{R} \setminus U$. Furthermore, the set 
	$\mathbb{R} \setminus U$ is closed and for closed sets $A$ the set distance function 
	has the property that $d(x, A) = 0$ if and only if $x \in A$. This shows that
	$d(x, \mathbb{R} \setminus U) > 0$. By separatedness of $\widetilde{L}$
	and $\widetilde{R}$ we also know that  $x$ cannot be an element of $V$, because $U$ 
	and $V$ are necessarily disjoint. It follows that $d(x, \mathbb{R} \setminus V) = 0$. 
	From this we conclude that $\varphi_f(x) < 0$ for 
	$x \in \widetilde{L}$. Analogously, it can be shown that $\varphi_f(x) > 0$ for 
	$x \in \widetilde{R}$. If $x \in \tO$, then $x \not\in U \cup V$, which implies that
    $d(x, \mathbb{R}\setminus U) = d(x, \mathbb{R}\setminus V) = 0$. Precisely what is needed
	
	Remark that $L  \setminus ( R \cup O) \subseteq \widetilde{L}$ and 
	$R \setminus (L \cup O)\subseteq \widetilde{R}$, 
	which tells us that $\varphi_f$ is negative on $L \setminus (R \cup O)$
    and positive on $R \setminus (L \cup O)$.
    Next, if $x \in L\cup R  \setminus O$ it must be in either $\tL$ or $\tR$, as 
    $L\cup R \cup O =\tL \cup \tR \cup \tO$ and $x$ cannot be in $\tO$
	From this we see that  $\varphi_f$ is non-zero on $L \cup R$. We conclude that 
	this constructed $\varphi_f$ is  a continuous recourse sensitive
    attribution function. 

    \emph{Only if}: We assume that we have a continuous recourse sensitive 
	attribution function $\varphi_f$ for $f$. Using this $\varphi_f$ we can 
    construct the required decomposition explicitly. 
    Define
	\begin{align*}
        \widetilde{L} &= \{x \in \mathcal{X}  \mid \varphi(x) < 0, x \in L\}, \\
        \widetilde{R} &= \{x \in \mathcal{X}  \mid \varphi(x) > 0, x \in R\}, \\
        \widetilde{O} &= \{x \in \mathcal{X}  \mid \varphi(x) = 0, x \in O \} 
	.\end{align*}
	We see that $\widetilde{L} \cup \widetilde{R}\cup \tO 
    = \{\varphi_f \in \mathbb{R}\}  \cap (L \cup R \cup O) = L\cup R  \cup O$. We know 
    that $\varphi_f$ is continuous. This means that $\varphi_f(x) \le 0$ on the closure
    of $\tL$. It follows that $\cl(\tL) \cap \tR =\varnothing$, as $\varphi_f(x)$ 
    is strictly positive on $\tR$. Analoguesly , it can be argued that 
    $\tL \cap \cl(\tR) = \varnothing$. Finally, $\varphi_f(x) = 0$ for all $x \in \tO$. So, 
    again $\varphi(x) = 0$ on the closure of $\tO$. This guarantees that
    $\cl(\tO) \cap \tL = \cl(\tO) \cap \tR = \varnothing$, which verifies property $3$
\end{proof}

\SuffSets*

\begin{proof}
	\emph{If:} Suppose there exists a partition $\{\tK_1,\tK_2\}$ such
	that $\tL$ and $\tR$ are separated. Then existence of $\phi_f$ follows
	from Theorem~\ref{thm:2a}: it is immediate that $\tL \subseteq L$ and
	$\tR \subseteq R$; and $\tL \union \tR = L \union R$ can be verified
	as follows: for any interval $L_i$ that is contained both in
	$\Lintervals$ and in $\Rintervals$, we have $i \in \tK$, so the
	interval is contained either in $\tL$ or in $\tR$. Any interval $L_i
	\in \Lintervals$ that is not in $\Rintervals$ is either contained in
	$\tL$ or there exists $R_j \in \Rintervals$ such that $L_i \subset
	R_j$. In the latter case $R_j$ must be contained in $\tR$, because
	there cannot exist any $i^* \in \I$ such that $R_j \subseteq L_{i^*}$.
	If there were such an $i^*$, then we would have $L_i \subset R_j
	\subseteq L_{i^*}$, which would contradict the fact that all intervals
	in $\Lintervals$ are separated. By an analogous argument, any interval
	$R_j \in \Rintervals$ that is not in $\Lintervals$ is either contained
	in $\tR$ or there exists $L_i \in \Lintervals$ that is contained in
	$\tL$.
	
	\emph{Only if:} Suppose that $\tL$ and $\tR$ satisfy the conditions of
	Theorem~\ref{thm:2a}. Then we will show that they must be of the form
	\eqref{eqn:2b} for some partition $\{\tK_1,\tK_2\}$. To this end, we
	first observe that each interval $L_i \in \Lintervals$ must either be
	fully included in $\tL$ or not included at all. Otherwise, the fact
	that $\tL \union \tR = L \union R$ would imply that part of the
	interval was included in $\tL$ and the other part in $\tR$, but then
	$\tL$ and $\tR$ would not be separated. Similarly, each interval $R_j
	\in \Rintervals$ must either be fully included in $\tR$ or not
	included at all.
	
	We can further restrict the intervals $L_i \in \Lintervals$ that can
	possibly be included in $\tL$: if there exists some $R_j \in
	\Rintervals$ such that $L_i \subset R_j$, then $R_j \setminus L \neq
	\emptyset$ (otherwise $L_i$ would not be a maximal interval), so $R_j$
	must be included in $\tR$ to ensure that $\tL \union \tR = L \union
	R$. But then $L_i$ cannot be included in $\tL$, because otherwise
	$\tL$ and $\tR$ would not be separated. Similarly, no $R_j \in
	\Rintervals$ for which there exists some $L_i \in \Lintervals$ such
	that $R_j \subset L_i$, can be included in $\tR$. This restricts
	attention to the intervals indexed by $\tI$, $\tJ$ and~$\tK$.
	
	We proceed to show that all intervals indexed by $\tI$ and $\tJ$ must
	be included in $\tL$ and $\tR$, respectively. By symmetry, it is
	sufficient to show this for intervals $L_i$ with $i \in \tI$. For
	these, we have that $L_i \setminus R \neq \emptyset$ (otherwise $R$
	would contain an interval containing $L_i$), so that $L_i$ must be
	included in $\tL$ because $\tL \union \tR = L \union R$.
	
	Finally, each interval indexed by $\tK$ must be included either in
	$\tL$ or in $\tR$, but not in both, if we are to end up with separated
	sets $\tL$ and $\tR$ that satisfy $\tL \union \tR = L \union R$.
	Consequently, the intervals indexed by $\tK$ should be partitioned
	among $\tL$ and $\tR$, as specified by the theorem.
\end{proof}

\HigherDim*

\begin{proof}
    Before we start proving both implications, we make the following observation.
    That is, the attribution $\varphi_f$ is only allowed to be non-zero in 
    the $i$'th component on the sets
    $\tL^{i}$ and $\tR^{i}$. Indeed, recourse sensitivity of $\varphi_f$ tells us that 
	$\varphi_f(x) = \gamma(y - x)$ for some $\gamma> 0$ and $\|x -y \|\le \delta$, but 
	most importantly $y$ has to be of the form $y = x \pm \alpha e_i$ by the 
	constraining set $C(x)$. The attribution is seen to be
	$\varphi_f(x) =\pm\gamma \alpha e_i$ and $\varphi_f(x)$ is only allowed to be non-zero
	in the $i$ 'th component. By continuity of $\varphi_f$ the above argument also 
    extends to the closures of $\tL^{i}$ and $\tR^{i}$.

    \emph{If}: Just as in the one-dimensional case, we are able to construct 
	a recourse sensitive function explicitly, using the set distance function $d(x, A)$ 
	defined in \eqref{eq:set_distance}. For each $\tL^{i}$ and $\tR^{i}$ find an open neighborhood
	$\tL^{i} \subseteq U^{i}$ and $\tR^{i}\subset V^{i}$ that is disjoint from all the other 
	neighborhoods and $\tO$. This  is possible because of the pairwise separatedness. 
    To see this, take one
	$\tL^{i}$ and enumerate all the other $\tL^{j}$ and $\tR^{j}$ from $k=1$ to $k=2d-1$ 
	and denote them by $\widetilde{W}_k$.
	By the pairwise separatedness we can find open neighborhoods 
	$U_k^{i}$ for $\tL^{i}$ and $V_k$ for $\widetilde{W}_k$ that are disjoint. We can also find 
    an open neighbourhood $U_{2d}^{i}$ of $\tL^i$ that is disjoint of $\tO$, Then, 
	take $U^{i}= \cap_{k=1}^{2d} U^{i}_k$. This is still an open set, as it is the finite 
	intersection of open sets, and $\tL^{i} \subseteq U^{i}$, because $\tL^{i}$ 
    is a subset of each of the 
	$U^{i}_k$. The set $U^{i}$ is also smaller than any of its components in the intersection, 
	meaning that $U^{i}$ is disjoint of all the other open neighborhoods. Repeat this
	procedure for every $\tL^{i}$ and $\tR^{i}$ to get our required op neighborhoods.
	
	We are now ready to define $\varphi_f$. For each component set
	\begin{align*}
        \varphi_f^{-}(x)_i &= \frac{d(x, \mathbb{R}^{d}\setminus U^{i})}
            {1 + d(x, \mathbb{R}^{d}\setminus U^{i})}, \\
        \varphi_f^{+}(x)_i &= \frac{d(x, \mathbb{R}^{d}\setminus V^{i})}
            {1 + d(x, \mathbb{R}^{d}\setminus V^{i})},\\
        \varphi_f(x)_i &= \varphi_f^{+}(x)_i - \varphi_f^{-}(x)_i
	.\end{align*}
	The attribution $\varphi_f$ now becomes
    \begingroup
        \renewcommand*{\arraystretch}{1.5}
        \begin{align*}
            \varphi_f(x) =
                \begin{bmatrix} 
                    \varphi_f(x)_1 \\
                    \varphi_f(x)_2 \\
                    \vdots \\
                    \varphi_f(x)_d
                \end{bmatrix} 
        .\end{align*}
    \endgroup
	All the components of $\varphi_f$ consist of continuous functions. So, $\varphi_f$ is itself 
	continuous. Next, note that if $x \in \tL^{i}$ or $x \in \tR^{i}$ for some $i$, it 
	is also contained in $U^{i}$ or $V^{i}$ respectively. Because all $U^{i}$
	and $V^{i}$ are mutually disjoint, we see that only $d(x, \mathbb{R}^{d} \setminus U^{i})$
	or $d(x, \mathbb{R}^{d} \setminus V^{i})$ is non-zero. This ensure that only the $i$ 'th
	component is non-zero, which is required be the remark at the start of this proof.
    Finally, if $x \in \tL^{i}$, then $ x \not\in \tR^{i}$ 
	and $\varphi_f(x)_i <0$,  because $x \in U^{i}$ and $\mathbb{R}^{d} \setminus U^{i}$ is closed. 
	Alternatively, if $x \in \tR^{i}$, then $x \not\in  \tL^{i} $ and $\varphi_f(x)_i >0$, as is
	required. 
	
	For notational sake denote $L = \bigcup_{j=1}^{d}L^{j}$and $ R =
        \bigcup_{j=1}^{d}R^{j}$.
    To conclude, we note that 
	\begin{align*}
        L^{i} \setminus \left( \bigcup_{\substack{j=1 \\j\neq i}}^{d}L^{j}
        \cup R  \cup O\right) &\subseteq \tL^{i}, \\
        R^{i} \setminus \left( L \cup
        \bigcup_{\substack{j=1 \\ j\neq i}} ^{d}R^{j} \cup O \right) &\subseteq \tR^{i}
	.\end{align*}
    and
    \begin{align*}
        O \setminus \left(  L \cup R\right)  \subseteq \tO
    .\end{align*}
	Combining this with the argument
	above, we see that $\varphi_f$ points in the correct directions on these sets. 
	Furthermore, $\varphi_f$ is also never zero on  
    $(L \cup R)  \setminus O$. By a similar reason as in 
    the one dimensional case we see that $x \in (L\cup R)  \setminus O$, implies that
    $x \in \tL^{i}$ or $\tR^{i}$ for some $i=1, \ldots, d$. This implies $\varphi_f(x) \neq 0$. 
    Finally, if $x \in \tO$, then $ x \not\in  U^{i} \cup V^{i} $ for all $i$. This immediately
    gives that $\varphi_f(x)  = 0 $, which shows that $\varphi_f$ is $0$ on 
    $O \setminus (L \cup R)$. 
    All together, we conclude that $\varphi_f$ is 
	a continuous recourse sensitive attribution function for $f$. 
	
    \emph{Only if}: Assuming that  $\varphi_f$ is a recourse sensitive and continuous
	attribution function for $f$, define for all $i=1, \ldots, d$ the sets
	\begin{align*}
        \widetilde{L}^{i} &= \{ x \in \mathcal{X}  \mid \varphi_f^i(x) < 0,
            x \in L^{i}\}, \\
        \widetilde{R}^{i} &= \{ x \in \mathcal{X}  \mid \varphi_f^i(x) > 0,
            x \in R^{i}\}, \\
            \tO &= \{ x \in \mathcal{X}  \mid  \varphi_f(x) = 0, x \in O\} 
	.\end{align*}
    These sets form the required partition, because
	\begin{align*}
        \tO \cup \bigcup_{i=1}^{d} \tL^{i} \cup \tR^{i}  
        = \bigcup_{i=1} ^{d} \{\varphi_f^i \in  \mathbb{R}\} \cap (L^{i} \cup R^{i} \cup O) 
        = O \cup \bigcup_{i=1}^{d}L^{i} \cup R^{i}
	,\end{align*}
    We can now verify properties $(2)$ and $(3)$ by using the continuity of 
    $\varphi_f$. Note that $\varphi_f^{i}(x) < 0$, implies that only the $i $'th component
    can be non-zero and that $x \in \tL^{i}$, by the remark at the start of the proof. 
    By continuity of $\varphi_f$  it follows that $\varphi^{i}_f(x) \le 0$ and 
    $\varphi^{j}_f(x) =0$ for all 
    $x \in \cl(\tL^{i})$. On all the other $\tL^{j}$ or $\tR^{j}$ it must be that
     $\varphi^{j}_f(x)$ is strictly non-zero, or positive for $\varphi^{i}_f(x) $ and 
     $\tR^{i}$. We see that $\cl(\tL^{i})$ is disjoint from all other $\tL^{j}$ or $\tR^{j}$. 
     This argument holds for all $i$ and we can proof it analogously for $\tR^i$. This verifies
     property $(2)$.

     Finally, $\varphi_f(x) = 0$ for all $x \in \tO$. So, again $\varphi_f(x) = 0$ on $\cl(\tO)$. 
     The function $\varphi_f$ will be non-zero on each of the sets $\tL^{i}$ and $\tR^{i}$. 
     Thus, $\cl(\tO) \cap \tL^{i}= \cl(\tO) \cap \tR^{i} = \varnothing$ for all $i =1,\ldots, d$. 
     This verifies property $(3)$. 
\end{proof}

\section{Additional Details for Section~\ref{sec:discussion}}\label{sec:additional_solutions}

In Section~\ref{sec:discussion}, it is mentioned that recourse can be provided when the model
is very simply, for example when using a linear classifier. This is also noted by 
\citet{ustun2019actionable}. In this section we will expand on this statement. 
We will also give an example of a classifier $f$ that is non-linear, but does allow a
linear representation $f(x) = \beta^{\top}g(x)$ using higher order or more abstract
features. In this example, the features $g(x)$ 
are still interpretable and providing a continuous recourse sensitive attribution function 
in terms of the features $g(x)$ is possible. 

\subsection{Linear Classifiers Admit Recourse}\label{app:linearmodels}
Consider the binary classification task using $f(x) = \beta^{\top} x$ for some vector $\beta$. 
Recall that the utility function is given by $u_f( x, y) = f(y) \ge 0$. A point is classified 
as the negative class if $f(x) < 0$ and as the preferred class if $f(x)
\ge 0$. In light of 
Theorem~\ref{thm:classification_sufficient} we see that $U$ is given by
\begin{align*}
    U = \{ x \in \mathbb{R}^{d}  \mid \beta^{\top}x \ge 0\} 
,\end{align*}
which is a convex and closed set. Using Theorem~\ref{thm:classification_sufficient} we
conclude that a recourse sensitive and robust attribution function exists.

\subsection{Attribution for Abstract Features}\label{app:higher_order_attribution}
Consider the non-linear classifier $f(x) = \|x\|^2 - 1$, which classifies if a point
is inside the circle or outside the circle. To show that there are no continuous and 
recourse sensitive functions for this classifier we consider the following two cases:
\begin{enumerate}
    \item $\delta > 0$ and 
        $C(x) = \{y \in \mathbb{R}^2  \mid \| x -y \|_0 \le 1\}$;
    \item $1\le \delta < 2$ and $C(x) =\mathbb{R}^2$.
\end{enumerate}

First, we will show the 
single feature case, because it follows from Theorem~\ref{thm:higher_dim_2}. The second
case requires special arguments and will follow afterwards.

\subsubsection{$\delta>0$ and 
$C(x) = \{y \in \mathbb{R}^2  \mid \|x - y\|_0 \le 1\}$}

To apply Theorem~\ref{thm:higher_dim_2} we first find all $4$ sets $L^{1}, R^{1}, L^{2}$ 
and $R^{2}$. If we know $L^{1}$, then we can find all the other sets as well by the symmetry of $f$. 
The set $L^{1}$ consists of all points such that you cross the decision boundary when you 
subtract $[\delta, 0]^{\top}$ from the input point. This is the strip to right of
the circle with width $\delta$ and
all the points within the circle that also do not lie in the translated circle 
$D_1(0)  + [\delta , 0]^{\top}$, where 
$D_1(0) = \{y \in \mathbb{R}^{2}  \mid \|y\| < \delta\}$.
In set notation 
\begin{align*}
    L^{1} &=
    \left\{  
    \begin{bmatrix}     
        \cos(\theta) \\
        \sin(\theta)
    \end{bmatrix}  + 
    \begin{bmatrix} 
        \alpha \\
        0
    \end{bmatrix}  \mid \theta \in \left( -\frac{\pi}{2}, \frac{\pi}{2} \right), 
    \alpha \in (0, \delta) 
    \right\} 
    \cup
    \left( 
    D_1(0)  \setminus \left( D_1(0) + 
    \begin{bmatrix} 
        \delta \\ 
        0
    \end{bmatrix}   
        \right)
    \right) \\
    &=: L^{1}_{\text{out}} \cup L^{1}_{\text{in}}
.\end{align*}
Note that $L^{1}_{\text{out}}$ and $L^{1}_{\text{in}}$ are two disjoint connected components. 
The set $R^{1}$ can be given in a similar form, with the $\alpha$ replaced by  $-\alpha$ and
the vector $[\delta, 0]^{\top}$ with $[-\delta, 0]^{\top}$. The sets $L^{2}$ and $R^{2}$ can be obtained
by rotating the sets $L^{1}$ and $R^{1}$ with $\frac{\pi}{2}$. 

Take any $\alpha \in \left( 0, \delta \right) $ and consider the point $x=[1 + \alpha, 0]^{\top}$. This point
is only an element of $L^{1}$ and not of any of the other sets. As $x$ is contained 
in $L^{1}_{\text{out}}$ and $L^{1}_{\text{out}}$ is connected, we know that it must be that
$L^{1}_{\text{out}} \subseteq \tL^{1}$ for any decomposition. Similarly, we see that 
$L^{2}_{\text{out}} \subseteq \tL^{2}$. However, $L^{1}_{\text{out}}$ and $L^{2}_{\text{out}}$ are
not disjoint, because $\sqrt{\alpha}[\nicefrac{1}{\sqrt{2}}, \nicefrac{1}{\sqrt{2}}]$ is an element 
of both sets for $\alpha \in (1, \min\left( 2,   2\delta  \right)  )$. It follows that
$\tL^{1}$ and $\tL^{2}$ cannot be separated and Theorem~\ref{thm:higher_dim_2} tells us that 
no continuous single feature attribution function can exist. 

\subsubsection{$1\le \delta < 2$ and $C(x) =\mathbb{R}^2$} Note that in all cases that 
follow an attribution can be given 
for the region outside of the circle by 
\begin{align*}
    \varphi_f(x) = 
    \begin{bmatrix} 
        -x_1 f(x) \\ 
        -x_2 f(x)
    \end{bmatrix} 
.\end{align*}
So, we only have to focus on the region inside the circle. If $0 < \delta < 1 $, then 
we cannot cross the decision boundary for any $x \in D_{1- \delta}(0)$. In that region
any value for the attribution is allowed and extending the above $\varphi_f$ to the
whole plane gives us a valid continuous recourse sensitive attribution function. 

When $\delta > 2$, we can cross the decision boundary for any $x\in D_1(0)$ by moving 
in any direction with length $\delta$. So, inside the circle we could set 
$\varphi_f(x)$ to any direction. A full continuous recourse sensitive attribution 
function would be given by
\begin{align*}
    \varphi_f(x) = 
    \begin{cases}
        \begin{bmatrix} 
            -x_1 f(x)\\
            -x_2 f(x)
        \end{bmatrix}  & \|x\|\ge 1 \\
        \begin{bmatrix} 
            f(x) \\
            0
        \end{bmatrix} & \|x\| < 1
    \end{cases}
.\end{align*}
Now, we can discuss the case when $1\le \delta < 2$. Again, the attribution 
outside of the circle does not pose a problem. Inside the circle we can identify 
two regions. Within $D_{\delta -1 }(0)$ we can move in any direction of length 
$\delta$ to cross the decision boundary. Indeed, take an $x \in D_{\delta -1}(0)$
and note that the worst direction to cross the decision boundary is $-x$. We can scale
this vector with $\frac{\delta}{\|x\|}$ to get a vector of length $\delta$. 
Using $\|x\| <  \delta -1$, we see that moving in that direction crosses the decision boundary, as
\begin{align*}
    \| x - \frac{\delta}{\|x\|}x \|
        = \left|1 - \frac{\delta}{\|x\|}\right| \|x\| 
        = | \delta - \|x\|| > 1
.\end{align*}
In the strip with $\delta -1\le \|x\| < 1 $, the set of feasible direction is more complicated. 
However, the important observation is that $-x$ is not contained in it. So, for any 
attribution $\varphi_f$ it cannot be that $\varphi_f(x) = - \alpha x$ for any $\alpha >0$. 
To conclude that $\varphi_f$ has a zero we will use the following Lemma,
which can be seen as a generalization of the intermediate value theorem. 

\begin{lemma}[Poincar\'e-Bohl]
    \label{lem:poincare-bohl}
    Assume that $U$ is an open bounded neighborhood of $\mathbb{R}^{d}$, with $0 \in U$, 
    and that $f \colon \cl(U) \to \mathbb{R}^{d}$ is a continuous function such that
    \begin{align*}
        f(x) \not\in  \{\alpha x \colon \alpha > 0\}, &
            \text{ for every } x \in \cl(U) \setminus U
    .\end{align*}
    Then, there is an $x_0 \in \cl(U)$ such that $f(x_0) =0$.
\end{lemma}

\begin{proof}
    See Theorem~$2$ in \citet{fonda2016generalizing}. 
\end{proof}

Applying Lemma~\ref{lem:poincare-bohl} to the function $-\varphi_f(x)$ immediately gives that 
there is some  $x \in D_1\left( 0 \right) $ such that $-\varphi_f(x) = 0 \iff \varphi_f(x) =  0$, 
which is not allowed if $\varphi$ by recourse sensitivity. 

However, if we write this function as a linear function of a feature map
consisting of linear and quadratic terms. The feature map $g$ and coefficients 
are given by
\begin{align*}
    g(x)  = 
    \begin{bmatrix} 
        x_1 \\
        x_2 \\
        x_1^2 \\
        x_2^2 \\
        1
    \end{bmatrix},
    \quad \quad
    \beta = 
    \begin{bmatrix} 
        0 \\
        0 \\ 
        1 \\
        1 \\
        -1
    \end{bmatrix}
.\end{align*}
The function $f$ is then represented by 
\begin{align*}
    f(x) = \beta^{\top}g(x)
,\end{align*}
and we could provide recourse by communicating 
\begin{align*}
    \varphi_f(x) =  
    \begin{bmatrix} 
        -x_1 f(x) \\
        -x_2 f(x) \\
        -f(x) \\
        -f(x) \\
        0
    \end{bmatrix} 
.\end{align*}
This attribution will only be $0$ on the decision boundary, which is allowed, and in almost 
all other cases the first two components will point towards the decision boundary. The first 
two components are only zero when $x=0$. In that case the first two components do 
not point towards the decision boundary, but the final two components do provide information 
on which (higher-level) action has to be taken to change the class. 
Namely, it tells the user to increase the norm, in whatever way possible. 

The above argument shows that, if it is possible to write the function $f$ 
as some linear function $f(x) = \beta^{\top}g(x)$, it will be possible to 
provide recourse in terms of the higher level features 
of $g(x)$.

\end{document}